\documentclass{article}

\PassOptionsToPackage{numbers}{natbib}

\usepackage[final]{neurips_2020}
\usepackage[utf8]{inputenc} 
\usepackage[T1]{fontenc}    
\usepackage{hyperref}       
\usepackage{url}            
\usepackage{booktabs}       
\usepackage{amsfonts}       
\usepackage{nicefrac}       
\usepackage{microtype}      
\usepackage{style}
\usepackage{microtype}
\usepackage{graphicx}
\usepackage{booktabs} 
\usepackage{bm}
\usepackage{subcaption}
\usepackage{diagbox}
\usepackage[dvipsnames]{xcolor}
\usepackage{bbding}
\usepackage{pifont}
\usepackage[export]{adjustbox}
\usepackage{listings}

\usepackage{hyperref}
\usepackage{amsmath,amsthm,amssymb,style}
\setcitestyle{numbers,open={[},close={]}}

\makeatletter

\newcommand*{\addFileDependency}[1]{
  \typeout{(#1)}
  \@addtofilelist{#1}
  \IfFileExists{#1}{}{\typeout{No file #1.}}
}

\makeatother

\newcommand{\Qffg}{\mcQ_{\mathrm{FFG}}}
\newcommand{\Qmcdo}{\mcQ_{\mathrm{MCDO}}}
\newcommand{\hatbfW}{\widehat{\bfW}}
\newcommand{\diag}{\mathrm{diag}}
\newcommand{\bfeps}{\bm{\epsilon}}

\newcommand{\ELBO}{\mcL}
\newcommand{\thetainput}{\theta_{\mathrm{in}}}

\title{On the Expressiveness of Approximate Inference in Bayesian Neural Networks}

\author{%
  \hspace{0.9cm}Andrew Y.~K.~Foong\thanks{Equal contribution.} \\
  \hspace{0.9cm}University of Cambridge\\
  \hspace{0.9cm}\texttt{ykf21@cam.ac.uk} \\
  \And
  \hspace{0.5cm}David R.~Burt\footnotemark[1] \\
  \hspace{0.5cm}University of Cambridge\\
  \hspace{0.5cm}\texttt{drb62@cam.ac.uk} \\
  \hspace{5cm}
  \And
  Yingzhen Li \\
  Microsoft Research\\
  \texttt{Yingzhen.Li@microsoft.com}
  \And Richard E.~Turner \\
  University of Cambridge\\
  Microsoft Research\\
  \texttt{ret26@cam.ac.uk} \\
}

\begin{document}

\maketitle

\begin{abstract}
While Bayesian neural networks (BNNs) hold the promise of being flexible, well-calibrated statistical models, inference often requires approximations whose consequences are poorly understood. We study the quality of common variational methods in approximating the Bayesian predictive distribution. For single-hidden layer ReLU BNNs, we prove a fundamental limitation in \emph{function-space} of two of the most commonly used distributions defined in \emph{weight-space}: mean-field Gaussian and Monte Carlo dropout. We find there are simple cases where neither method can have substantially increased uncertainty in between well-separated regions of low uncertainty. We provide strong empirical evidence that exact inference does not have this pathology, hence it is due to the approximation and not the model. In contrast, for deep networks, we prove a universality result showing that there exist approximate posteriors in the above classes which provide flexible uncertainty estimates. However, we find empirically that pathologies of a similar form as in the single-hidden layer case can persist when performing variational inference in deeper networks. Our results motivate careful consideration of the implications of approximate inference methods in BNNs.
\end{abstract}

\section{Introduction}\label{sec:intro}
Bayesian neural networks (BNNs) \cite{mackay1992practical,neal2012bayesian} aim to combine the strong inductive biases and flexibility of neural networks (NNs) with the probabilistic framework for uncertainty quantification provided by Bayesian statistics. However, performing exact inference in BNNs is analytically intractable and requires approximations. A variety of scalable approximate inference techniques have been proposed, with mean-field variational inference (MFVI) \citep{hinton1993keeping,blundell2015weight} and Monte Carlo dropout (MCDO) \citep{gal2016dropout} among the most used methods. These methods have been succesful in applications such as active learning and out-of-distribution detection \citep{oatml2019bdlb,ovadia2019can}. However, it is unclear to what extent the successes (and failures) of BNNs are attributable to the exact Bayesian predictive, rather than the peculiarities of the approximation method. From a Bayesian modelling perspective, it is therefore crucial to ask, \emph{does the approximate predictive distribution retain the qualitative features of the exact predictive?} 

Frequently, BNN approximations define a simple class of distributions over the model parameters, (an \emph{approximating family}), and choose a member of this family as an approximation to the posterior. Both MFVI and MCDO follow this paradigm. For such a method to succeed, two criteria must be met:
\begin{enumerate}[label=\textbf{Criterion \arabic*},topsep=-1pt,itemsep=0cm,parsep=1ex,align=left,leftmargin=*,itemindent=1.8cm]
    \item\label{cond:expressive_family} The approximating family must contain good approximations to the posterior.
    \item \label{cond:choose_member}  The method must then select a good approximate posterior within this family.
\end{enumerate}

For nearly all tasks, the performance of a BNN only depends on the distribution over weights to the extent that it affects the distribution over predictions (i.e.~in `function-space'). Hence for our purposes, a `good' approximation is one that captures features of the exact posterior in function-space that are relevant to the task at hand. However, approximating families are often defined in weight-space for computational reasons. Evaluating \ref{cond:expressive_family} therefore involves understanding how weight-space approximations translate to function-space, which is a non-trivial task for highly nonlinear models such as BNNs.

In this work we provide both theoretical and empirical analyses of the flexibility of the predictive mean and variance functions of approximate BNNs. Our major findings are:
\begin{enumerate}[topsep=-2pt,itemindent=.4cm,align=left,leftmargin=*]
    \item For shallow BNNs, there exist simple situations where \emph{no} mean-field Gaussian or MC dropout distribution can faithfully represent the exact posterior predictive uncertainty (\ref{cond:expressive_family} is not satisfied). We prove in \cref{sec:SHL-theory} that in these instances the variance function of any fully-connected, single-hidden layer ReLU BNN using these families suffers a lack of `\emph{in-between uncertainty}': increased uncertainty in between well-separated regions of low uncertainty. This is especially problematic for lower-dimensional data where we may expect some datapoints to be in between others. Examples include spatio-temporal data, or Bayesian optimisation for hyperparameter search, where we frequently wish to make predictions in unobserved regions in between observed regions. We verify that the exact posterior predictive does not have this limitation; hence this pathology is attributable solely to the restrictiveness of the approximating family. 
    \item In \cref{sec:deep-stuff} we prove a universal approximation result showing that the mean and variance functions of deep approximate BNNs using mean-field Gaussian or MCDO distributions can uniformly approximate any continuous function and any continuous non-negative function respectively. However, it remains to be shown that appropriate predictive means and variances will be found when optimising the ELBO. To test this, we focus on the low-dimensional, small data regime where comparisons to references for the exact posterior such as the limiting GP \citep{neal2012bayesian, lee2017deep, matthews2018gaussian} are easier to make. In \cref{sec:deep-empirical} we provide empirical evidence that in spite of its theoretical flexibility, VI in deep BNNs can still lead to distributions that suffer from similar pathologies to the shallow case, i.e.~\ref{cond:choose_member} is not satisfied.
\end{enumerate}

In \cref{sec:experiments}, we provide an active learning case study on a real-world dataset showing how in-between uncertainty can be a crucial feature of the posterior predictive. In this case, we provide evidence that although the inductive biases of the BNN model with exact inference can bring considerable benefits, these are lost when MFVI or MCDO are used. Code to reproduce our experiments can be found at \url{https://github.com/cambridge-mlg/expressiveness-approx-bnns}.

\section{Background}\label{sec:background}
Consider a regression dataset $\mcD = \{ (\bfx_n, y_n) \}_{n=1}^N$ with $\bfx_n \in \R^D$ and $y_n \in \R$. To define a BNN, we specify a prior distribution with density $p(\theta)$ over the NN parameters. Each parameter setting corresponds to a function $f_\theta: \R^D \to \R$. We specify a likelihood $p(\{y_n\}_{n=1}^N| \{\bfx_n\}_{n=1}^N,f_\theta)$ which describes the relationship between the observed data and the model parameters. The posterior distribution over parameters has density $p(\theta|\mcD) \propto p(\{y_n\}_{n=1}^N| \{\bfx_n\}_{n=1}^N, f_\theta)p(\theta)$. The posterior does not have a closed form and approximations must be made in order to make predictions.

\subsection{Approximate Inference Methods}\label{sec:approx-family-methods}
Many approximate inference algorithms define a parametric class of distributions, $\mcQ$, from which to select an approximation to the posterior. For BNNs, the distributions in $\mcQ$ are defined over the model parameters $\theta$. For example, $\mcQ$ may be the set of all fully-factorised Gaussian distributions, in which case the variational parameters $\phi$ are a vector of means and variances. We denote this family as $\Qffg$. 
A density $q_\phi(\theta) \in \mcQ$ is then chosen to best approximate the exact posterior according to some criteria. Once $q_\phi$ is selected, predictions at a test point $(\bfx_*, y_*)$ can be made by replacing the expectation under the exact posterior by an expectation under the approximate posterior:
\begin{align} 
    p(y_*| \bfx_*, \mcD) \!=\! \Exp{p( \theta |  \mcD)}{p(y_*|\bfx_*, f_\theta)} \!\approx\! \Exp{q_{\phi}(\theta)}{p(y_*|\bfx_*, f_\theta)}  \approx \frac{1}{M}\sum_{m=1}^M p(y_*|\bfx_*, f_{\theta_m}),\label{eqn:approx_MC_estimate}
\end{align}
where $\theta_m \sim q_{\phi}$ on the RHS of \cref{eqn:approx_MC_estimate}. Many approximate inference algorithms may share the same $\mcQ$, e.g.~VI, the diagonal Laplace approximation \citep{denker1991transforming}, probabilistic backpropagation \citep{hernandez2015probabilistic}, stochastic expectation propagation \citep{li2015stochastic}, black-box alpha divergence minimisation \citep{pmlr-v48-hernandez-lobatob16}, R{\'e}nyi divergence VI \citep{li2016renyi}, natural gradient VI \citep{khan2018fast} and functional variational BNNs \citep{sun2019functional} all frequently use $\Qffg$.

\subparagraph{Mean-Field Variational Inference} Variational inference \citep{beal2003variational,jordan1999introduction} is an approximate inference method that selects $q_\phi$ by minimising the KL divergence between the approximate and exact posterior \citep{blei2017variational}. This is equivalent to maximising an evidence lower bound (ELBO): $
    \ELBO(\phi) = \sum_{n=1}^N \Exp{q_{\phi}}{\log p(y_n|\bfx_n, f_\theta)} - \KL{q_{\phi}(\theta)}{p(\theta)}$. Most commonly in BNNs, $q_\phi$ is chosen from $\Qffg$. This is known as mean-field variational inference (MFVI). 

\subparagraph{Monte Carlo Dropout}
MCDO with $\ell_2$ regularisation has been interpreted as VI \citep{gal2016uncertainty}. Although the MCDO objective is not strictly an ELBO \citep{hron2018variational}, we will sometimes refer to it as such. The variational family, $\Qmcdo$, is the set of distributions determined by random variables of the form $\hatbfW \coloneqq \bfW \diag(\bfeps)$; where the weights $\bfW$ are variational parameters and $\epsilon_i \iidsim \mathrm{Bern}(1-p)$, with $p$ the dropout probability. Frequently, the first weight matrix $\bfW_1$ is deterministic (i.e.~inputs are not dropped out) --- we analyse this case in the main body and use $\Qmcdo$ to refer to this family. There are fundamentally different considerations when $\bfW_1$ is also stochastic, addressed in \cref{app:dropout-drop-inputs}. 
\subsection{BNN Priors and References for the Exact Posterior Predictive} \label{sec:priors_and_references}
In this paper, we examine how closely approximate BNN predictive distributions resemble exact inference. To make this comparison, a choice of BNN prior must be made. Common practice is to choose an independent $\mcN(0,1)$ prior for all parameters, regardless of the size of the network. However, such priors are known to lead to extremely large variance in function space for wide or deep networks \citep{neal2012bayesian}. For example, choosing such a prior for a 4-hidden layer BNN with 50 neurons in each layer leads to a prior standard deviation of ${\sim}10^3$ in function space at the origin. This is orders of magnitude too large for normalised data. It is conceivable that one may combine an unreasonable prior with poor approximate inference to obtain practically useful uncertainty estimates that bear little relation to the exact Bayesian predictive --- we do not consider this case. Instead, we focus our study on the quality of approximate inference in models with moderate prior variances in function space.

There is a body of literature on BNN priors \citep{neal2012bayesian,matthews2018gaussian,schoenholz2016deep,lee2017deep} which shows how to select prior weight variances that lead to reasonable prior variances in function space, even as the width of the hidden layers tends to infinity. For a layer with $N_{\mathrm{in}}$ inputs, we choose independent $\mcN(0, \sigma_w^2/{N_{\mathrm{in}}})$ priors for the weights, with $\sigma_w^2$ a constant. For regression with a Gaussian likelihood, as the width tends to infinity, both the prior and posterior of such a BNN converges to a Gaussian process (GP) \citep{hron2020exact,neal2012bayesian,matthews2018gaussian}. It has been shown that even moderately wide BNNs closely resemble their corresponding infinite-width GP counterparts \citep{matthews2018gaussian}. In this work, we use exact inference in the corresponding infinite-width limit GP and also `gold-standard' Hamiltonian Monte Carlo (HMC) \citep{neal2011mcmc,hoffman2014no} as references for the exact posterior.

\begin{figure}[t]
\begin{center}
\includegraphics[width=0.25\columnwidth]{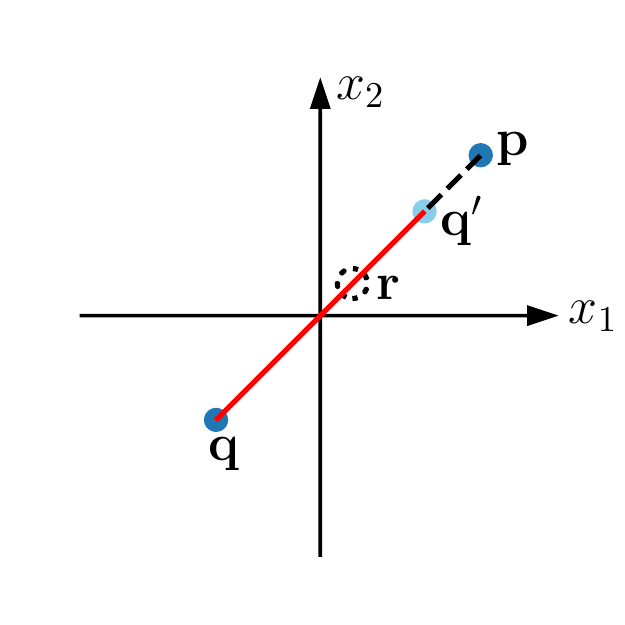}
\includegraphics[width=0.25\columnwidth]{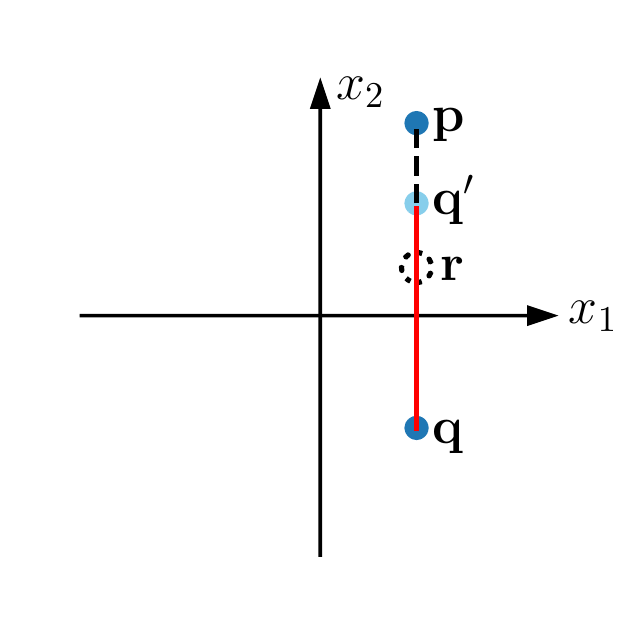}

\caption{Illustration of the bounded regions in \cref{thm:mfgaussian}, showing the input domain of a 1HL mean-field Gaussian BNN, for the case $\bfx \in \R^2$. Left (resp.~Right): For any two points $\bfp$ and $\bfq$ such that the line joining them crosses the origin (resp.~is orthogonal to and intersects a plane $x_d = 0$), the output variance at any point $\bfr$ on the solid red portion of the line is upper bounded by $\Var[f(\bfp)] + \Var[f(\bfq)]$, illustrating condition (i) (resp.~condition (ii)) of \cref{thm:mfgaussian}. The bounded region extends from $\bfq = (q_1, q_2)$ to $\bfq'$, where $\bfq' = (-q_1, -q_2)$ (Left), or $\bfq' = (q_1, -q_2)$ (Right).}
\label{fig:mean-field-illustration}
\vspace{-.3cm}
\end{center}
\end{figure}

\section{Single-Hidden Layer Neural Networks}\label{sec:SHL-theory}
\begin{figure*}
    \centering
    \includegraphics[width=.48\columnwidth,trim={.1cm .7cm .1cm .25cm},clip]{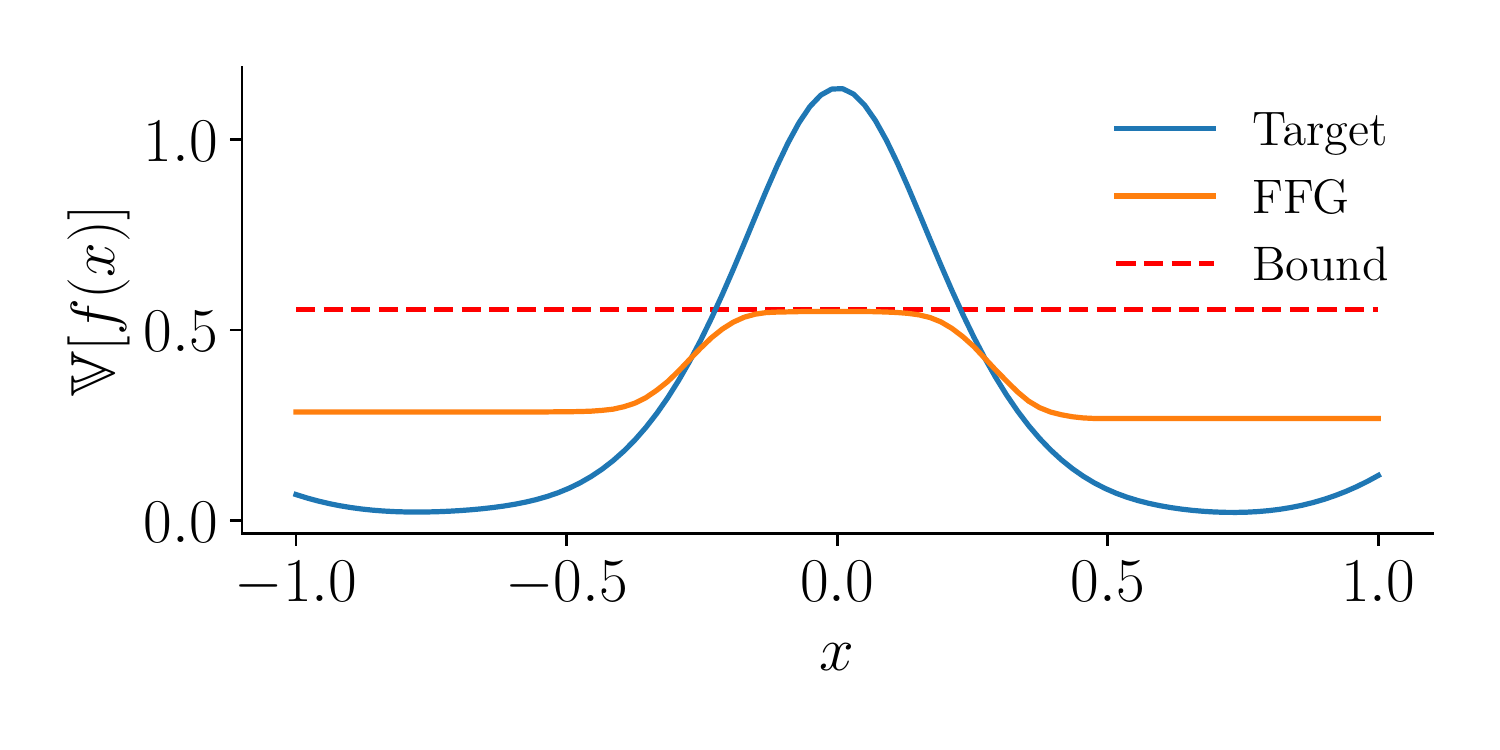}
        \includegraphics[width=.48\columnwidth,trim={.1cm .7cm .1cm .25cm},clip]{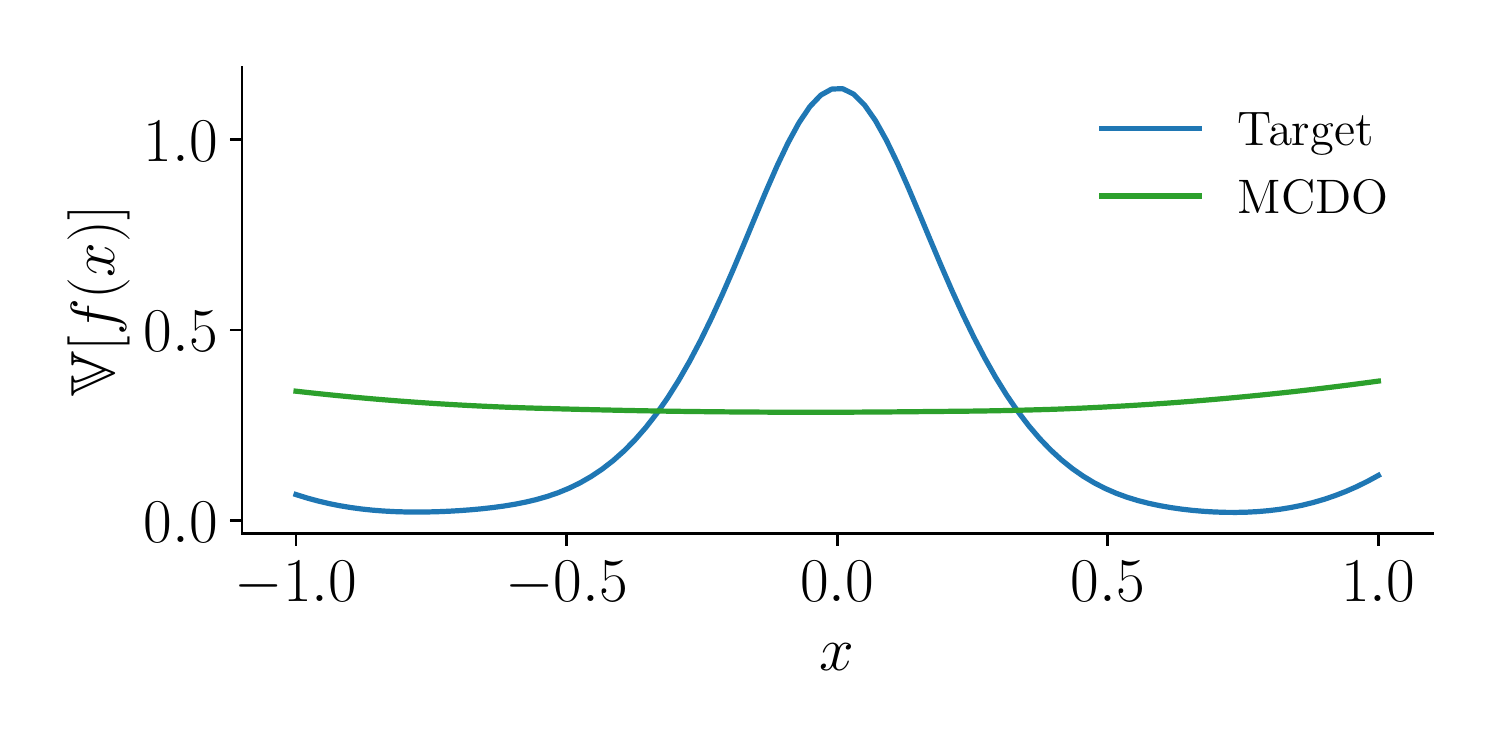}
    \caption{Results of \emph{directly minimising the squared error in function space} between $\Var[f(x)]$ (for a single-hidden layer NN) and a target variance function. Left: FFG distribution, Right: MCDO distribution. The bound for FFG distributions (red) applies on $[-1,1]$ with $\bfp=-1,\bfq=1$. The MCDO variance function is convex, and almost constant. The FFG and MCDO variance functions underestimate the target near the origin and overestimate it away from the origin.}
    \label{fig:single-hidden-layer-approximation}
    \vspace{-.3cm}
\end{figure*}
In this section, we prove that for single-hidden layer (1HL) ReLU BNNs, $\Qffg$ and $\Qmcdo$ are not expressive enough to satisfy \ref{cond:expressive_family}. We identify limitations on the variance in function-space, $\Var[f(\bfx)]$, implied by these families. We show empirically that the exact posterior does not have these restrictions, implying that approximate inference does not qualitatively resemble the posterior. 
\begin{thm}[Factorised Gaussian]\label{thm:mfgaussian}
    Consider any 1HL fully-connected ReLU NN $f:\R^D\to \R$. Let $x_d$ denote the $d$\textsuperscript{th} element of the input vector $\bfx$. Assume a fully factorised Gaussian distribution over the parameters. Consider any points $\bfp, \bfq, \bfr \in \R^D$ such that $\bfr \in \overrightarrow{\bfp\bfq}$ and either: 
    \begin{enumerate*}[label=\roman*.,itemsep=0pt,topsep=0pt]
        \item $\overrightarrow{\bfp\bfq}$ contains $\mathbf{0}$ \emph{and} $\bfr$ is closer to $\mathbf{0}$ than both $\bfp$ and $\bfq$,
        \label{item:case1} \textup{or}
        \item $\overrightarrow{\bfp\bfq}$ is orthogonal to and intersects the plane $x_d=0$, \emph{and} $\bfr$ is closer to the plane $x_d=0$ than both $\bfp$ and $\bfq.$ \label{item:case2}
    \end{enumerate*}
     Then $\Var[f(\bfr)] \leq \Var[f(\bfp)]+\Var[f(\bfq)]$. 
\end{thm}

 \Cref{thm:mfgaussian} states that there are line segments (illustrated in \cref{fig:mean-field-illustration}) in input space such that the predictive variance on the line is bounded by the sum of the variance at the endpoints. Analogous but weaker bounds on higher dimensional sets in input space enclosed by these lines can be obtained as a corollary (see \cref{app:SHL-proofs}). \Cref{thm:mfgaussian} applies to any method using $\Qffg$, as listed in \cref{sec:approx-family-methods}. Although \cref{thm:mfgaussian} only bounds certain lines in input space, in \cref{app:random-regression} we provide figures empirically showing that lines joining random points in input space suffer from similar behaviour. We provide similar results for MC dropout:

 \begin{thm}[MC dropout]\label{thm:dropout}
  Consider the same network architecture as in \cref{thm:mfgaussian}. Assume an MC dropout distribution over the parameters, with inputs not dropped out. Then $\Var[f(\bfx)]$ is convex in $\bfx$. 
 \end{thm}

\begin{rem}
In \cref{app:dropout-drop-inputs}, we consider MC dropout with the inputs also dropped out. We prove that the variance at the origin is bounded by the maximum of the variance at any set of points containing the origin in their convex hull. This also applies to variational Gaussian dropout \citep{kingma2015variational}. In the main body, we assume inputs are not dropped out.
\end{rem}

\Cref{thm:dropout} implies the predictive variance on any line segment in input space is bounded by the maximum of the variance at its endpoints. Full proofs of \cref{thm:mfgaussian,thm:dropout} are in \cref{app:SHL-proofs}.  \Cref{thm:mfgaussian,thm:dropout} show that there are simple cases where 1HL approximate BNNs using $\Qffg$ and $\Qmcdo$ cannot represent \emph{in-between uncertainty}: i.e., increased uncertainty in between well separated regions of low uncertainty. As \cref{thm:mfgaussian,thm:dropout} depend only on the approximating family, this cannot be fixed by improving the optimiser, regulariser or prior. \Cref{fig:single-hidden-layer-approximation} shows a numerical verification of \cref{thm:mfgaussian,thm:dropout}. Since we are concerned with whether there are \emph{any} distributions that show in-between uncertainty, we do not maximise the ELBO in this experiment (we consider ELBO maximisation in \cref{sec:deep-empirical,sec:empirical-tests}). Instead, we train 1HL networks of width 50 with $\Qffg$ and $\Qmcdo$ distributions to \emph{directly} minimise the squared error between $\Var[f(x)]$ and a pre-specified target variance function displaying in-between uncertainty. Full details are given in \cref{app:approximator-details}. Although \cref{thm:mfgaussian,thm:dropout} apply only to 1HL BNNs, 1HL BNN regression tasks are a very common benchmark in the BNN literature \citep{mukhoti2018importance, tomczak2018neural,gal2016dropout,hernandez2015probabilistic, sun2019functional}, and have been used to assess different inference methods.

\subsection{Intuition for Theorems 1 and 2}\label{sec:SHL-intuition}
We now provide intuition for the proofs of \cref{thm:mfgaussian,thm:dropout}. Let  $\thetainput$ be the parameters in the first layer. By the law of total variance,  $\Var[f(\bfx)]\! =\! \Exp{\!}{\Var[f(\bfx)|\thetainput]}\!+\! \Var[\Exp{\!}{f(\mathbf{x})|\thetainput}].$ For $\Qmcdo$ the second term is $0$ as $\thetainput$ is deterministic. Hence to prove \cref{thm:dropout}, it suffices to show the first term is convex. We have:
\begin{align}
    \Var[f(\bfx)|\thetainput&] = \Var \bigg[\sum_{i=1}^I w_i \psi(a_i(\bfx; \thetainput)) + b \bigg| \thetainput \bigg] = \sum_{i=1}^I \Var[w_i] \psi(a_i(\bfx; \thetainput)) ^2 + \Var[b], \label{eqn:sum_of_variances}
\end{align}
where $\{w_i\}_{i=1}^I$ and $b$ are the output weights and bias, $\psi(a) = \mathrm{max}(0, a)$, and $a_i(\bfx; \thetainput)$ is the activation of the $i$\textsuperscript{th} neuron. Since $a_i(\bfx; \thetainput)$ is affine in $\bfx$, $\psi(a_i(\bfx; \thetainput)) ^2$ is a `rectified quadratic' in $\bfx$ and therefore convex. This proves \cref{thm:dropout}. The same argument also applies to show that $\Var[f(\bfx)|\thetainput]$ is convex for $\Qffg$. To arrive at \cref{eqn:sum_of_variances}, we used that for $\Qffg$ and $\Qmcdo$, the output weights of each neuron are independent. Correlations between the weights could introduce negative covariance terms, leading to non-convex behaviour. Thus we see how \emph{weight-space} factorisation assumptions can lead to \emph{function-space} restrictions on the predictive uncertainty.

\begin{figure*}[t]
    \centering
    \begin{subfigure}[b]{0.24\columnwidth}
        \includegraphics[width=\textwidth]{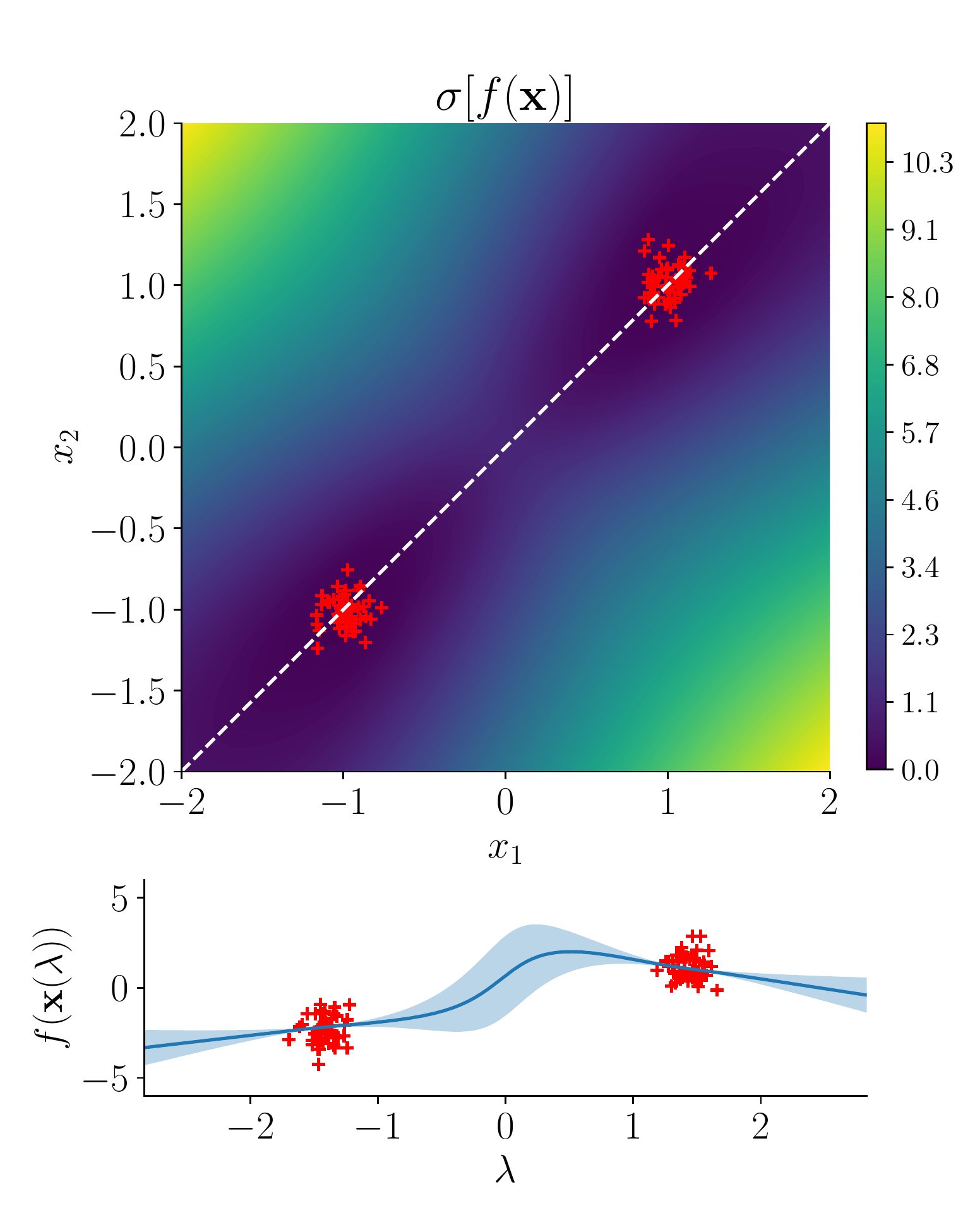}
        \caption{Infinite-width limit GP}
    \end{subfigure}
    \begin{subfigure}[b]{0.24\columnwidth}
        \includegraphics[width=\textwidth]{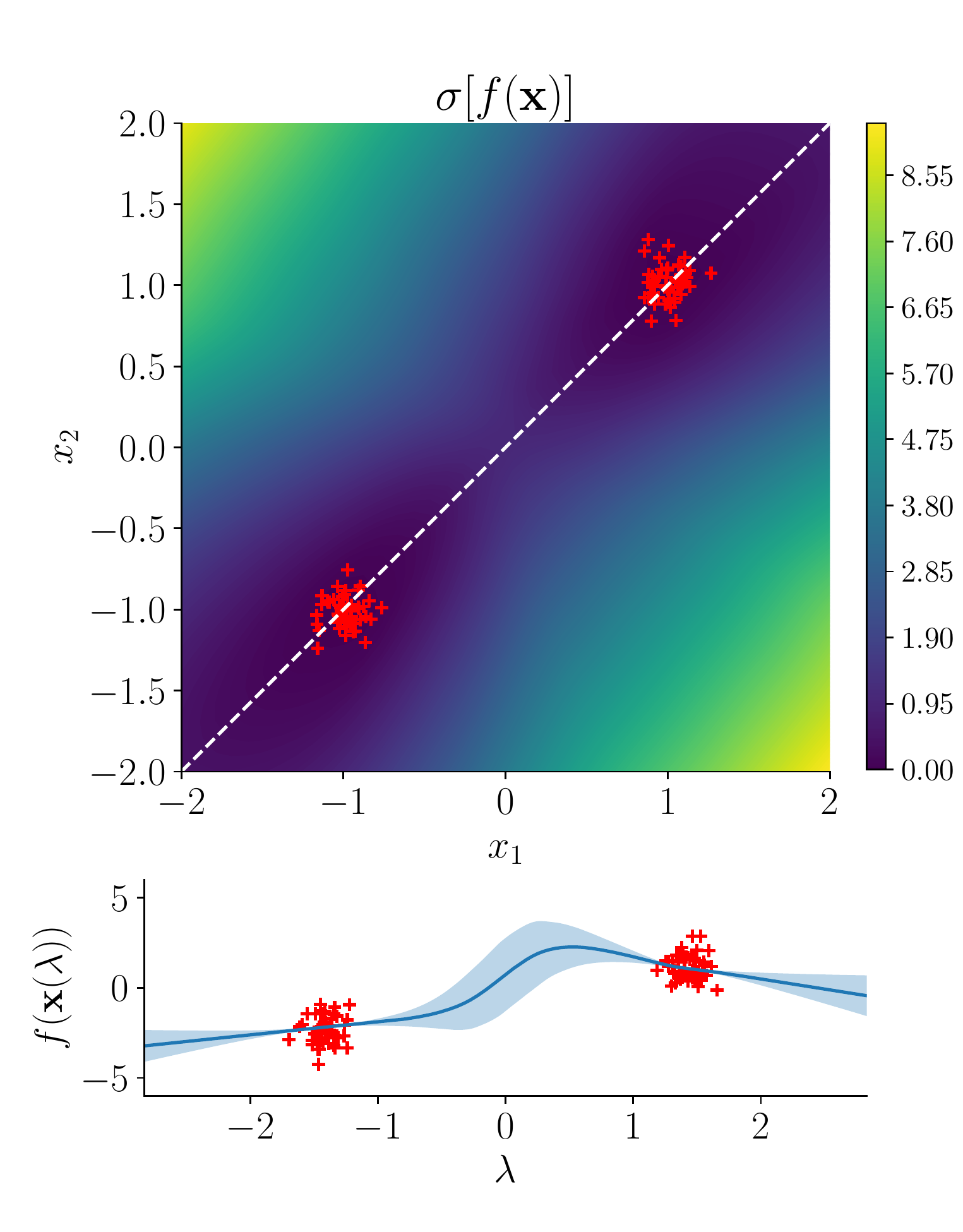}
        \caption{HMC}
    \end{subfigure}
    \begin{subfigure}[b]{0.24\columnwidth}
        \includegraphics[width=\textwidth]{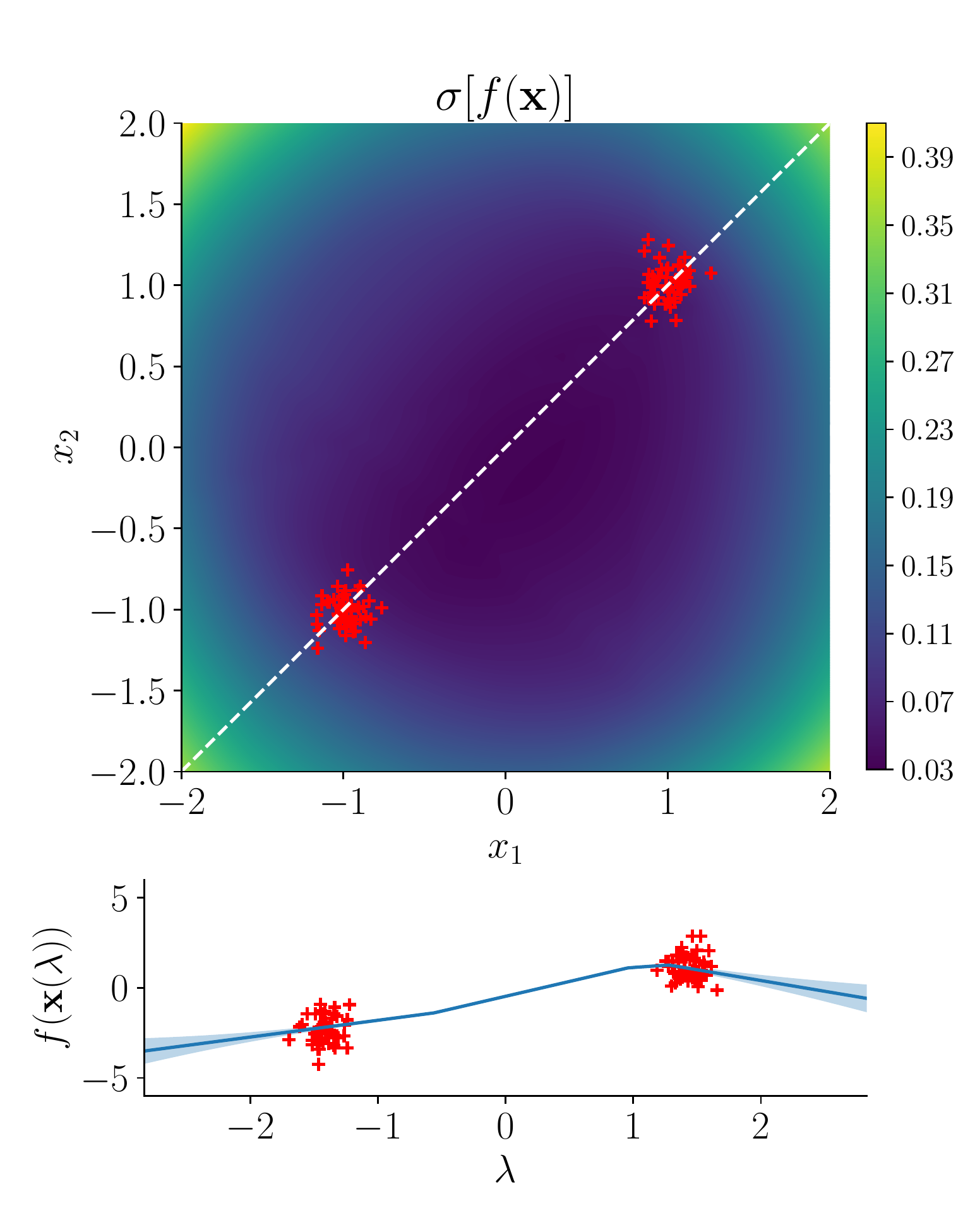}
        \caption{MFVI}
    \end{subfigure}
    \begin{subfigure}[b]{0.24\columnwidth}
        \includegraphics[width=\textwidth]{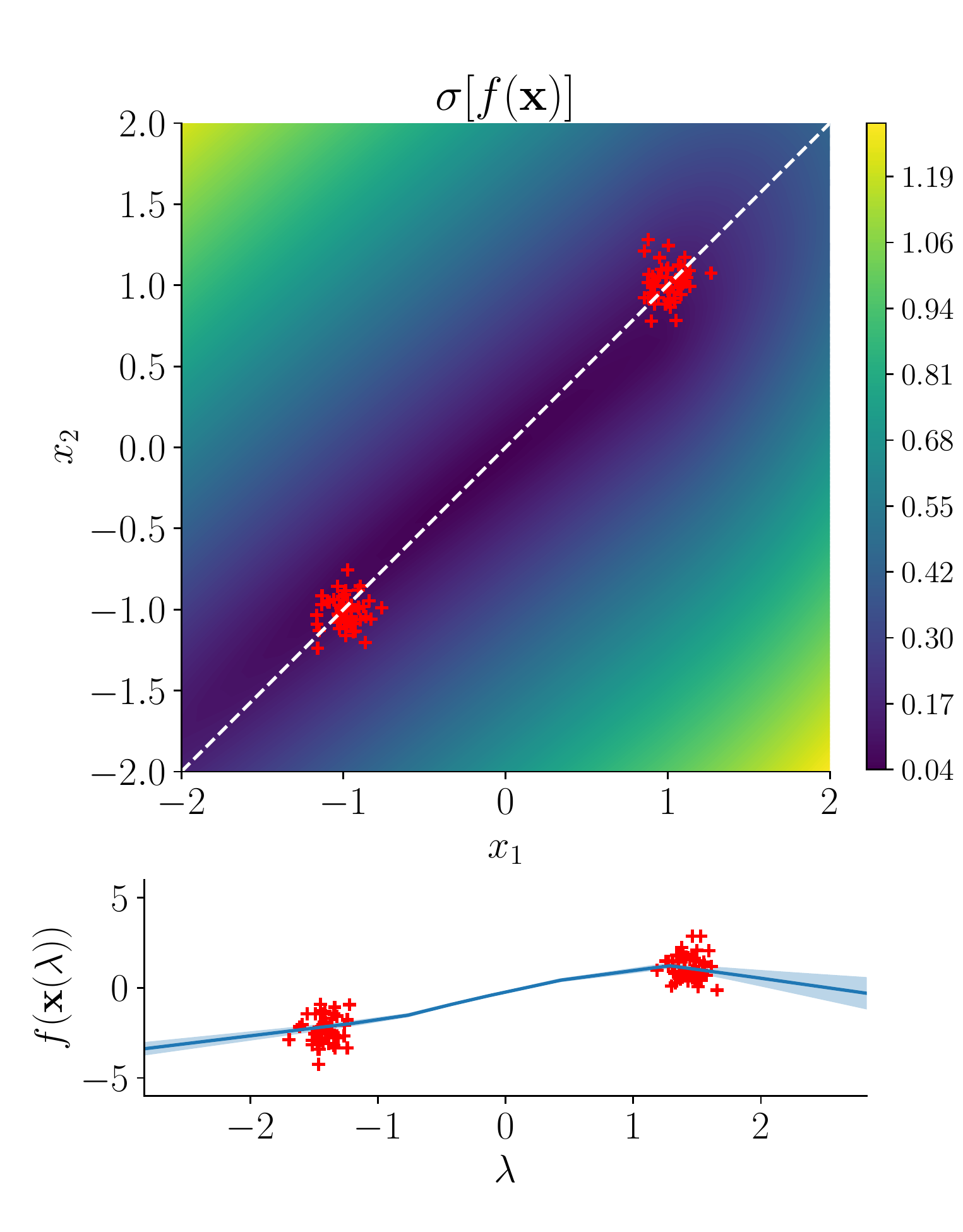}
        \caption{MCDO}
    \end{subfigure}
    \caption{Regression on a 2D synthetic dataset (red crosses). The colour plots show the standard deviation of the output, $\sigma[f(\bfx)]$, in 2D input space. The plots beneath show the mean with 2-standard deviation bars along the dashed white line (parameterised by $\lambda$). MFVI and MCDO are overconfident for $\lambda \in [-1, 1]$. \Cref{thm:mfgaussian,thm:dropout} explain this: given the uncertainty is near zero at the data, there is \emph{no} setting of the variational parameters that has variance much greater than zero in the line segment between them.}
    \label{fig:2D_dataset}
    \vspace{-.3cm}
\end{figure*}
To complete the proof of \cref{thm:mfgaussian}, we need to analyse $\Var[\Exp{\!}{f(\mathbf{x})|\thetainput}]$. Because of the factorisation assumptions on the weights in the first layer, this term is a linear combination of the variances of each activation function. While these variances are not convex, in \cref{app:SHL-proofs} we show they satisfy restrictive conditions that imply bounds on arbitrary positive linear combinations of these functions.  

\subsection{Empirical Tests of Approximate Inference in Single-Hidden Layer BNNs} \label{sec:empirical-tests}
It is not immediately apparent that \cref{thm:mfgaussian,thm:dropout} are problematic from the perspective of Bayesian inference. For example, even exact inference in a Bayesian linear regression model results in a convex predictive variance function. Here we provide strong evidence that, in contrast, the modelling assumptions of 1HL BNNs lead to \emph{exact} posteriors that \emph{do} show in-between uncertainty. \Cref{thm:mfgaussian,thm:dropout} thus imply that it is \emph{approximate} inference with $\Qffg$ or $\Qmcdo$ that fails to reflect this intuitively desirable property of the exact predictive, violating \ref{cond:expressive_family}.

\Cref{fig:2D_dataset} compares the predictive distributions obtained from MFVI and MCDO (here we optimise the ELBO for MFVI and the standard MCDO objective, in contrast with \cref{fig:single-hidden-layer-approximation} --- see \cref{app:details-effect-of-depth} for experimental details) with HMC and the limiting GP on a regression dataset consisting of two clusters of covariates. We use 1HL BNNs with 50 hidden units and ReLU activations. The HMC and limiting GP posteriors are almost indistinguishable, suggesting they both resemble the exact predictive. For these methods $\Var[f(\bfx)]$ is markedly larger near the origin than near the data. In contrast, MFVI and MCDO are as confident in between the data as they are near the data. This provides strong evidence that the lack of in-between uncertainty is not a feature of the BNN model or prior, but is caused by approximate inference.
 
\section{Deeper Networks}\label{sec:deep-stuff}

\Cref{thm:mfgaussian,thm:dropout} pose an important question: is the structural limitation observed in the 1HL case fundamental to $\Qffg$ and $\Qmcdo$ even in deeper networks, or can depth help these approximations satisfy \ref{cond:expressive_family}? In \cref{thm:universal}, we provide universality results for the mean and variance functions of approximate BNNs with at least two hidden layers using $\Qffg$ and $\Qmcdo$. As the predictive mean and variance often determine the performance of BNNs in regression applications, this provides theoretical evidence that approximate inference in \emph{deep} BNNs satisfies \ref{cond:expressive_family}. 

\begin{thm}[Deeper networks]\label{thm:universal}
Let $g$ be any continuous function on a compact set $A \subset \R^D$, and $h$ be any continuous, non-negative function on $A$. For any $\epsilon > 0$, for both $\Qffg$ and $\Qmcdo$ there exists a 2HL ReLU BNN such that $\sup_{\bfx \in A} |\Exp{\!}{f(\bfx)} - g(\bfx)| \!<\! \epsilon $ and $\sup_{\bfx \in A} |\Var[f(\bfx)] - h(\bfx)| \!<\! \epsilon.$
\end{thm}
\begin{rem}
If MC dropout is used with inputs also dropped out, the analogous statement to \cref{thm:universal} is false. In \cref{app:dropout-drop-inputs}, we provide a counterexample that holds for arbitrarily deep networks and shows that if this is the case, $\Var[f]$ cannot be made small at two points $\bfx_1, \bfx_2$ which have significantly different values of $\Exp{\!}{f(\bfx_1)}$ and $\Exp{\!}{f(\bfx_2)}$.
\end{rem}

\Cref{fig:universal} shows the result of directly minimising the squared error between the network output mean and variance and a given target mean and variance function, using the same method and architecture as with the 1HL network in \cref{fig:single-hidden-layer-approximation}. In contrast to \cref{fig:single-hidden-layer-approximation}, the variances of both $\Qffg$ and $\Qmcdo$ are able to fit the target.

 \begin{figure*}[ht]
    \centering
    \includegraphics[width=.48\columnwidth,trim={.1cm .7cm .1cm .25cm}]{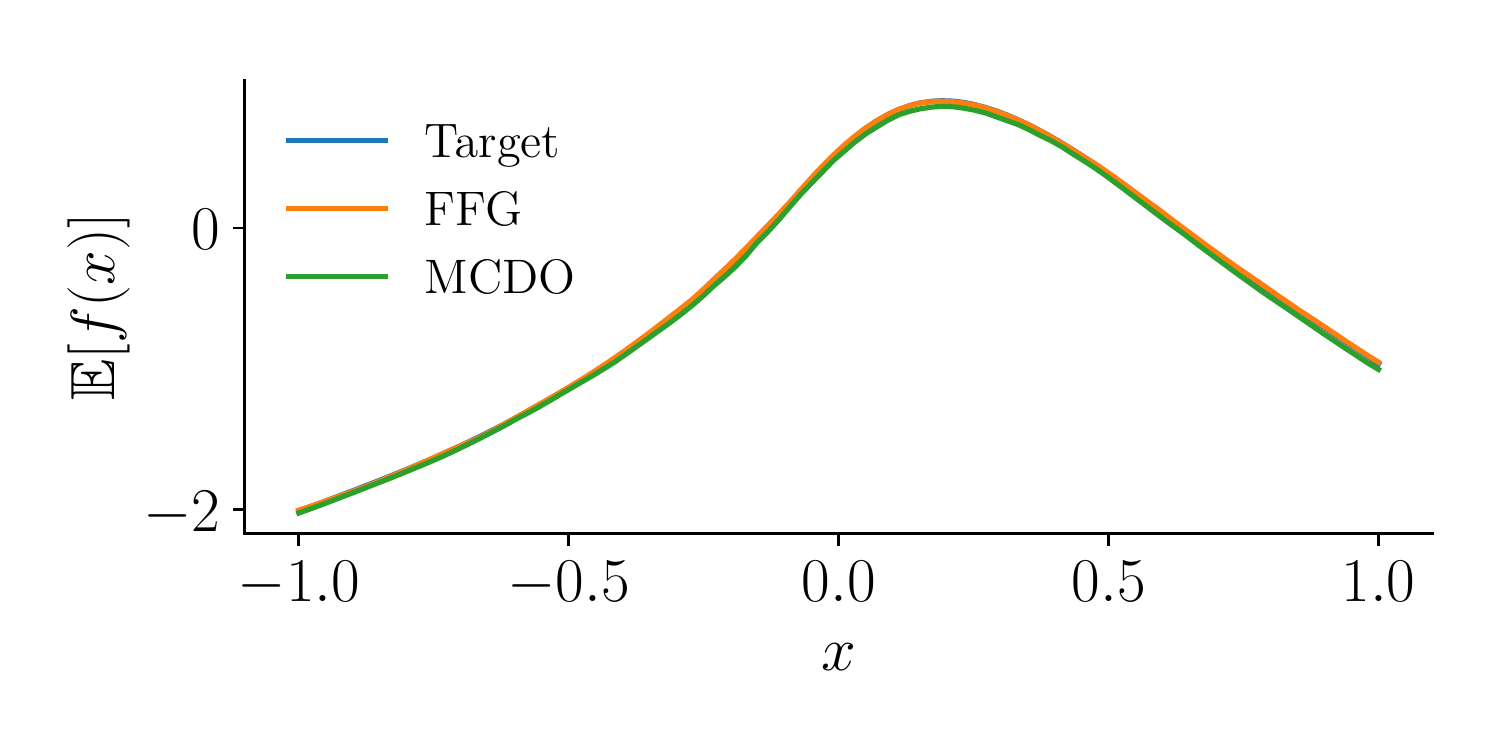}
    \includegraphics[width=.48\columnwidth,trim={.1cm .7cm .1cm .25cm}]{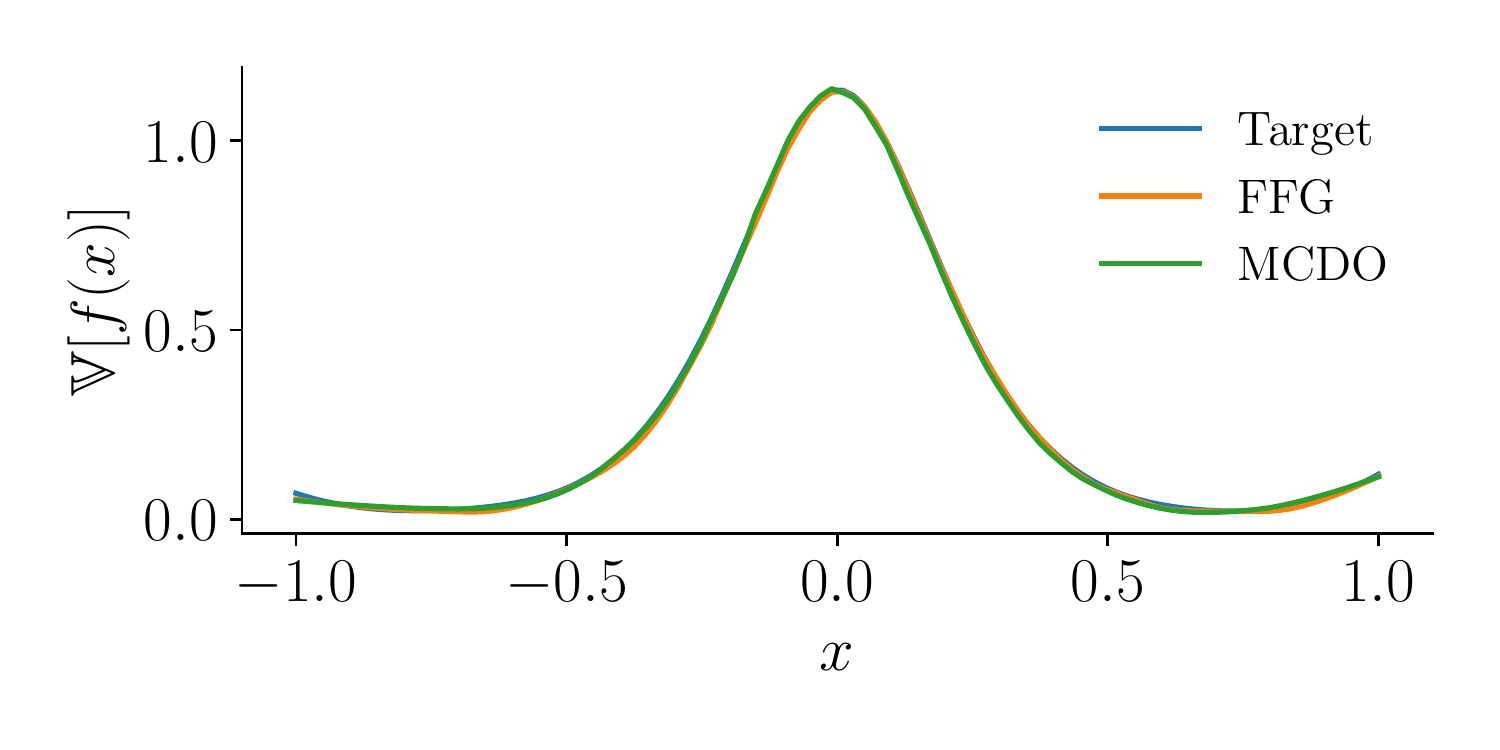}
    \caption{Results of minimising the squared error in function space between $\Exp{\!}{f(x)}$ and a target mean function (left), and between $\Var[f(x)]$ and a target variance function (right), for a 2-hidden layer BNN with FFG and MCDO distributions.}
    \label{fig:universal}
\end{figure*}

While \cref{thm:universal} gives some cause for optimism for approximating family methods with deep BNNs, it only shows that the mean and variance of marginal distributions of the output are universal (it does not tell us about higher moments or covariances between outputs). Additionally, it does not say whether good distributions will actually be \emph{found} by an optimiser when maximising the ELBO, i.e~it does not address \ref{cond:choose_member}.

\subsection{Proof Sketch of Theorem 3}
To prove \cref{thm:universal} for $\Qffg$, we provide a construction that relies on the universal approximation theorem for deterministic NNs \citep{leshno1993multilayer}. Consider a 2HL NN whose second hidden layer has two neurons, with activations $a_1, a_2$. Let $w_1, w_2$ denote the weights connecting $a_1, a_2$ to the output, and $b$ denote the output bias, such that the output $f(\bfx) = w_1\psi(a_1) + w_2\psi(a_2) + b$. In this construction, $a_1$ will be used to control the mean, and $a_2$ the variance, of the BNN output. By setting the variances in the first two linear layers to be sufficiently small, we can consider $a_1$ and $a_2$ to be essentially deterministic functions of $\bfx$. By the universal approximation theorem, $a_1$ and $a_2$ can approximate any continuous functions. Choose $a_1 \!\approx \!g(\bfx)-\min_{\bfx' \in A} g(\bfx')$ and $a_2 \!\approx\! \sqrt{h(\bfx)}$. Choose $\Exp{\!}{b} \!=\! \min_{\bfx' \in A} g(\bfx')$, $\Var[b] \!\approx\! 0$; $\Exp{\!}{w_1} \!=\! 1$, $\Var[w_1] \!\approx\! 0$; and $\Exp{\!}{w_2} \!=\! 0$, $\Var[w_2] \!= \!1$. By linearity of expectation, the factorisation assumptions, and $a_1,a_2\geq 0$:
\begin{align*}
\Exp{\!}{f(\bfx)} &= \Exp{\!}{w_1\psi(a_1)+w_2\psi(a_2)+b} 
=\Exp{\!}{w_1}\Exp{\!}{\psi(a_1)}+\Exp{\!}{w_2}\Exp{\!}{\psi(a_2)} +\Exp{\!}{b} \\
&\approx g(\bfx)-\min_{\bfx' \in A} g(\bfx') + \min_{\bfx' \in A} g(\bfx') =g(\bfx),
\end{align*}
as desired. By the law of total variance, the variance of the network output is \begin{align*}
\Var[f(\bfx)] \!=\!\Exp{\!}{\Var[f(\bfx)|a_1,a_2]}\!+\!\Var[\Exp{\!}{f(\mathbf{x})|a_1,a_2}] \!\approx\! \Exp{\!}{\Var[f(\bfx)|a_1,a_2]}
\!\approx\! \Exp{\!}{\psi(a_2)^2} \!+\! \Var[b] \!\approx\! h(\bfx),\end{align*}
where we used that $w_1, b$ are essentially deterministic and $\Var[\Exp{\!}{f(\mathbf{x})|a_1,a_2}] \approx 0$ since $a_1, a_2$ are essentially deterministic. Also, we have that $\psi(a_2) \approx a_2$ since $a_2 \approx \sqrt{h(\bfx)} \geq 0$. The approximations come from the standard universal function approximation theorem, and the variances of weights not being set exactly to $0$ so that we remain in $\Qffg$. A rigorous proof, along with a proof for $\Qmcdo$ with any dropout rate $p \in (0,1)$, is given in \cref{app:deep-proofs}. The proof for $\Qmcdo$ uses a similar strategy, but is more involved as we cannot set individual weights to be essentially deterministic.

\subsection{Empirical Tests of Approximate Inference in Deep BNNs} \label{sec:deep-empirical}

We now consider empirically whether the distributions found by optimising the ELBO with these families resemble the exact predictive distribution (\ref{cond:choose_member}). 
To do this, we define the `overconfidence ratio' at an input $\bfx$ as $\gamma(\bfx) = (\Var_{\mathrm{GP}}[f(\bfx)]/\Var_{q_\phi}[f(\bfx)])^{1/2}$, where $\Var_{\mathrm{GP}}$ is the predictive variance of exact inference in the infinite-width BNN. We then compute $\gamma(\bfx)$ at 300 points $\{\bfx_n \}_{n=1}^{300}$ evenly spaced along the dashed white line joining the data clusters in \cref{fig:2D_dataset}, i.e., from $\bfx = (-1.2, -1.2)$ to $\bfx = (1.2, 1.2)$. We then create boxplots of the values $\{\gamma(\bfx_n) \}_{n=1}^{300}$ for varying BNN depths, shown in \cref{fig:deep-box}. Accurate inference should lead to similar uncertainty estimates to the limiting GP, i.e.~the boxplot should be tightly centered around $1$ (dashed line). For the 1HL and 2HL BNNs, the GP and HMC agree closely, suggesting both resemble the exact predictive. In contrast, MFVI and MCDO are often an order of magnitude overconfident ($\gamma(\bfx) > 1$) \emph{between} the data clusters (upper tail of the boxplot) and somewhat underconfident ($\gamma(\bfx) < 1$) \emph{at} the data clusters (lower tail of the boxplot). Increased depth does not alleviate this behaviour. See \cref{app:details-effect-of-depth} for experimental details and figures demonstrating this for different priors. In addition, in \cref{app:random-regression}, we plot the uncertainty on line segments in between \emph{random} clusters of data in a 5-dimensional input space, with similar results, showing that this phenomenon is not specific to the dataset from \cref{fig:2D_dataset}. 

\begin{figure}[t]
    \centering
    \includegraphics[width=.5\columnwidth]{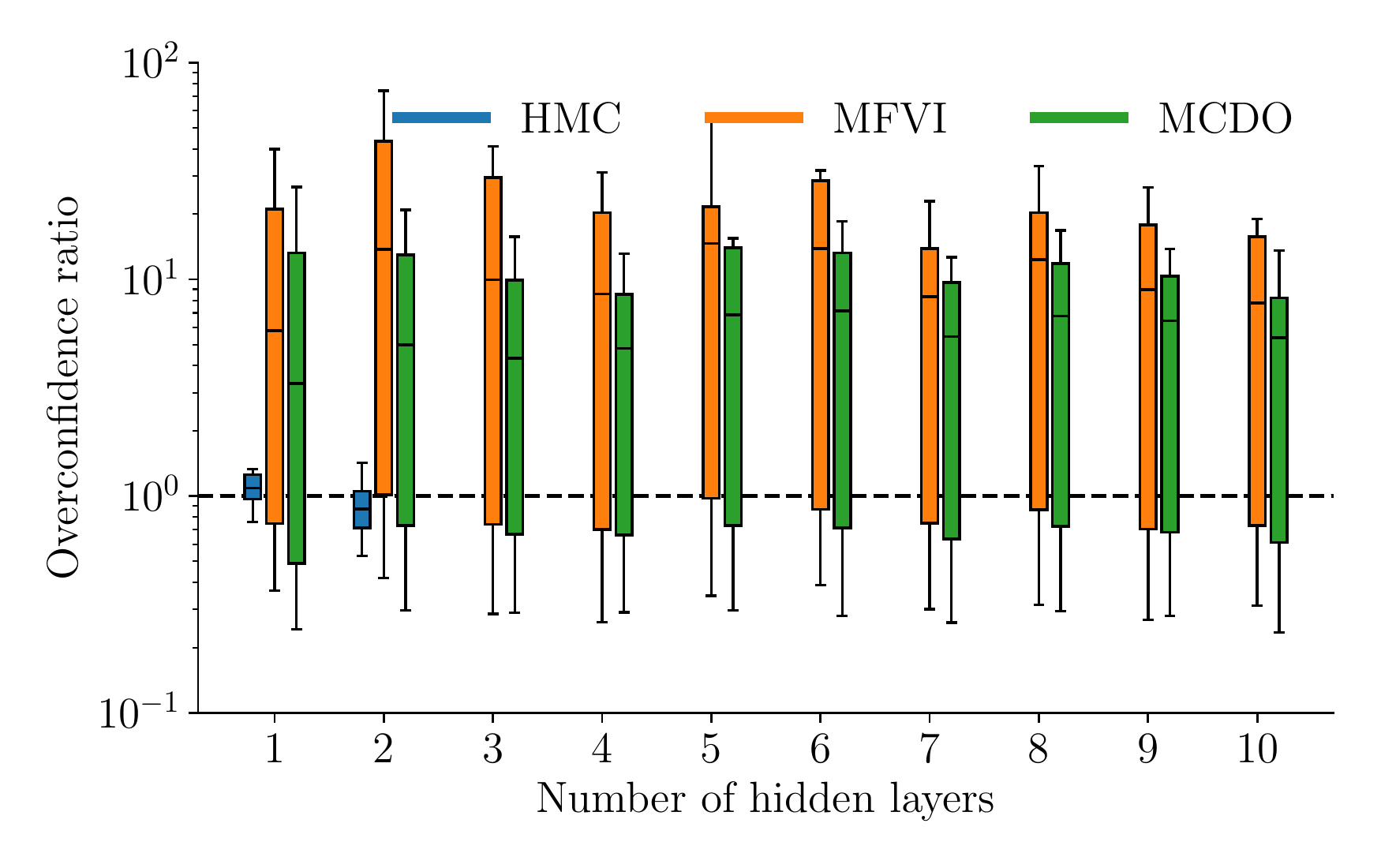}
    \caption{Box and whisker plots of the overconfidence ratios of HMC, MFVI and MCDO  relative to exact inference in the corresponding infinite-width limit GP. The whiskers show the smallest and largest overconfidence ratios computed, and the the box extends from the lower to upper quartile values of the overconfidence ratios, with a line at the median. HMC is only run for 1 and 2 hidden layers due to difficulty ensuring convergence in larger models. We fix the BNN width to 50.}
    \label{fig:deep-box}
    \vspace{-.3cm}
\end{figure}

In light of \cref{thm:universal}, it is perhaps surprising that VI fails to capture important properties of the predictive with deep networks. In \cref{app:initialise-VI} we initialise the variational parameters such that the approximate predictive has mean and variance functions that closely match a reference predictive that exhibits in-between uncertainty. This is done by directly minimising the squared loss between the BNN mean and variance functions and the references. We find that proceeding to optimise the ELBO from this initialisation \emph{still} leads to a lack of in-between uncertainty. This suggests that the objective function is at least partially at fault for the mismatch between the approximate and exact posteriors.

\section{Case Study: Active Learning with BNNs}\label{sec:experiments}
 We now consider the impact of the pathologies described in \cref{sec:SHL-theory,sec:deep-stuff} on active learning \citep{settles2009active} on a real-world dataset, where the task is to use uncertainty information to intelligently select which points to label. Active learning with approximate BNNs has been considered in previous works, often showing improvements over random selection of datapoints \citep{hernandez2015probabilistic,gal2017deep}. However, in cases when active learning fails, common metrics such as RMSE are insufficient to diagnose the causes. In particular, it is difficult to attribute the failure to the model or to poor approximate inference. In this section, we specifically analyse a dataset where we have observed active learning with approximate BNNs to fail --- the Naval regression dataset \citep{Coraddu2013Machine}, which is 14-dimensional and consists of 11,934 datapoints. We find via PCA that this dataset has most of its variance along a single direction. It hence may be especially problematic for methods that struggle with in-between uncertainty, as points are more likely to lie roughly in between others.
 
The main questions we address are: i) Is a lack of in-between uncertainty indicative of pathological behaviour in the 1HL case? In higher dimensional datasets such as Naval, it is not immediately apparent that \cref{thm:dropout,thm:mfgaussian} are problematic, since the convex hull of the datapoints may have low volume in high dimensions. However, these theorems may be symptomatic of \emph{related} problems with 1HL BNN uncertainty estimates. ii) Given \cref{thm:universal}, will deeper approximate BNNs usefully reflect BNN modelling assumptions for active learning? 

\subparagraph{Experimental Set-up and Results}
We compare MFVI, MCDO and the limiting GP. We do not run HMC as this would take too long to wait for convergence at each iteration of active learning. We normalise the dataset to have zero mean and unit standard deviation in each dimension. 5 datapoints are chosen randomly as an initial active set, with the rest being the pool set. Models are trained on the active set, then the datapoint from the pool set that has the highest predictive variance is added to the active set, following \citet{hernandez2015probabilistic}. We train MFVI and MCDO for 20,000 iterations of ADAM at each step of active learning. This process is repeated 50 times. \Cref{tab:active} shows the RMSE of each model on a held-out test set after this process, compared to a baseline where points are chosen randomly. Full details are in \cref{app:details-active-learning}. Active learning significantly reduces RMSE for the GP compared to random selection, often by more than a factor of three. However it \emph{increases} RMSE for 1HL MFVI and MCDO, and either increases it or does not significantly decrease it for deeper networks. The one exception is 3HL MCDO, where active performs about 10\% better than random, which is still far less than the factor of three improvement suggested by exact inference in the infinite-width BNN. 

 \begin{figure*}[t]
    \centering
    \begin{subfigure}[b]{0.24\columnwidth}
        \includegraphics[width=\textwidth]{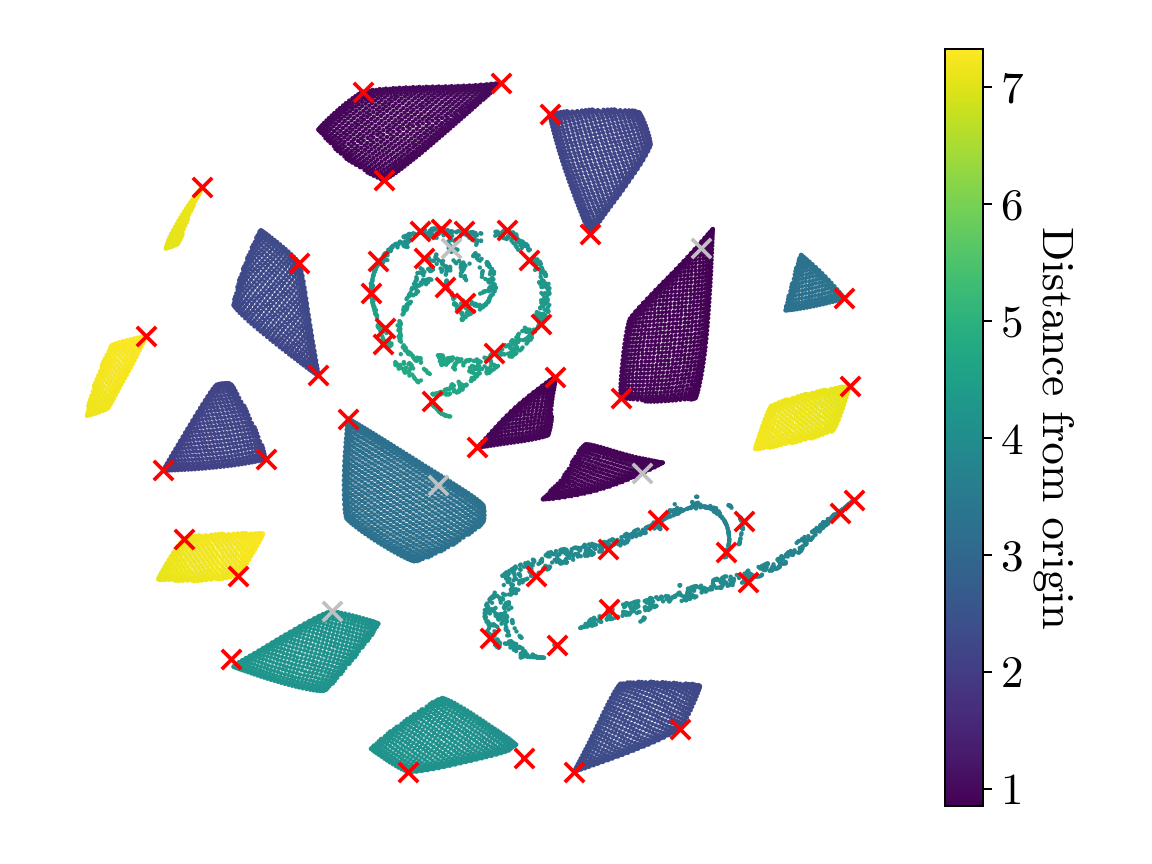}
        \caption{Wide-limit GP}
    \end{subfigure}
    \begin{subfigure}[b]{0.24\columnwidth}
        \includegraphics[width=\textwidth]{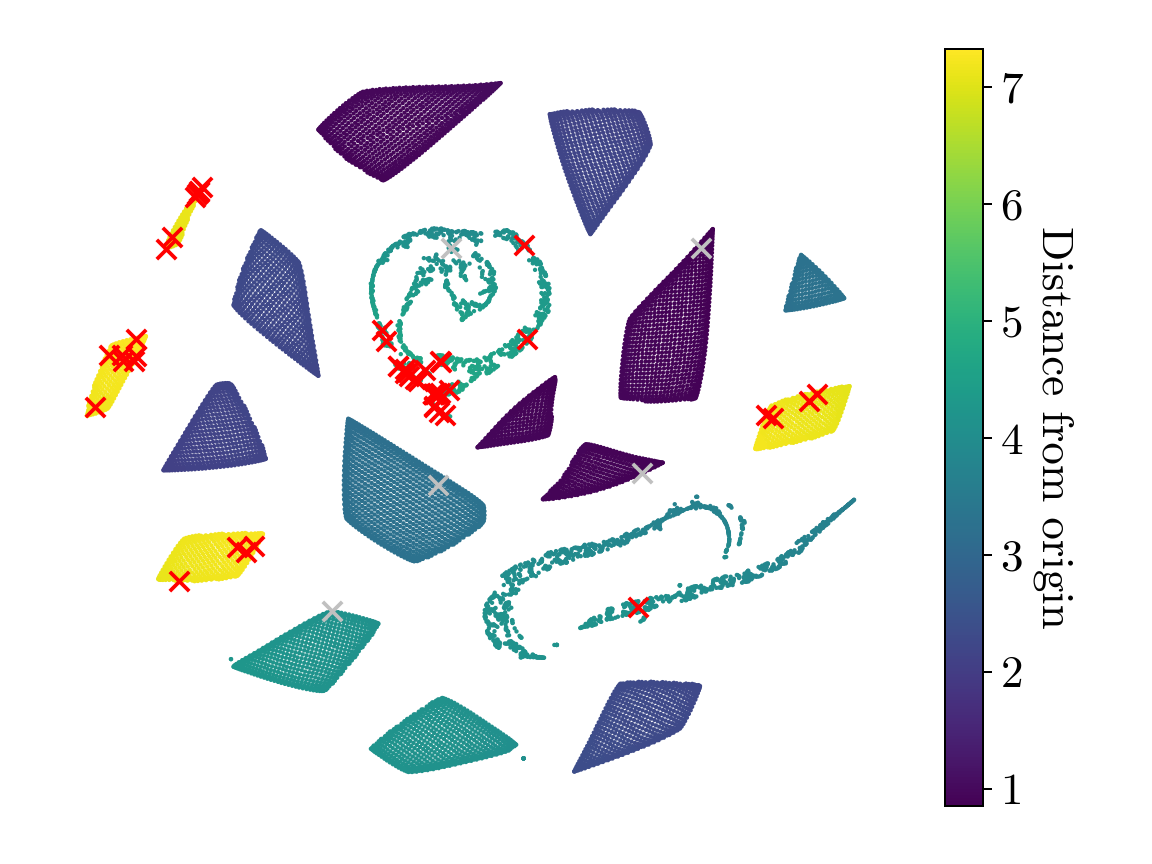}
        \caption{MFVI}
    \end{subfigure}
    \begin{subfigure}[b]{0.24\columnwidth}
        \includegraphics[width=\textwidth]{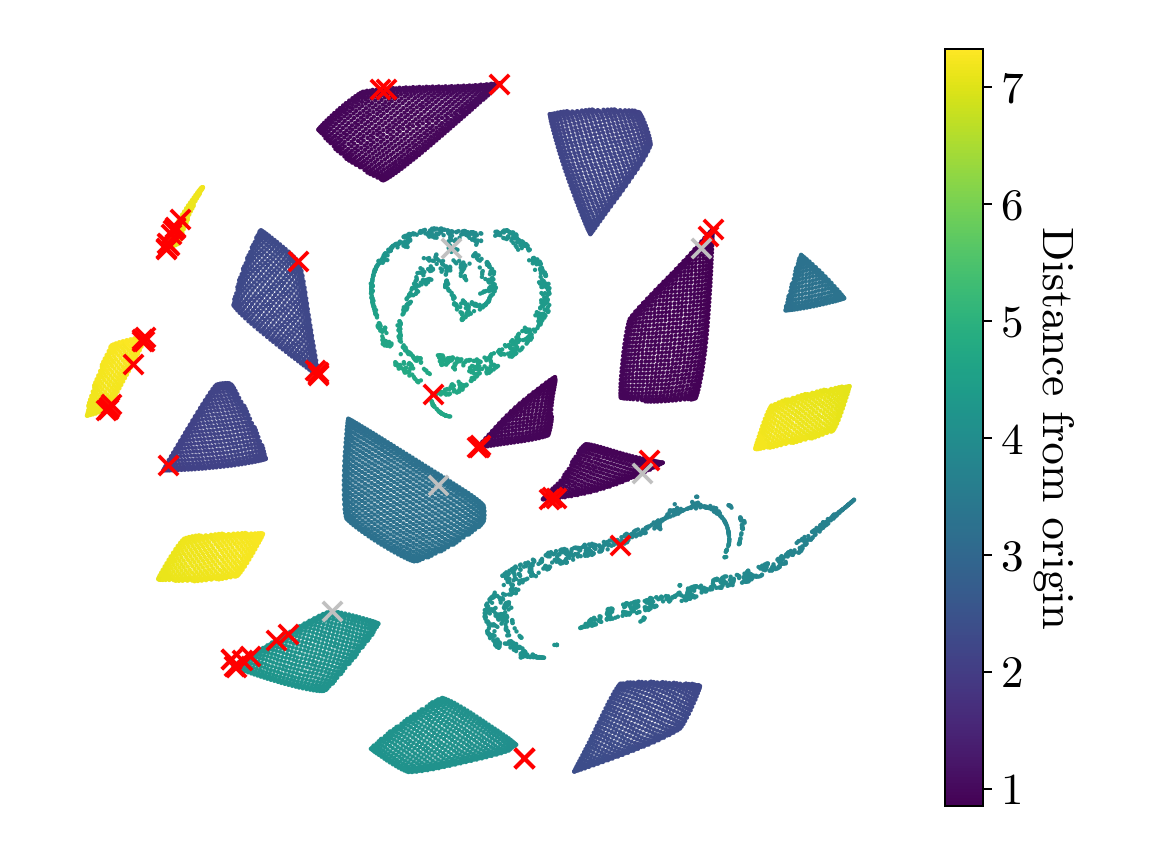}
        \caption{MCDO}
    \end{subfigure}
    \begin{subfigure}[b]{0.24\columnwidth}
        \includegraphics[width=\textwidth]{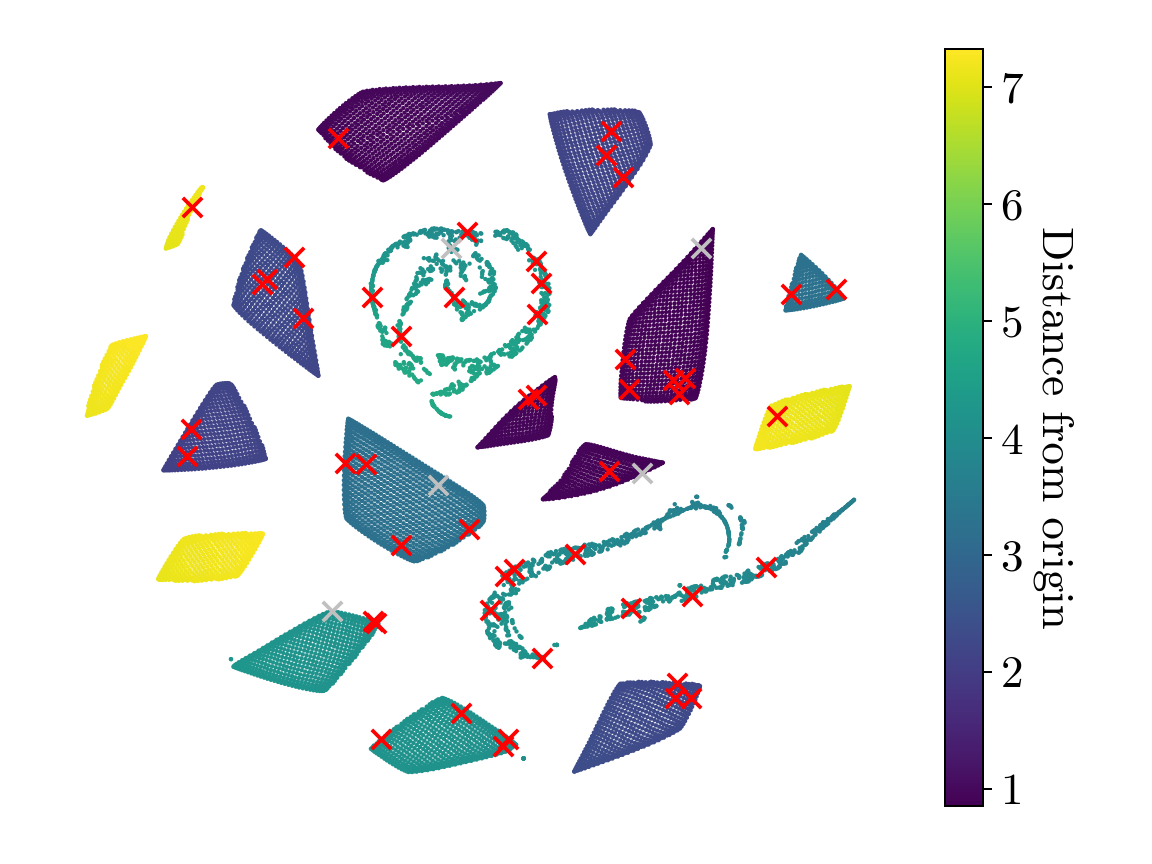}
        \caption{Random}
    \end{subfigure}
    \caption{Points chosen during active learning in the 1HL case. Colours denote distance from the origin in 14-dimensional input space. Grey crosses (\textcolor{gray}{\ding{53}}) denote the five points randomly chosen as an initial training set. Red crosses (\textcolor{red}{\ding{53}}) denote the 50 points selected by active learning. Both MFVI and MCDO entirely miss some clusters which are nearer the origin, and oversample certain clusters which are far from the origin, as might be expected of methods that struggle to represent in-between uncertainty. In contrast, the limiting GP samples the `corners' of each cluster, without missing any entirely. Note that t-SNE does not preserve relative positions, so that clusters near the origin may appear on the ‘outside’ of the t-SNE plot.}
    \label{fig:tsne1}
    \vspace{-.3cm}
\end{figure*}

\begin{table}[b]
    \vspace{-.5cm}
    \centering
    \caption{Test RMSEs ($\pm$ 1 standard error) after the 50\textsuperscript{th} iteration of active learning, averaged over 20 random seeds. As the data is normalised, a method that predicts $0$ will have an RMSE near $1$.}
    \label{tab:active}
   {
    \begin{tabular}{ccccc}
    & 1 HL & 2 HL & 3 HL & 4 HL \\ \hline 
GP Active & $0.04 \pm 0.00$ & $0.04 \pm 0.00$ & $0.04 \pm 0.00$ & $0.05 \pm 0.00$ \\ 
 GP Random & $0.12 \pm 0.01$ & $0.13 \pm 0.01$ & $0.15 \pm 0.01$ & $0.16 \pm 0.01$ \\ 
 MFVI Active& $0.94 \pm 0.11$ & $0.46 \pm 0.04$ & $0.35 \pm 0.03$ & $0.31 \pm 0.02$ \\ 
 MFVI Random & $0.15 \pm 0.01$ & $0.23 \pm 0.01$ & $0.28 \pm 0.01$ & $0.32 \pm 0.01$ \\ 
 MCDO Active & $0.69 \pm 0.04$ & $0.36 \pm 0.02$ & $0.38 \pm 0.02$ & $0.45 \pm 0.02$ \\ 
 MCDO Random & $0.22 \pm 0.01$ & $0.35 \pm 0.01$ & $0.43 \pm 0.01$ & $0.47 \pm 0.02$ \\
    \end{tabular}}
    \vspace{-.2cm}
\end{table}

\subparagraph{Discussion} 
In \cref{fig:tsne1} we visualise the dataset using t-SNE \citep{maaten2008visualizing}. The covariates of Naval are clustered, with points in the same cluster roughly the same distance from the origin. Since the dataset is mean-centred, points closer to the origin are in a sense `in between' others. We see that although the GP chooses points from every cluster during active learning, MFVI fails to select any points from many clusters --- including all the clusters closest to the origin. It ignores points in the `inside' and oversamples points on the `outside', leading to a selection strategy worse than random. This behaviour is consistent with \cref{thm:mfgaussian}. MCDO also fails to sample from many clusters; in \cref{app:figures-active-learning} we show this is because it fails to reduce its uncertainty at clusters it has already heavily sampled. Interestingly, it sometimes chooses from clusters near the origin, even though its variance function is provably convex. This may be because the minimum of the variance function for MCDO is not centred at the origin, or because the variance has the shape of an elongated valley. In contrast, the GP seems to select the `corners' of each cluster, which is intuitively efficient. The success of the infinite-width GP provides evidence that this BNN model combined with exact inference has desirable inductive biases for this task; it is \emph{approximate} inference that has caused active learning to fail. In deeper networks, \cref{thm:universal} gives hope that the BNN predictive variance may be useful for active learning. While we find the problems are indeed less severe than in the 1HL case, MFVI still oversamples points away from the origin compared to those near the origin (see \cref{app:figures-active-learning}).

\section{Related Work}\label{sec:related-work}
Concerns have been raised about the suitability of $\Qffg$ since early work on BNNs. \citet[Figure 1]{mackay1992practical} noted that a full-covariance Gaussian family was needed to obtain predictions with increased uncertainty away from data with the Laplace approximation, although no detailed explanation was provided. The desire to go beyond $\Qffg$ has motivated a great deal of research into more flexible approximating families \citep{barber1998ensemble,louizos2017multiplicative,ritter2018scalable}. However, to our knowledge, \cref{thm:mfgaussian} is the first theoretical result showing that $\Qffg$ can have a pathologically restrictive effect on BNN predictive uncertainties. 

Recently, \Citet{farquhar2020try} argued that MFVI becomes a less restrictive approximation as depth increases in BNNs. 
However, they use different criteria to assess the quality of approximate inference. \Citet{farquhar2020try} use performance on classification benchmarks such as ImageNet and also the KL-divergence between certain Gaussian approximations to HMC samples in weight space to evaluate inference. In contrast, we investigate the resemblance between the function-space predictive distributions of the approximate and exact posteriors with the same prior, and focus on separating the effects of modelling and approximate inference. Additionally, we do not consider posterior tempering, and we use a different scaling for our BNN priors (see \cref{sec:priors_and_references}). Recently, \citet{wenzel2020good} also performed a study of the quality of approximate inference in deep BNNs. They focused on stochastic gradient Markov Chain Monte Carlo (SGMCMC) \citep{welling2011bayesian,chen2014stochastic,zhang2019cyclical} in deep convolutional networks, and concluded that SGMCMC is accurate enough for inference, suggesting that the prior is at fault for poor performance. In contrast, we provide theoretical and empirical evidence for a specific pathology in VI in lower-dimensional regression problems, and demonstrate cases where BNN priors \emph{do} encode useful inductive biases which are subsequently lost by approximate inference.

\Citet{osband2018randomized} note that the MCDO predictive distribution is invariant to duplicates of the data, and in the linear case predictive uncertainty does not decrease as dataset size increases (if the dropout rate and regulariser are fixed\footnote{Note that for a fixed prior, the `KL condition' \citep[Section 3.2.3]{gal2016uncertainty} requires the $\ell_2$ regularisation constant to decrease with increasing dataset size.}). \Cref{thm:dropout} shows that in the non-linear 1HL case, predictive uncertainty in the MCDO posterior is fundamentally restricted even for datasets without repeated entries.

\section{Conclusions} 
Principled approximate Bayesian inference involves defining a reasonable model, then finding an approximate posterior that retains the relevant properties of the exact posterior. We have shown that for BNNs, in-between uncertainty is a feature of the predictive that is often lost by variational inference. Although this is of greatest relevance for lower-dimensional regression tasks, the fact that MFVI and MCDO often fail these simple sanity checks indicates that these methods might generally have predictive distributions which are qualitatively different from the exact predictive. While BNNs have previously been shown to provide uncertainty estimates that are useful for a range of tasks, it remains an open question as to what extent this is attributable to a resemblance between the approximate and exact predictive posteriors. 

\section*{Broader Impact}
Bayesian approaches to deep learning problems are often proposed in situations where uncertainty estimation is critical. Often the justification given for this approach is the probabilistic framework of Bayesian inference. However, in cases where approximations are made, the quality of these approximations should also be taken into account. Our work illustrates that the uncertainty estimates given by approximate inference with commonly used algorithms may not qualitatively resemble the uncertainty estimates implied by Bayesian modelling assumptions. This may possibly have adverse consequences if Bayesian neural networks are used in safety-critical applications. Our work motivates a careful consideration of these situations.

\begin{ack}
We thank Wessel Bruinsma for the proof of lemma 5, and Jos\'e Miguel Hern\'andez-Lobato, Ross Clarke and Sebastian W. Ober for helpful discussions. AYKF gratefully acknowledges funding from the Trinity Hall Research Studentship and the George and Lilian Schiff Foundation. DRB gratefully acknowledges funding from the Herchel Smith Fellowship through Williams College, as well as the Qualcomm Innovation Fellowship. RET is supported by Google, Amazon, ARM, Improbable, EPSRC grants EP/M0269571 and EP/L000776/1, and the UKRI Centre for Doctoral Training in the Application of Artificial Intelligence to the study of Environmental Risks (AI4ER).
\end{ack}

\bibliographystyle{plainnat}
\bibliography{references}

\begin{thebibliography}{47}
\providecommand{\natexlab}[1]{#1}
\providecommand{\url}[1]{\texttt{#1}}
\expandafter\ifx\csname urlstyle\endcsname\relax
  \providecommand{\doi}[1]{doi: #1}\else
  \providecommand{\doi}{doi: \begingroup \urlstyle{rm}\Url}\fi

\bibitem[Abramowitz and Stegun(1965)]{abramowitz1965handbook}
Milton Abramowitz and Irene~A Stegun.
\newblock \emph{Handbook of mathematical functions: with formulas, graphs, and
  mathematical tables}, volume~55.
\newblock Courier Corporation, 1965.

\bibitem[Barber and Bishop(1998)]{barber1998ensemble}
David Barber and Christopher~M Bishop.
\newblock Ensemble learning in {B}ayesian neural networks.
\newblock \emph{Neural networks and machine learning}, 168:\penalty0 215--237,
  1998.

\bibitem[Beal(2003)]{beal2003variational}
Matthew~James Beal.
\newblock \emph{Variational algorithms for approximate {B}ayesian inference}.
\newblock PhD thesis, {University College London}, 2003.

\bibitem[Bingham et~al.(2018)Bingham, Chen, Jankowiak, Obermeyer, Pradhan,
  Karaletsos, Singh, Szerlip, Horsfall, and Goodman]{bingham2018pyro}
Eli Bingham, Jonathan~P. Chen, Martin Jankowiak, Fritz Obermeyer, Neeraj
  Pradhan, Theofanis Karaletsos, Rohit Singh, Paul Szerlip, Paul Horsfall, and
  Noah~D. Goodman.
\newblock Pyro: Deep universal probabilistic programming.
\newblock \emph{Journal of Machine Learning Research (JMLR)}, 2018.

\bibitem[Blei et~al.(2017)Blei, Kucukelbir, and McAuliffe]{blei2017variational}
David~M Blei, Alp Kucukelbir, and Jon~D McAuliffe.
\newblock Variational inference: A review for statisticians.
\newblock \emph{Journal of the American Statistical Association}, 112\penalty0
  (518):\penalty0 859--877, 2017.

\bibitem[Blundell et~al.(2015)Blundell, Cornebise, Kavukcuoglu, and
  Wierstra]{blundell2015weight}
Charles Blundell, Julien Cornebise, Koray Kavukcuoglu, and Daan Wierstra.
\newblock Weight uncertainty in neural networks.
\newblock In \emph{International Conference on Machine Learning (ICML)}, 2015.

\bibitem[Chen et~al.(2014)Chen, Fox, and Guestrin]{chen2014stochastic}
Tianqi Chen, Emily Fox, and Carlos Guestrin.
\newblock Stochastic gradient {H}amiltonian {M}onte {C}arlo.
\newblock In \emph{International Conference on Machine Learning}, pages
  1683--1691, 2014.

\bibitem[Coraddu et~al.(2014)Coraddu, Oneto, Ghio, Savio, Anguita, and
  Figari]{Coraddu2013Machine}
Andrea Coraddu, Luca Oneto, Alessandro Ghio, Stefano Savio, Davide Anguita, and
  Massimo Figari.
\newblock Machine learning approaches for improving condition-based maintenance
  of naval propulsion plants.
\newblock \emph{Journal of Engineering for the Maritime Environment}, 2014.

\bibitem[Denker and Lecun(1991)]{denker1991transforming}
John~S Denker and Yann Lecun.
\newblock Transforming neural-net output levels to probability distributions.
\newblock In \emph{Advances in Neural Information Processing Systems (NIPS)},
  1991.

\bibitem[Farquhar et~al.(2020)Farquhar, Smith, and Gal]{farquhar2020try}
Sebastian Farquhar, Lewis Smith, and Yarin Gal.
\newblock Liberty or depth: Deep {B}ayesian neural nets do not need complex
  weight posterior approximations.
\newblock \emph{arXiv preprint arXiv:2002.03704}, 2020.

\bibitem[Filos et~al.(2019)Filos, Farquhar, Gomez, Rudner, Kenton, Smith,
  Alizadeh, de~Kroon, and Gal]{oatml2019bdlb}
Angelos Filos, Sebastian Farquhar, Aidan~N. Gomez, Tim G.~J. Rudner, Zachary
  Kenton, Lewis Smith, Milad Alizadeh, Arnoud de~Kroon, and Yarin Gal.
\newblock Benchmarking {B}ayesian deep learning with diabetic retinopathy
  diagnosis.
\newblock \url{https://github.com/OATML/bdl-benchmarks}, 2019.

\bibitem[Frey and Hinton(1999)]{frey1999variational}
Brendan~J Frey and Geoffrey~E Hinton.
\newblock Variational learning in nonlinear {G}aussian belief networks.
\newblock \emph{Neural Computation}, 11\penalty0 (1):\penalty0 193--213, 1999.

\bibitem[Gal(2016)]{gal2016uncertainty}
Yarin Gal.
\newblock \emph{Uncertainty in deep learning}.
\newblock PhD thesis, University of {Cambridge}, 2016.

\bibitem[Gal and Ghahramani(2016)]{gal2016dropout}
Yarin Gal and Zoubin Ghahramani.
\newblock Dropout as a {B}ayesian approximation: Representing model uncertainty
  in deep learning.
\newblock In \emph{International Conference on Machine Learning (ICML)}, 2016.

\bibitem[Gal et~al.(2017)Gal, Islam, and Ghahramani]{gal2017deep}
Yarin Gal, Riashat Islam, and Zoubin Ghahramani.
\newblock Deep {B}ayesian active learning with image data.
\newblock In \emph{International Conference on Machine Learning (ICML)}, 2017.

\bibitem[Hern{\'a}ndez-Lobato and Adams(2015)]{hernandez2015probabilistic}
Jos{\'e}~Miguel Hern{\'a}ndez-Lobato and Ryan Adams.
\newblock Probabilistic backpropagation for scalable learning of {B}ayesian
  neural networks.
\newblock In \emph{International Conference on Machine Learning (ICML)}, 2015.

\bibitem[Hern{\'a}ndez-Lobato et~al.(2016)Hern{\'a}ndez-Lobato, Li, Rowland,
  Bui, Hern{\'a}ndez-Lobato, and Turner]{pmlr-v48-hernandez-lobatob16}
Jos{\'e}~Miguel Hern{\'a}ndez-Lobato, Yingzhen Li, Mark Rowland, Thang Bui,
  Daniel Hern{\'a}ndez-Lobato, and Richard Turner.
\newblock Black-box $\alpha$-divergence minimization.
\newblock In \emph{International Conference on Machine Learning (ICML)}, 2016.

\bibitem[Hinton and Van~Camp(1993)]{hinton1993keeping}
Geoffrey~E Hinton and Drew Van~Camp.
\newblock Keeping the neural networks simple by minimizing the description
  length of the weights.
\newblock In \emph{Conference on Computational learning theory (COLT)}, 1993.

\bibitem[Hoffman and Gelman(2014)]{hoffman2014no}
Matthew~D Hoffman and Andrew Gelman.
\newblock The {No-U-Turn} sampler: adaptively setting path lengths in
  {Hamiltonian Monte Carlo}.
\newblock \emph{Journal of Machine Learning Research (JMLR)}, 15\penalty0
  (1):\penalty0 1593--1623, 2014.

\bibitem[Hron et~al.(2018)Hron, Matthews, and Ghahramani]{hron2018variational}
Jiri Hron, Alex Matthews, and Zoubin Ghahramani.
\newblock Variational {B}ayesian dropout: Pitfalls and fixes.
\newblock In \emph{International Conference on Machine Learning (ICML)}, 2018.

\bibitem[Hron et~al.(2020)Hron, Bahri, Novak, Pennington, and
  Sohl-Dickstein]{hron2020exact}
Jiri Hron, Yasaman Bahri, Roman Novak, Jeffrey Pennington, and Jascha
  Sohl-Dickstein.
\newblock Exact posterior distributions of wide {B}ayesian neural networks.
\newblock In \emph{Uncertainty in deep learning Workshop, {ICML}.}, 2020.

\bibitem[Jordan et~al.(1999)Jordan, Ghahramani, Jaakkola, and
  Saul]{jordan1999introduction}
Michael~I Jordan, Zoubin Ghahramani, Tommi~S Jaakkola, and Lawrence~K Saul.
\newblock An introduction to variational methods for graphical models.
\newblock \emph{Machine Learning}, 37\penalty0 (2):\penalty0 183--233, 1999.

\bibitem[Khan et~al.(2018)Khan, Nielsen, Tangkaratt, Lin, Gal, and
  Srivastava]{khan2018fast}
Mohammad~Emtiyaz Khan, Didrik Nielsen, Voot Tangkaratt, Wu~Lin, Yarin Gal, and
  Akash Srivastava.
\newblock Fast and scalable {B}ayesian deep learning by weight-perturbation in
  {A}dam.
\newblock \emph{International Conference on Machine Learning (ICML)}, 2018.

\bibitem[Kingma et~al.(2015)Kingma, Salimans, and
  Welling]{kingma2015variational}
Durk~P Kingma, Tim Salimans, and Max Welling.
\newblock Variational dropout and the local reparameterization trick.
\newblock In \emph{Advances in Neural Information Processing Systems (NIPS)},
  pages 2575--2583, 2015.

\bibitem[Lee et~al.(2018)Lee, Sohl-{D}ickstein, Pennington, Novak, Schoenholz,
  and Bahri]{lee2017deep}
Jaehoon Lee, Jascha Sohl-{D}ickstein, Jeffrey Pennington, Roman Novak, Sam
  Schoenholz, and Yasaman Bahri.
\newblock Deep neural networks as {G}aussian processes.
\newblock In \emph{International Conference on Learning Representations
  (ICLR)}, 2018.

\bibitem[Leshno et~al.(1993)Leshno, Lin, Pinkus, and
  Schocken]{leshno1993multilayer}
Moshe Leshno, Vladimir~Ya Lin, Allan Pinkus, and Shimon Schocken.
\newblock Multilayer feedforward networks with a nonpolynomial activation
  function can approximate any function.
\newblock \emph{Neural Networks}, 6\penalty0 (6):\penalty0 861--867, 1993.

\bibitem[Li and Turner(2016)]{li2016renyi}
Yingzhen Li and Richard~E Turner.
\newblock R{\'e}nyi divergence variational inference.
\newblock In \emph{Advances in Neural Information Processing Systems (NIPS)},
  pages 1073--1081, 2016.

\bibitem[Li et~al.(2015)Li, Hern{\'a}ndez-Lobato, and Turner]{li2015stochastic}
Yingzhen Li, Jos{\'e}~Miguel Hern{\'a}ndez-Lobato, and Richard~E Turner.
\newblock Stochastic expectation propagation.
\newblock In \emph{Advances in Neural Information Processing Systems (NIPS)},
  pages 2323--2331, 2015.

\bibitem[Louizos and Welling(2017)]{louizos2017multiplicative}
Christos Louizos and Max Welling.
\newblock Multiplicative normalizing flows for variational {B}ayesian neural
  networks.
\newblock In \emph{International Conference on Machine Learning (ICML)}, 2017.

\bibitem[MacKay(1992)]{mackay1992practical}
David J.~C. MacKay.
\newblock A practical {B}ayesian framework for backpropagation networks.
\newblock \emph{Neural Computation}, 4\penalty0 (3):\penalty0 448--472, 1992.

\bibitem[Matthews et~al.(2017)Matthews, {van der Wilk}, Nickson, Fujii,
  {Boukouvalas}, {Le{\'o}n-Villagr{\'a}}, Ghahramani, and Hensman]{GPflow2017}
Alexander G. de~G. Matthews, Mark {van der Wilk}, Tom Nickson, Keisuke. Fujii,
  Alexis {Boukouvalas}, Pablo {Le{\'o}n-Villagr{\'a}}, Zoubin Ghahramani, and
  James Hensman.
\newblock {{GP}flow: A {G}aussian process library using {T}ensor{F}low}.
\newblock \emph{Journal of Machine Learning Research (JMLR)}, 18\penalty0
  (40):\penalty0 1--6, 2017.

\bibitem[Matthews et~al.(2018)Matthews, Hron, Rowland, Turner, and
  Ghahramani]{matthews2018gaussian}
{Alexander G. de G.} Matthews, Jiri Hron, Mark Rowland, Richard~E. Turner, and
  Zoubin Ghahramani.
\newblock Gaussian process behaviour in wide deep neural networks.
\newblock In \emph{International Conference on Learning Representations
  (ICLR)}, 2018.

\bibitem[Mukhoti et~al.(2018)Mukhoti, Stenetorp, and
  Gal]{mukhoti2018importance}
Jishnu Mukhoti, Pontus Stenetorp, and Yarin Gal.
\newblock On the importance of strong baselines in {B}ayesian deep learning.
\newblock \emph{arXiv preprint arXiv:1811.09385}, 2018.

\bibitem[Neal(1995)]{neal2012bayesian}
Radford~M Neal.
\newblock \emph{{B}ayesian learning for neural networks}.
\newblock PhD thesis, University of Toronto, 1995.

\bibitem[Neal et~al.(2011)]{neal2011mcmc}
Radford~M Neal et~al.
\newblock {MCMC} using {H}amiltonian dynamics.
\newblock \emph{Handbook of Markov chain Monte Carlo}, 2\penalty0
  (11):\penalty0 2, 2011.

\bibitem[Osband et~al.(2018)Osband, Aslanides, and
  Cassirer]{osband2018randomized}
Ian Osband, John Aslanides, and Albin Cassirer.
\newblock Randomized prior functions for deep reinforcement learning.
\newblock In \emph{Advances in Neural Information Processing Systems
  (NeurIPS)}, pages 8617--8629, 2018.

\bibitem[Ovadia et~al.(2019)Ovadia, Fertig, Ren, Nado, Sculley, Nowozin,
  Dillon, Lakshminarayanan, and Snoek]{ovadia2019can}
Yaniv Ovadia, Emily Fertig, Jie Ren, Zachary Nado, D~Sculley, Sebastian
  Nowozin, Joshua~V Dillon, Balaji Lakshminarayanan, and Jasper Snoek.
\newblock Can you trust your model's uncertainty? {E}valuating predictive
  uncertainty under dataset shift.
\newblock In \emph{Advances in Neural Information Processing Systems
  (NeurIPS)}, 2019.

\bibitem[Ritter et~al.(2018)Ritter, Botev, and Barber]{ritter2018scalable}
Hippolyt Ritter, Aleksandar Botev, and David Barber.
\newblock A scalable {L}aplace approximation for neural networks.
\newblock In \emph{International Conference on Learning Representations
  (ICLR)}, 2018.

\bibitem[Schoenholz et~al.(2017)Schoenholz, Gilmer, Ganguli, and
  Sohl-{D}ickstein]{schoenholz2016deep}
Samuel~S Schoenholz, Justin Gilmer, Surya Ganguli, and Jascha Sohl-{D}ickstein.
\newblock Deep information propagation.
\newblock In \emph{International Conference on Learning Representations
  (ICLR)}, 2017.

\bibitem[Settles(2009)]{settles2009active}
Burr Settles.
\newblock Active learning literature survey.
\newblock Technical report, University of Wisconsin-Madison Department of
  Computer Sciences, 2009.

\bibitem[Sun et~al.(2019)Sun, Zhang, Shi, and Grosse]{sun2019functional}
Shengyang Sun, Guodong Zhang, Jiaxin Shi, and Roger Grosse.
\newblock Functional variational {B}ayesian neural networks.
\newblock In \emph{International Conference on Learning Representations
  (ICLR)}, 2019.

\bibitem[Swaroop et~al.(2019)Swaroop, Nguyen, Bui, and
  Turner]{swaroop2019improving}
Siddharth Swaroop, Cuong~V Nguyen, Thang~D Bui, and Richard~E Turner.
\newblock Improving and understanding variational continual learning.
\newblock \emph{arXiv preprint arXiv:1905.02099}, 2019.

\bibitem[Tomczak et~al.(2018)Tomczak, Swaroop, and Turner]{tomczak2018neural}
Marcin~B Tomczak, Siddharth Swaroop, and Richard~E Turner.
\newblock Neural network ensembles and variational inference revisited.
\newblock In \emph{1st Symposium on Advances in Approximate Bayesian Inference
  (AABI)}, 2018.

\bibitem[van~der Maaten and Hinton(2008)]{maaten2008visualizing}
Laurens van~der Maaten and Geoffrey Hinton.
\newblock Visualizing data using {t-SNE}.
\newblock \emph{Journal of Machine Learning Research ({JMLR})}, 9:\penalty0
  2579--2605, 2008.

\bibitem[Welling and Teh(2011)]{welling2011bayesian}
Max Welling and Yee~W Teh.
\newblock Bayesian learning via stochastic gradient {L}angevin dynamics.
\newblock In \emph{Proceedings of the 28th international conference on machine
  learning (ICML-11)}, pages 681--688, 2011.

\bibitem[Wenzel et~al.(2020)Wenzel, Roth, Veeling, {\'S}wi{\k{a}}tkowski, Tran,
  Mandt, Snoek, Salimans, Jenatton, and Nowozin]{wenzel2020good}
Florian Wenzel, Kevin Roth, Bastiaan~S Veeling, Jakub {\'S}wi{\k{a}}tkowski,
  Linh Tran, Stephan Mandt, Jasper Snoek, Tim Salimans, Rodolphe Jenatton, and
  Sebastian Nowozin.
\newblock How good is the {B}ayes posterior in deep neural networks really?
\newblock In \emph{International Conference on Machine Learning (ICML)}, 2020.

\bibitem[Zhang et~al.(2020)Zhang, Li, Zhang, Chen, and
  Wilson]{zhang2019cyclical}
Ruqi Zhang, Chunyuan Li, Jianyi Zhang, Changyou Chen, and Andrew~Gordon Wilson.
\newblock Cyclical stochastic gradient {MCMC} for {B}ayesian deep learning.
\newblock In \emph{International Conference on Learning Representations
  (ICLR)}, 2020.

\end{thebibliography}

\clearpage
\appendix

\section{In-between Uncertainty in Other Regions of Input Space}\label{app:random-regression}

In this appendix, we show plots generated by placing two Gaussian clusters of data with centers randomly chosen on the sphere of radius $\sqrt{D}$, where $D=5$ denotes input dimension. We generate synthetic data by sampling from a wide-limit Gaussian process. For each plot, we show the predictive mean and 2 standard deviations along the line segment in input space joining the centres of these two clusters. For all plots, we choose $\sigma_w = \sqrt{2}$, $\sigma_b = 1$, networks of width 50 and dropout probability of $p=0.05$ for MCDO. We set the observation noise standard deviation to 0.01, which is the ground truth value used to generate the synthetic data. The initialisation of MFVI and MCDO is the same as discussed in \cref{app:details-effect-of-depth}.

In the 1HL case, \Cref{thm:dropout-general} implies that MCDO's predictive variance will be convex along any line, including the line plotted. In contrast, \cref{thm:mfgaussian-general} only applies to certain lines in input space, and does not bound the variance in these cases. However, we suspect that \cref{thm:mfgaussian-general} is indicative of a lack of in-between uncertainty in more general cases. \Cref{fig:1HL_random_regression} shows that although that exact inference with the GP with the limiting BNN prior exhibits in-between uncertainty, this is lost by both MFVI and MCDO, even on general lines in input space. MFVI and MCDO are often \emph{more} confident in between the data clusters than at the data clusters.

We next run the same experiment but with 3HL BNNs and their limiting GP. In this case \cref{thm:universal-supplement} implies that for sufficiently wide BNNs, there exist variational parameters that can approximate \emph{any} predictive mean and standard deviation. However, in \cref{fig:3HL_random_regression} we see that compared to exact inference in the limiting GP, MFVI and MCDO both underestimate in-between uncertainty --- or sometimes show as large uncertainty \emph{at} the data as \emph{in between} the data. 

\begin{figure}[h] 
\begin{tabular}{ccc}
GP & MFVI  & MCDO \\
    \includegraphics[trim =.15cm 0cm 0cm 0cm, clip, width=.30\textwidth]{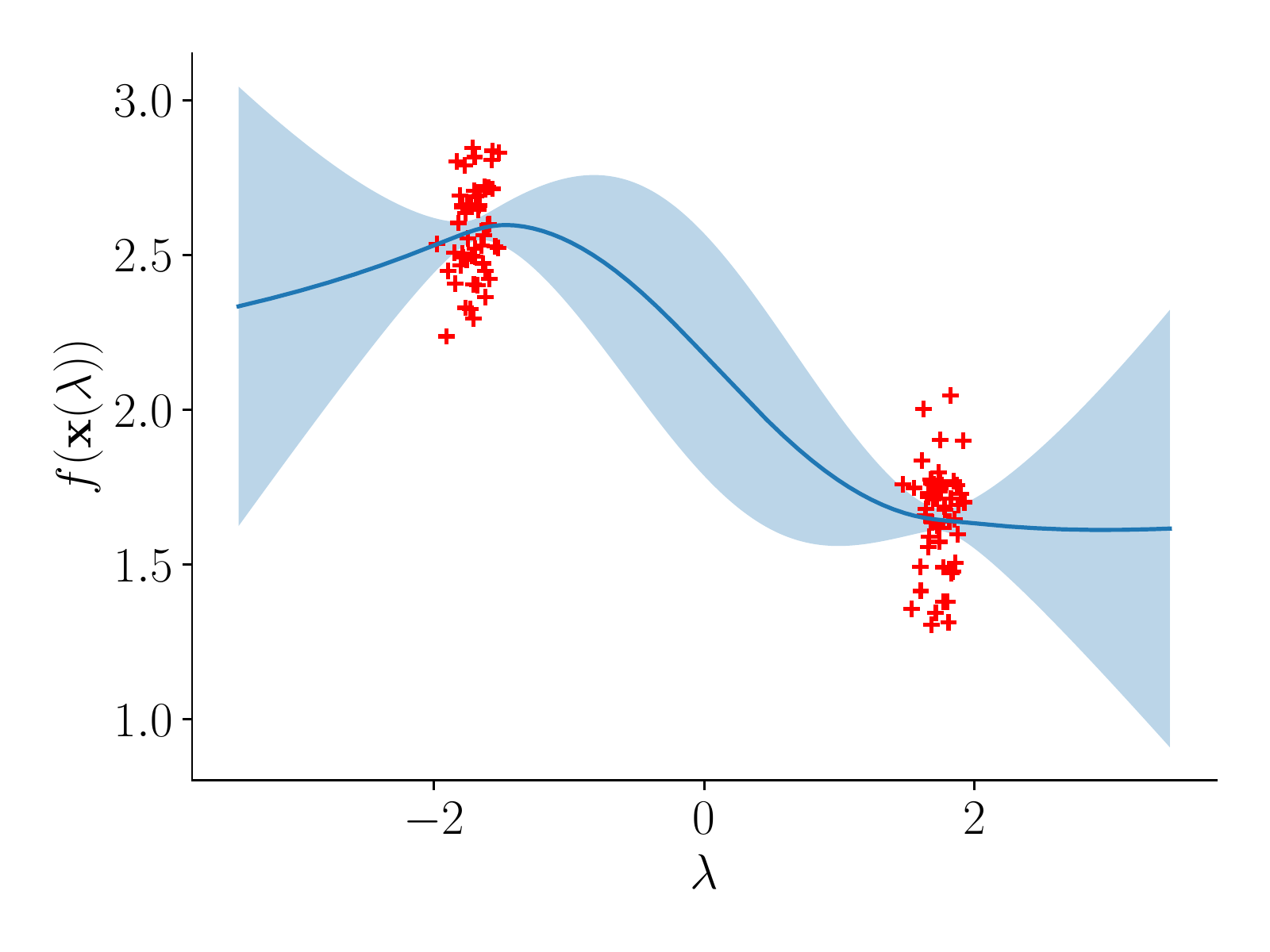} &\includegraphics[trim=.15cm 0cm 0cm 0cm, clip, width=.30\textwidth]{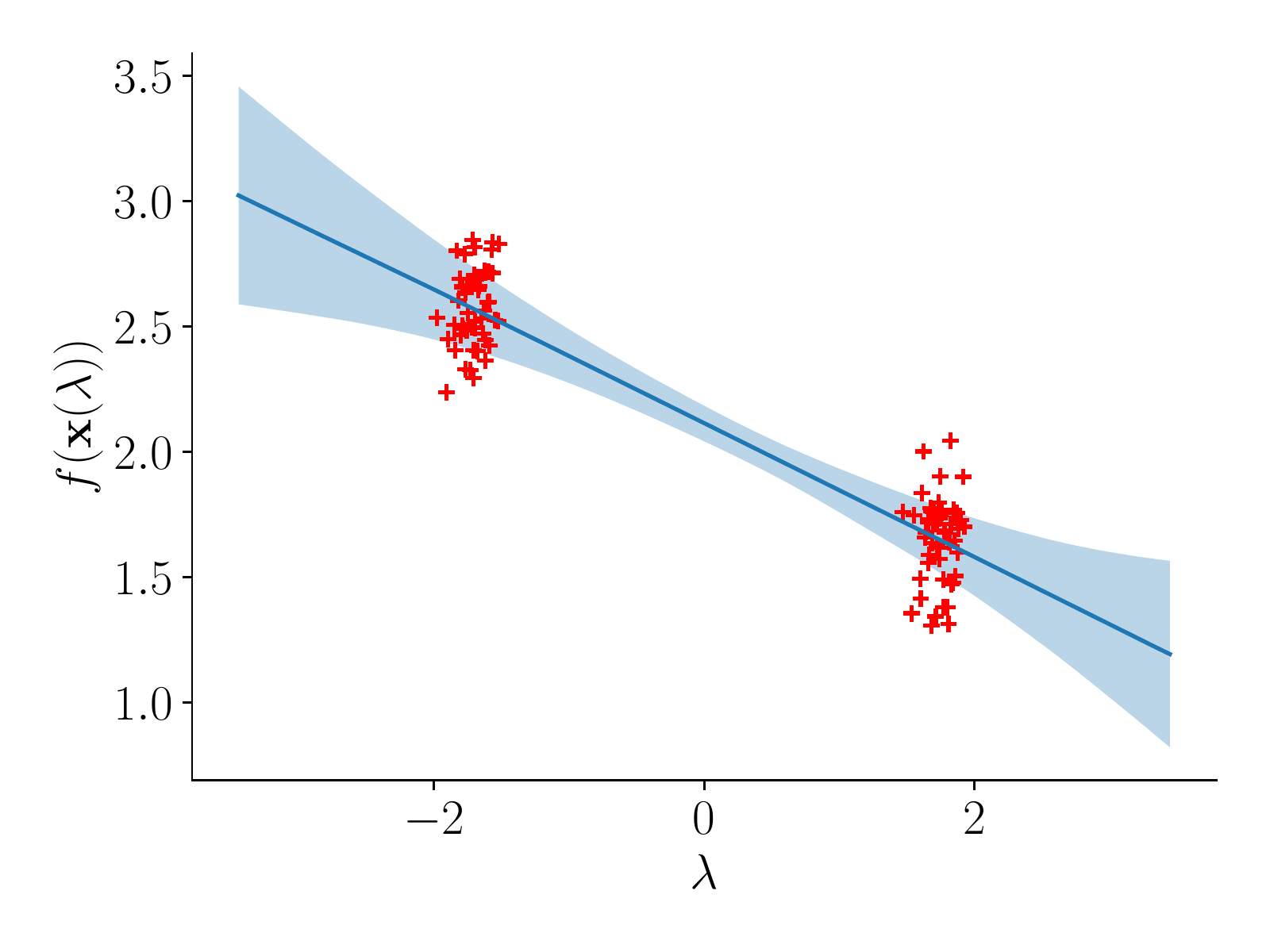}  &\includegraphics[trim=.15cm 0cm 0cm 0cm, clip, width=.30\textwidth]{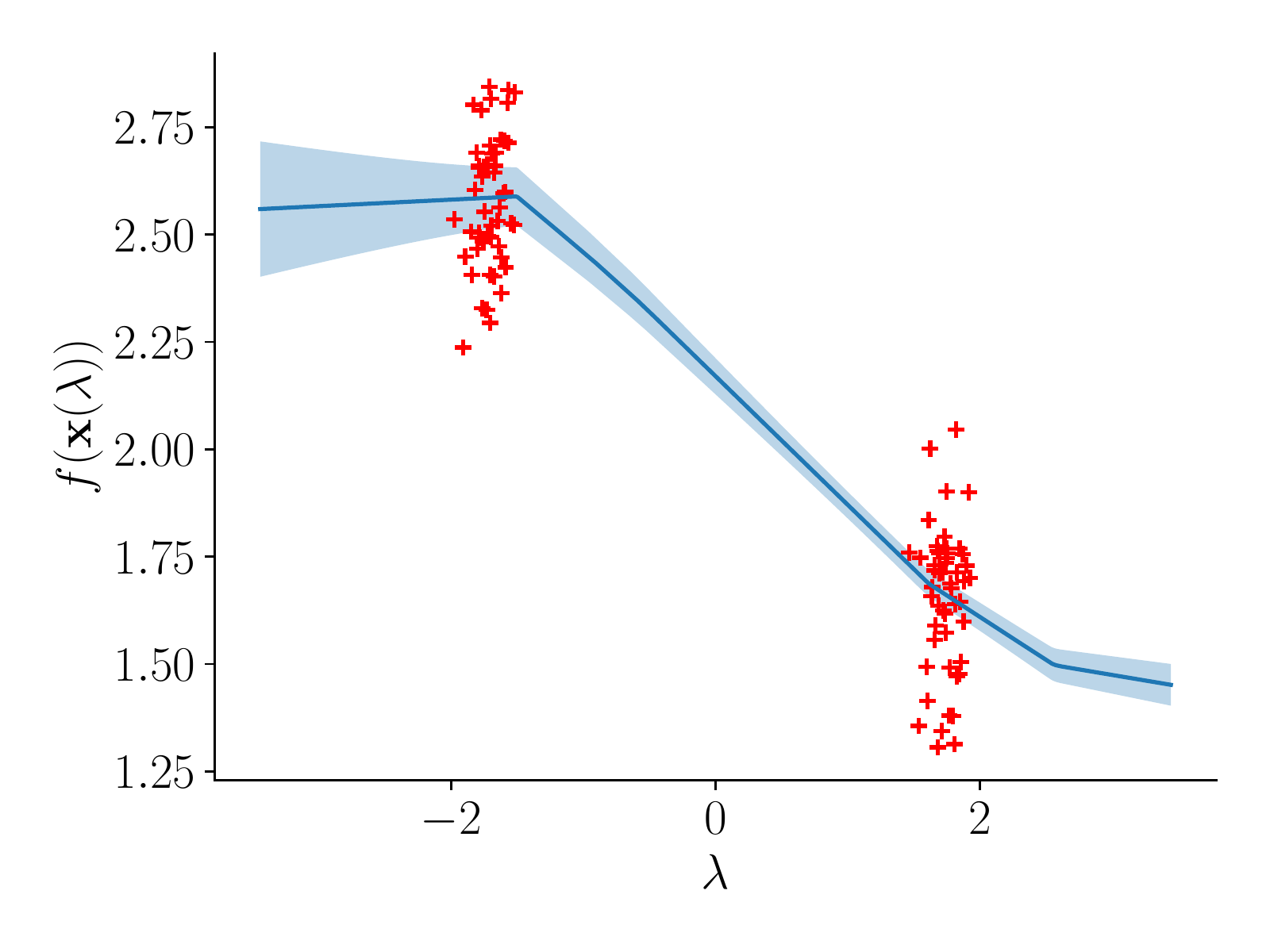}\\
    \includegraphics[trim=.15cm 0cm 0cm 0cm, clip, width=.30\textwidth]{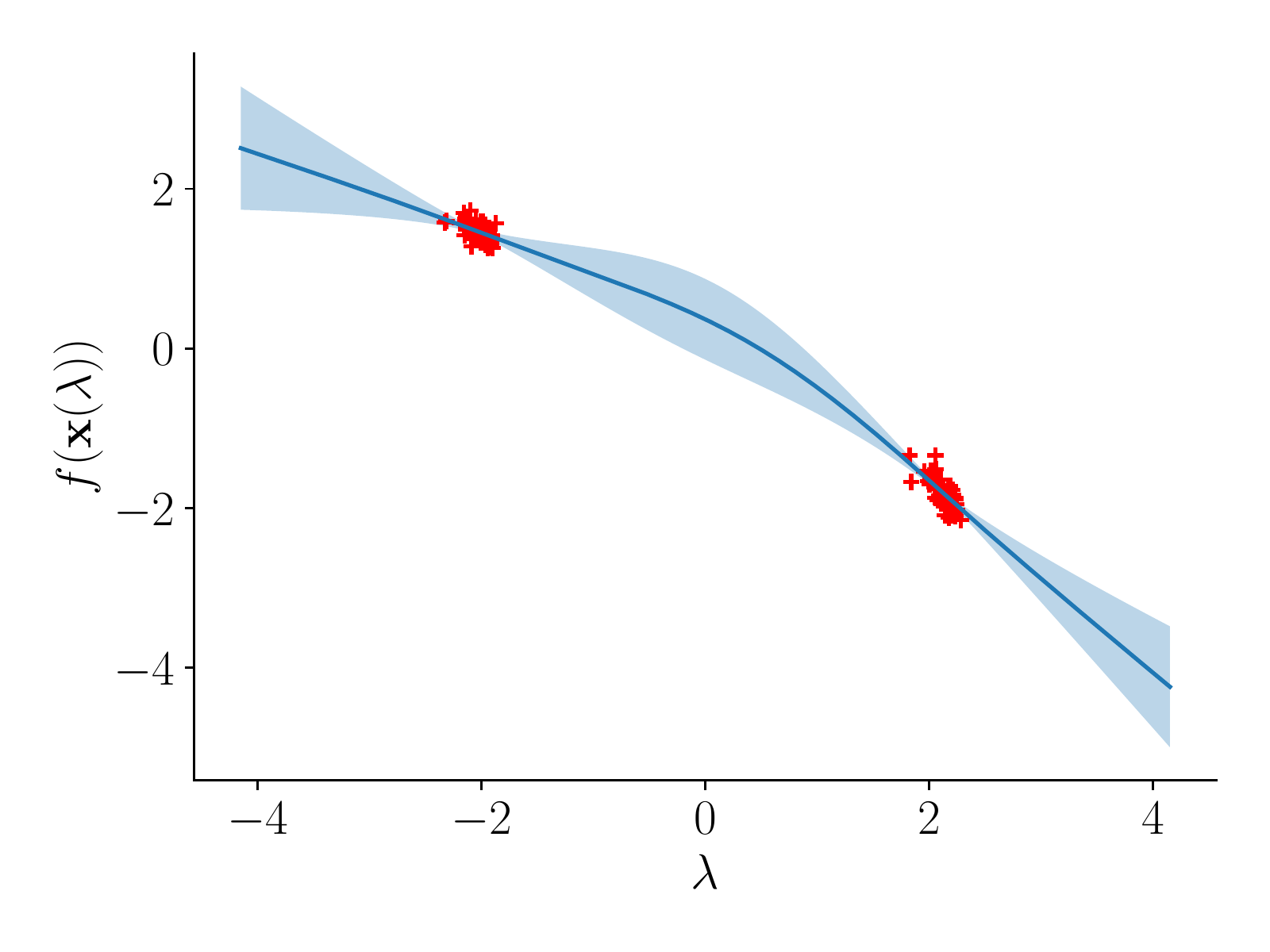} &\includegraphics[trim=.15cm 0cm 0cm 0cm, clip, width=.30\textwidth]{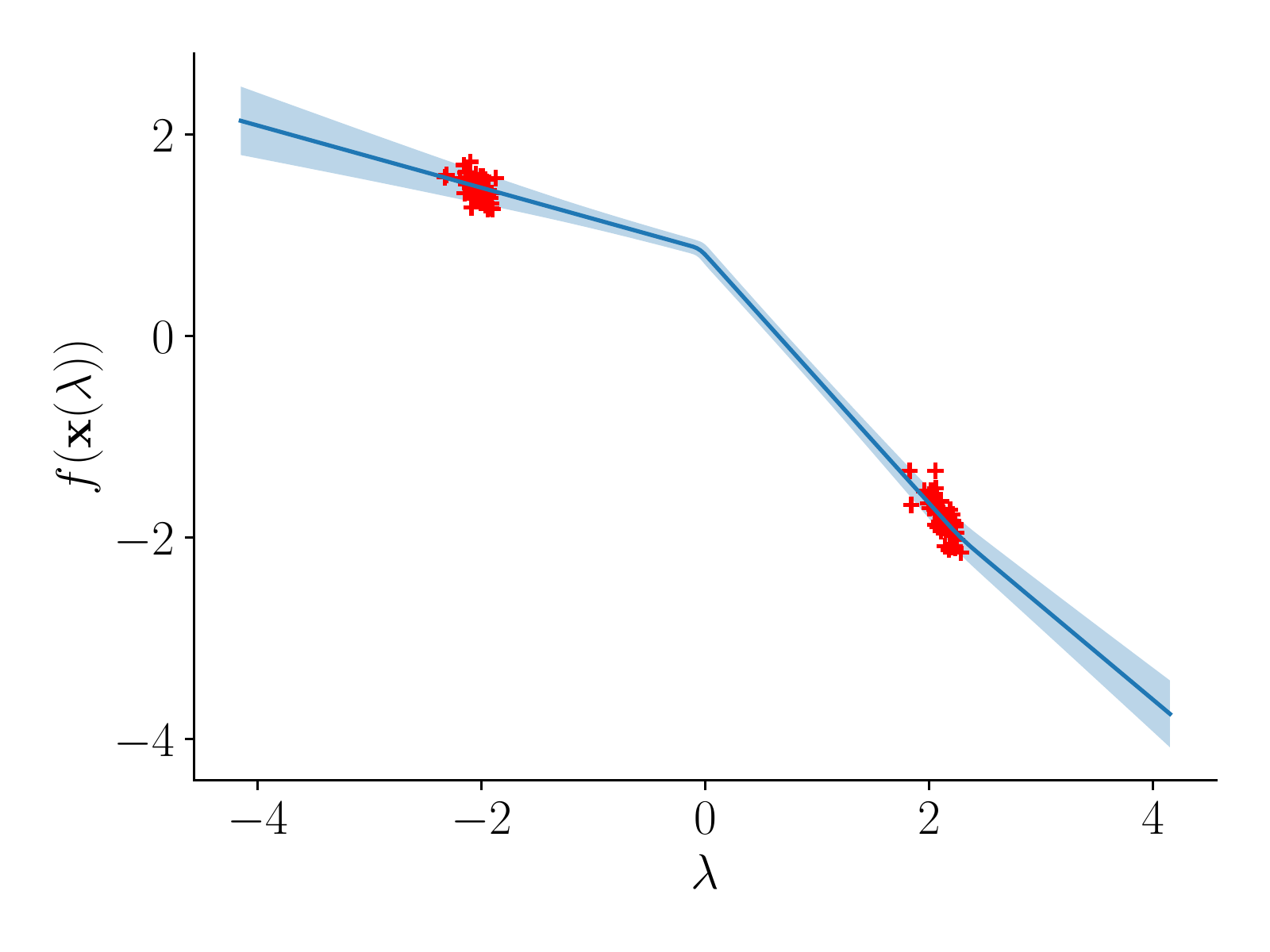} & \includegraphics[trim=.15cm 0cm 0cm 0cm, clip, width=.30\textwidth]{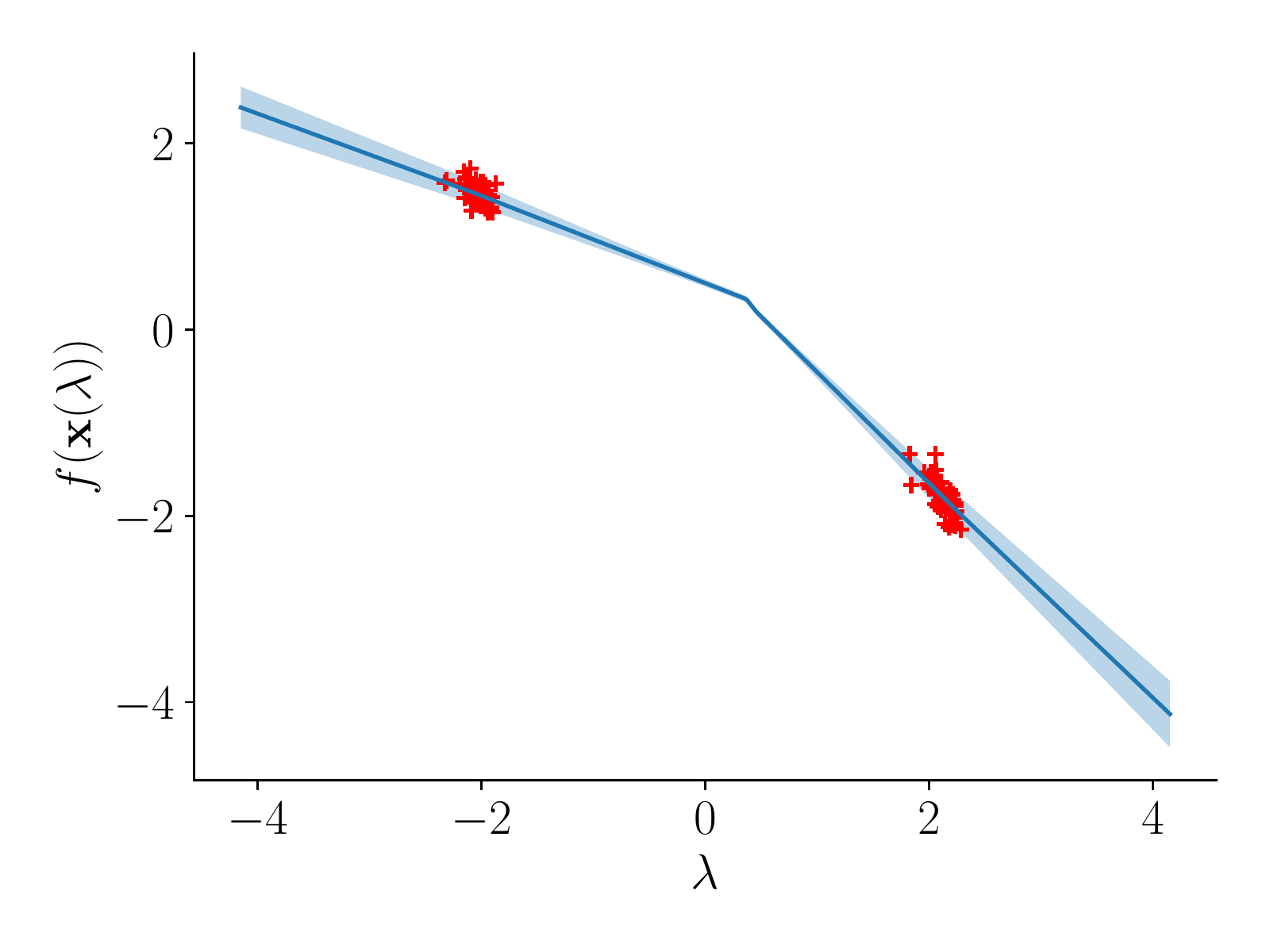}\\
   \includegraphics[trim=.15cm 0cm 0cm 0cm, clip, width=.30\textwidth]{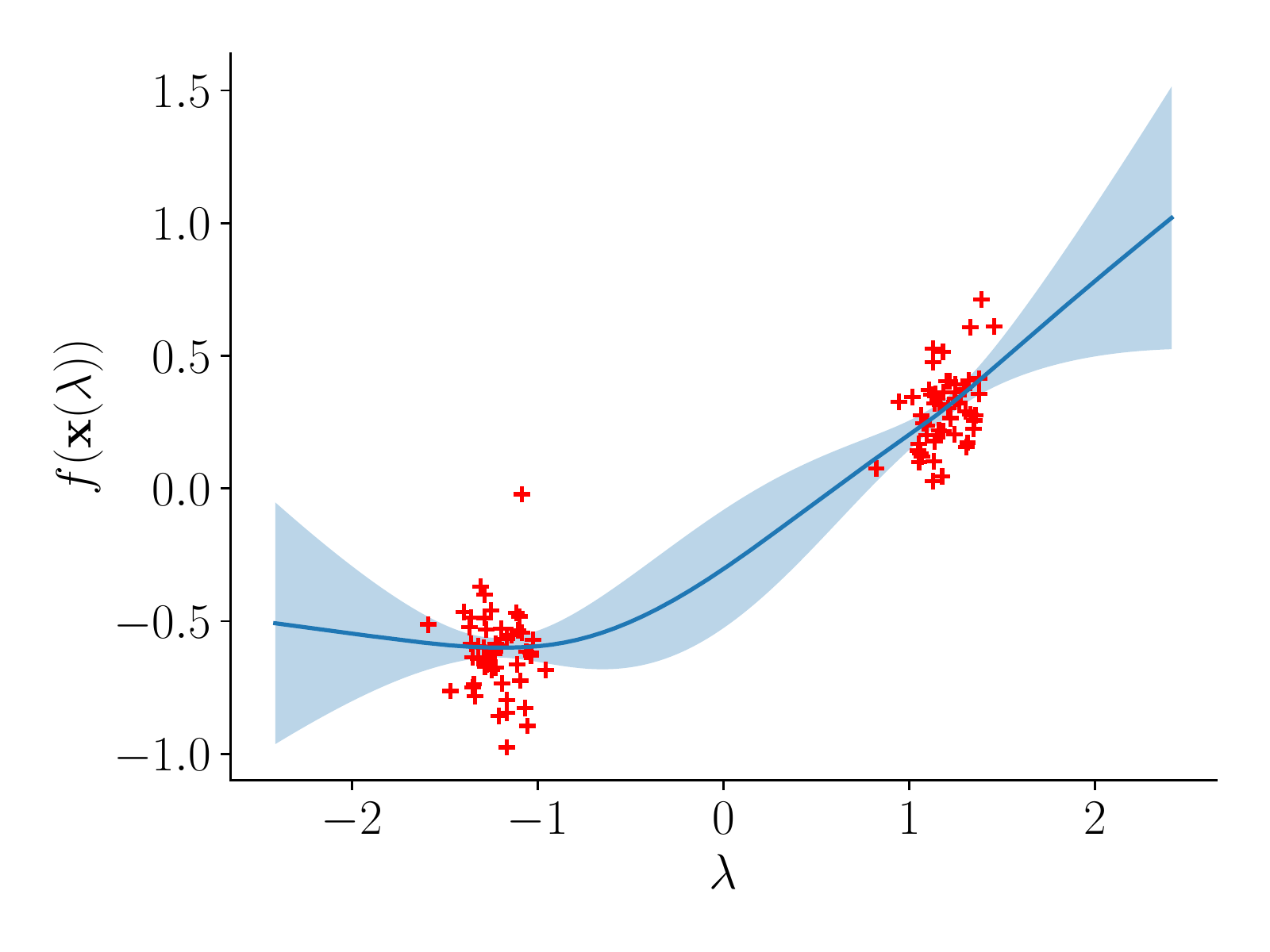}  & \includegraphics[trim=.15cm 0cm 0cm 0cm, clip, width=.30\textwidth]{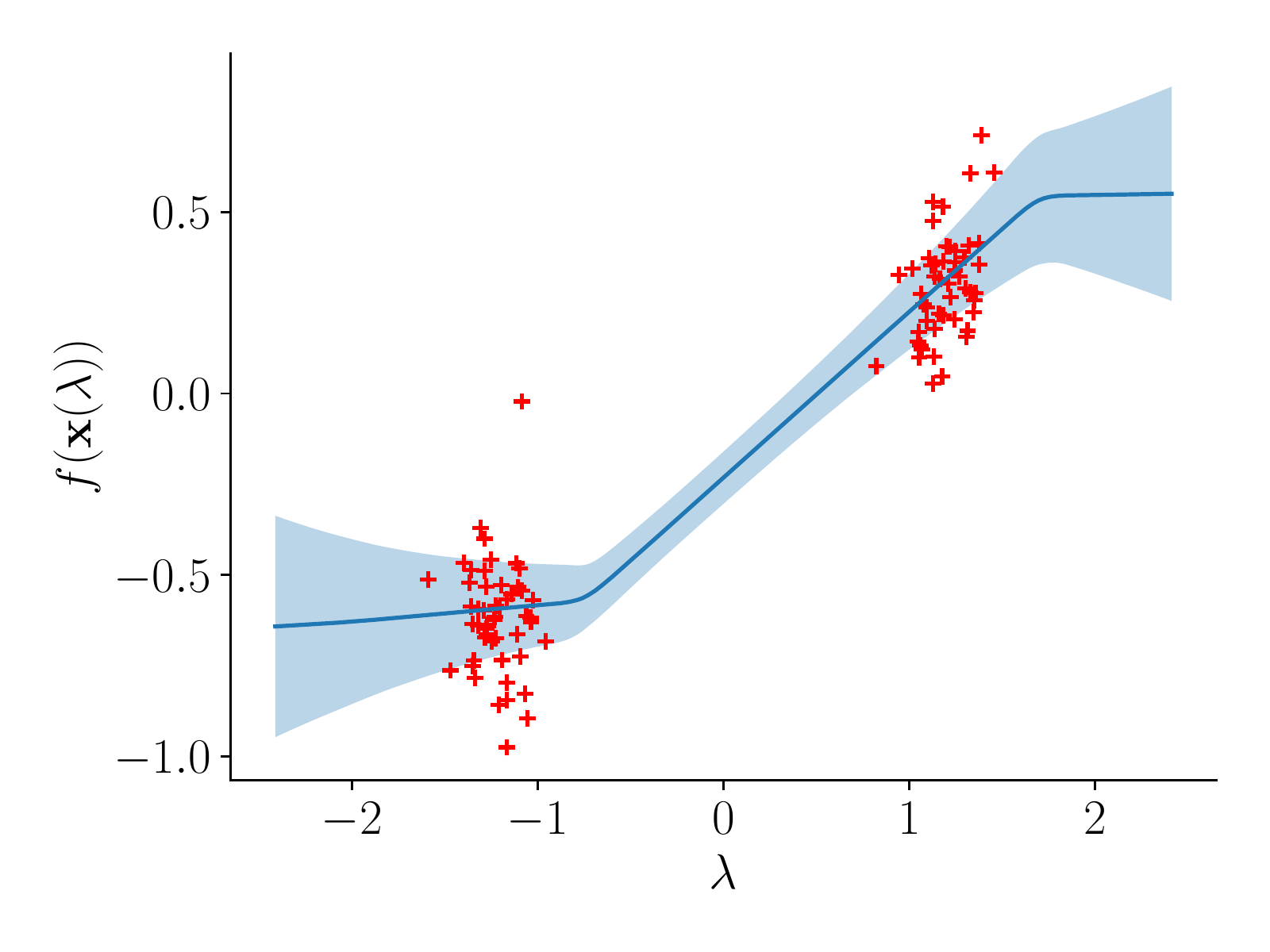}&\includegraphics[trim=.15cm 0cm 0cm 0cm, clip, width=.30\textwidth]{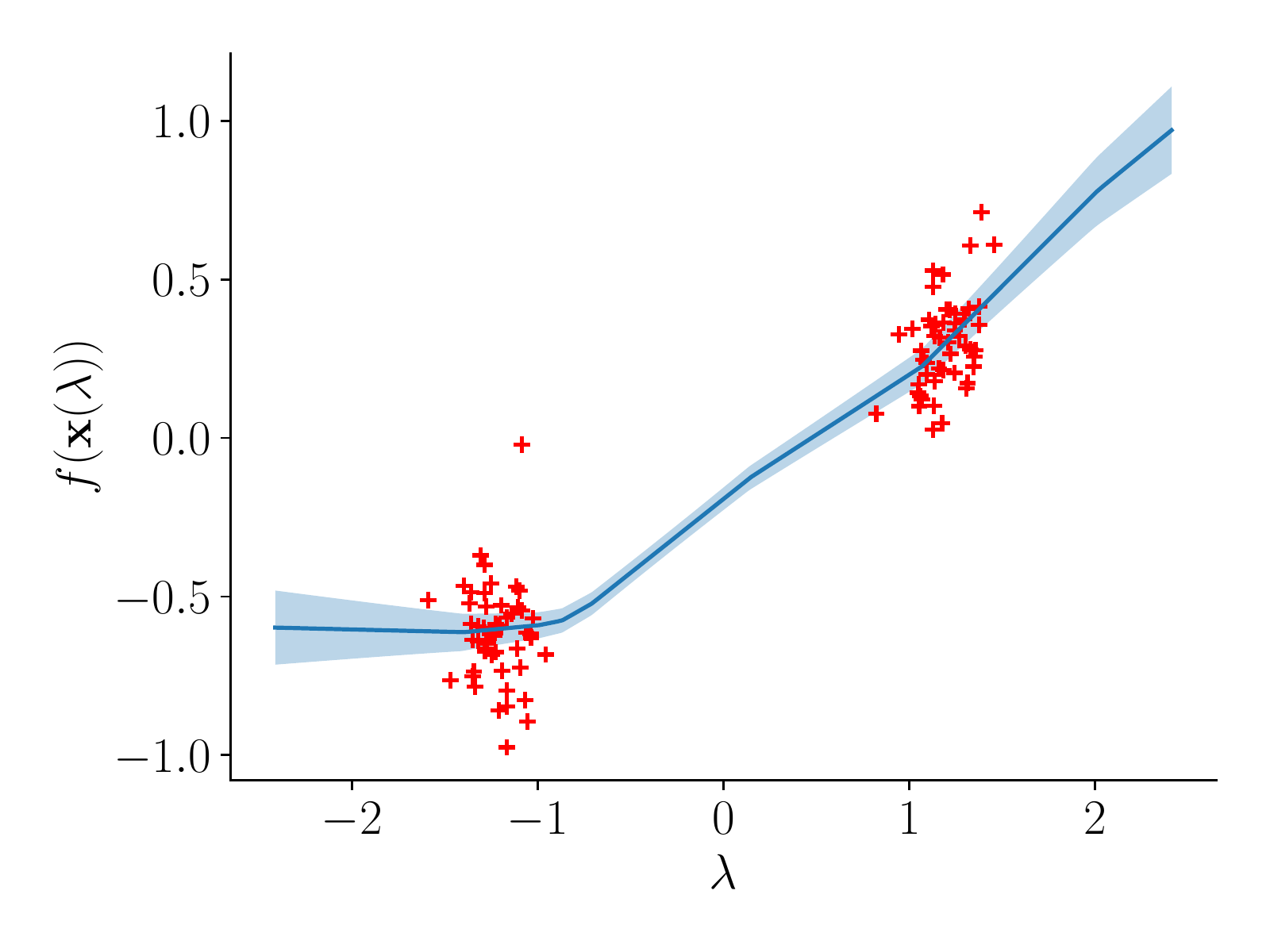}
\end{tabular}
\caption{Mean and 2 standard deviation bars of the predictive distribution on lines joining random clusters of data, for 1HL BNNs. Each row represents the same random dataset. We also plot the projection of the 5-dimensional data onto this line segment as its $\lambda$-value, along with the output value of the data. Note that the data appears noisy, but this is due to the projection onto a lower-dimensional space.}\label{fig:1HL_random_regression}
\end{figure}

\begin{figure}[h]
\begin{tabular}{ccc}
GP & MFVI  & MCDO \\
    \includegraphics[trim =.15cm 0cm 0cm 0cm, clip, width=.30\textwidth]{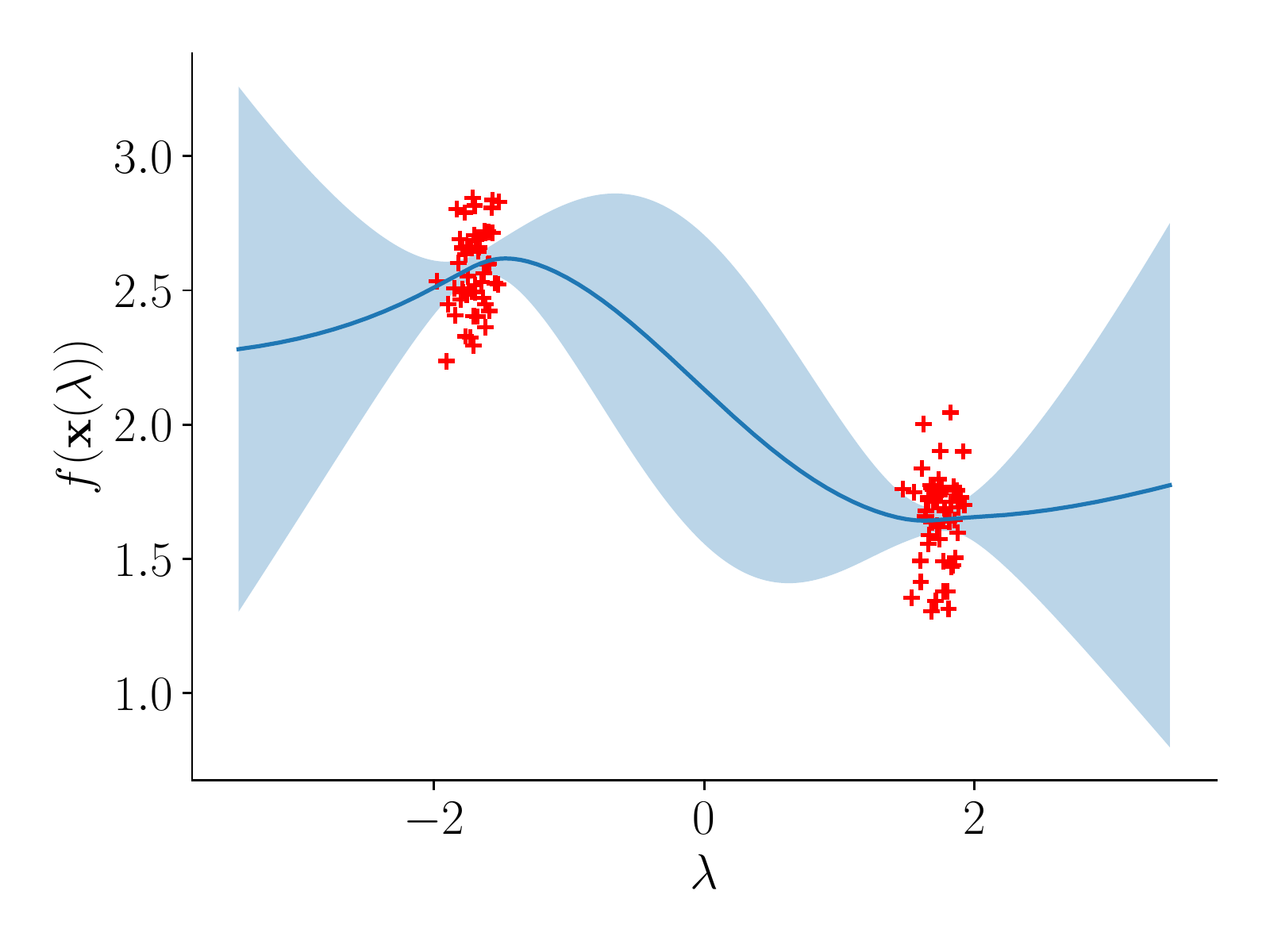} &\includegraphics[trim=.15cm 0cm 0cm 0cm, clip, width=.30\textwidth]{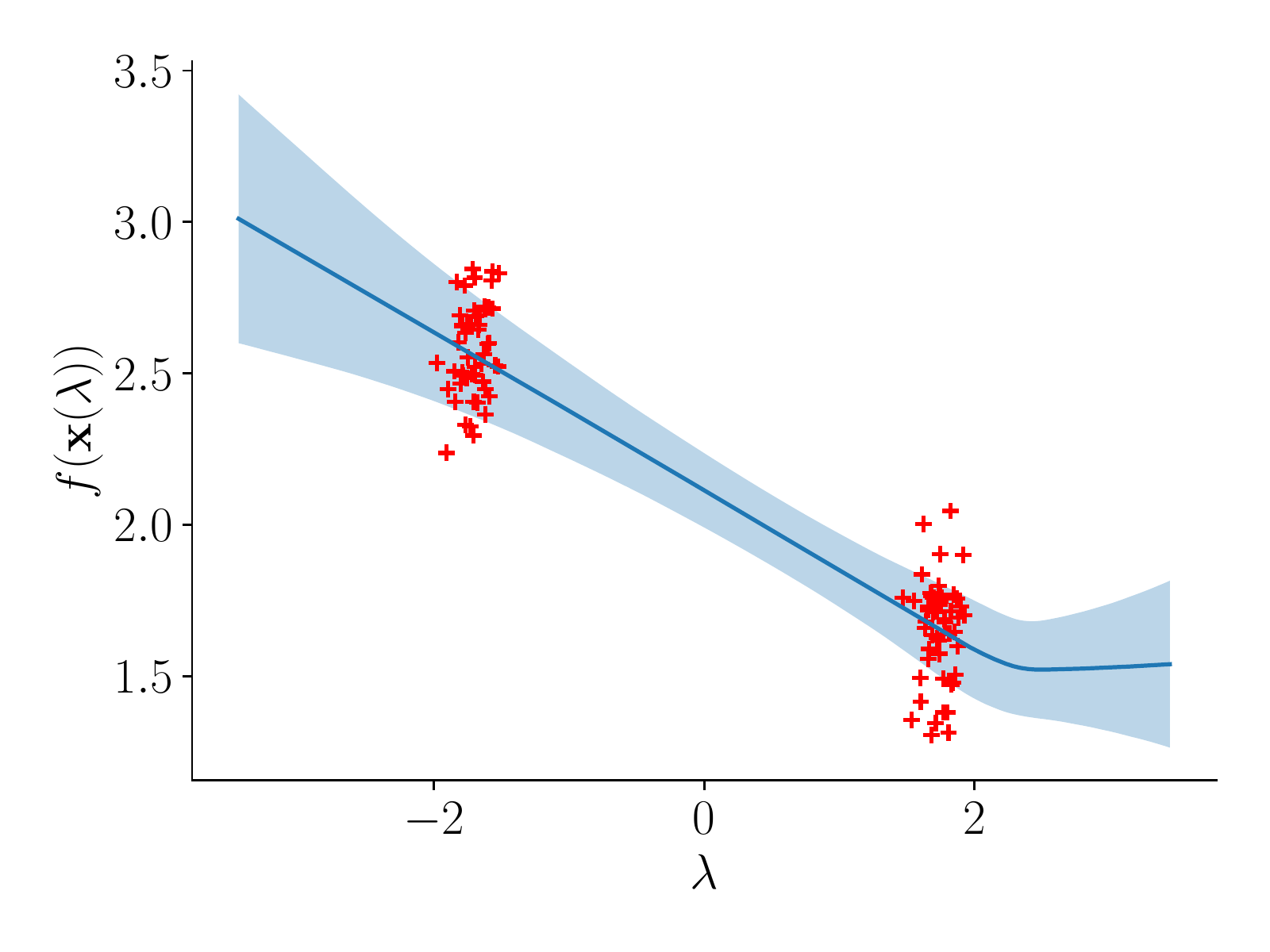}  &\includegraphics[trim=.15cm 0cm 0cm 0cm, clip, width=.30\textwidth]{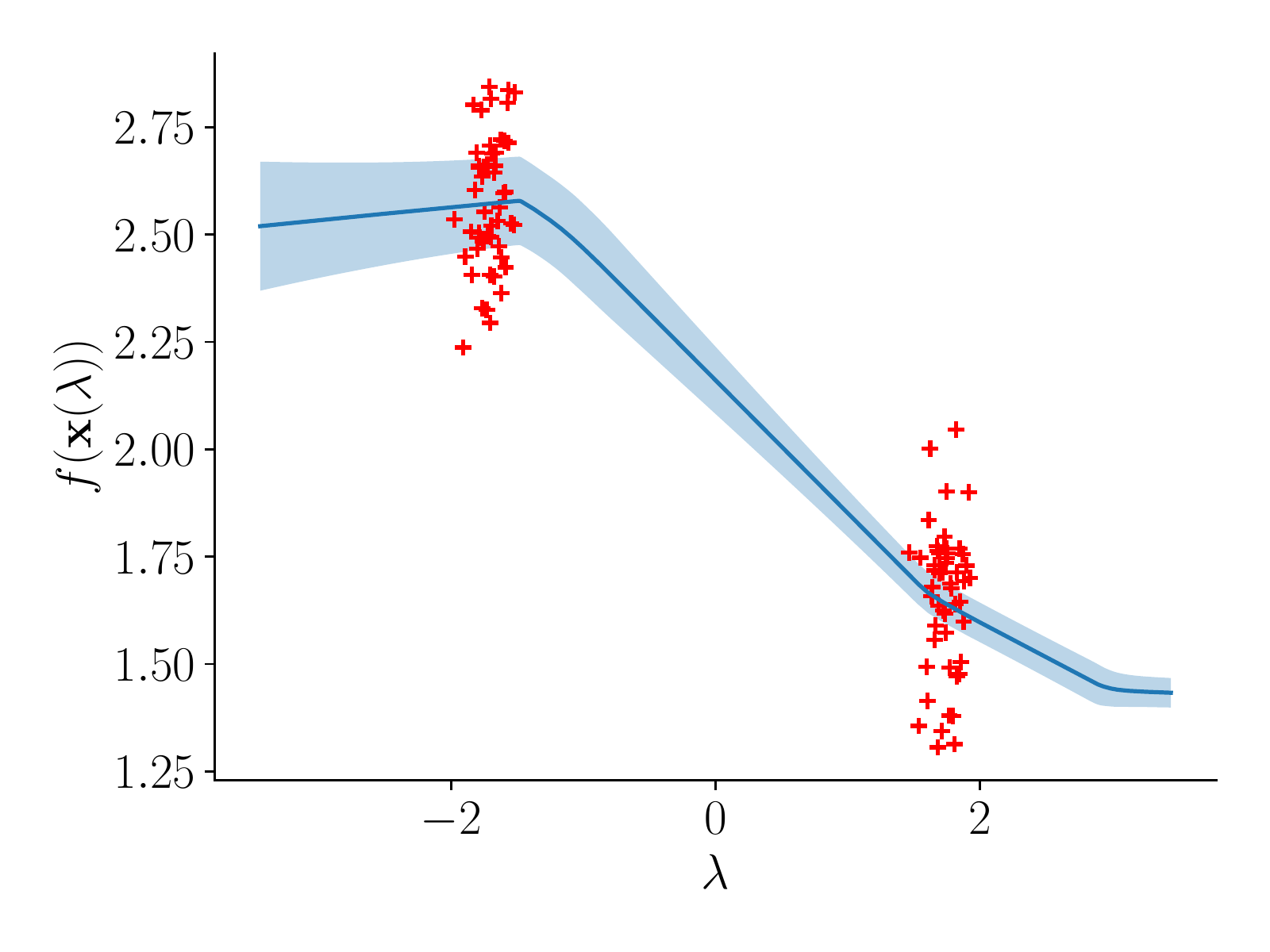}\\
    \includegraphics[trim=.15cm 0cm 0cm 0cm, clip, width=.30\textwidth]{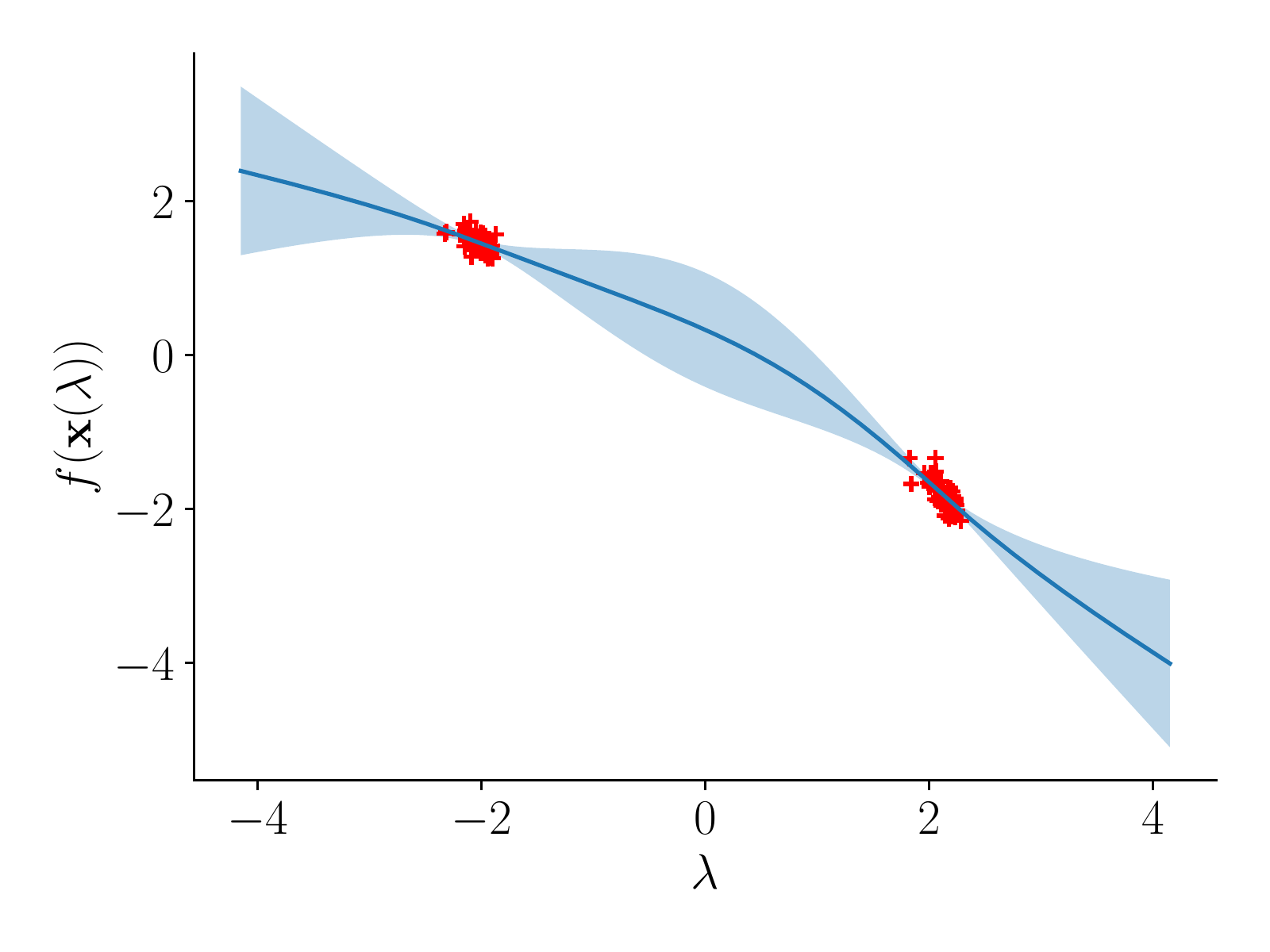} &\includegraphics[trim=.15cm 0cm 0cm 0cm, clip, width=.30\textwidth]{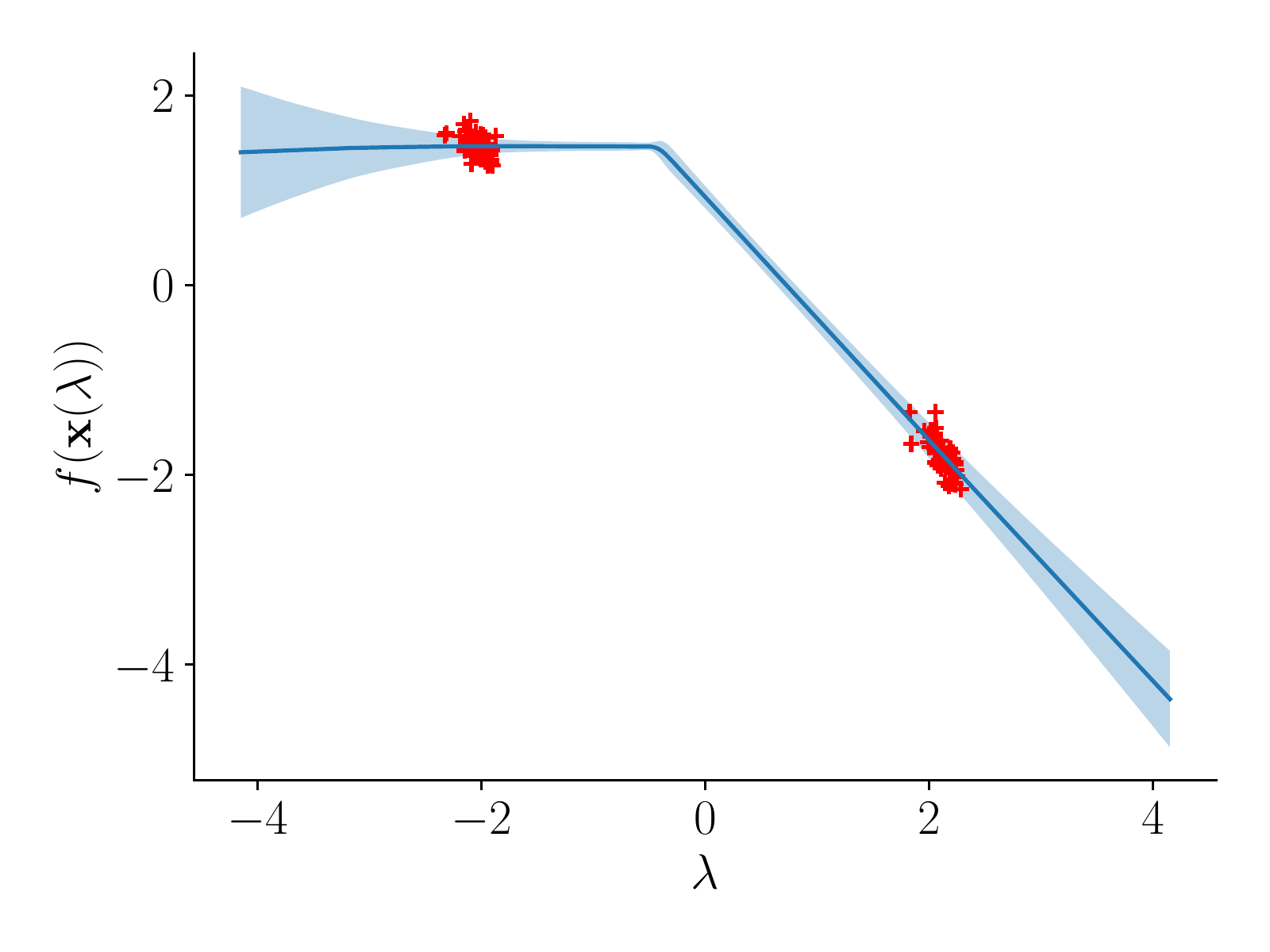} & \includegraphics[trim=.15cm 0cm 0cm 0cm, clip, width=.30\textwidth]{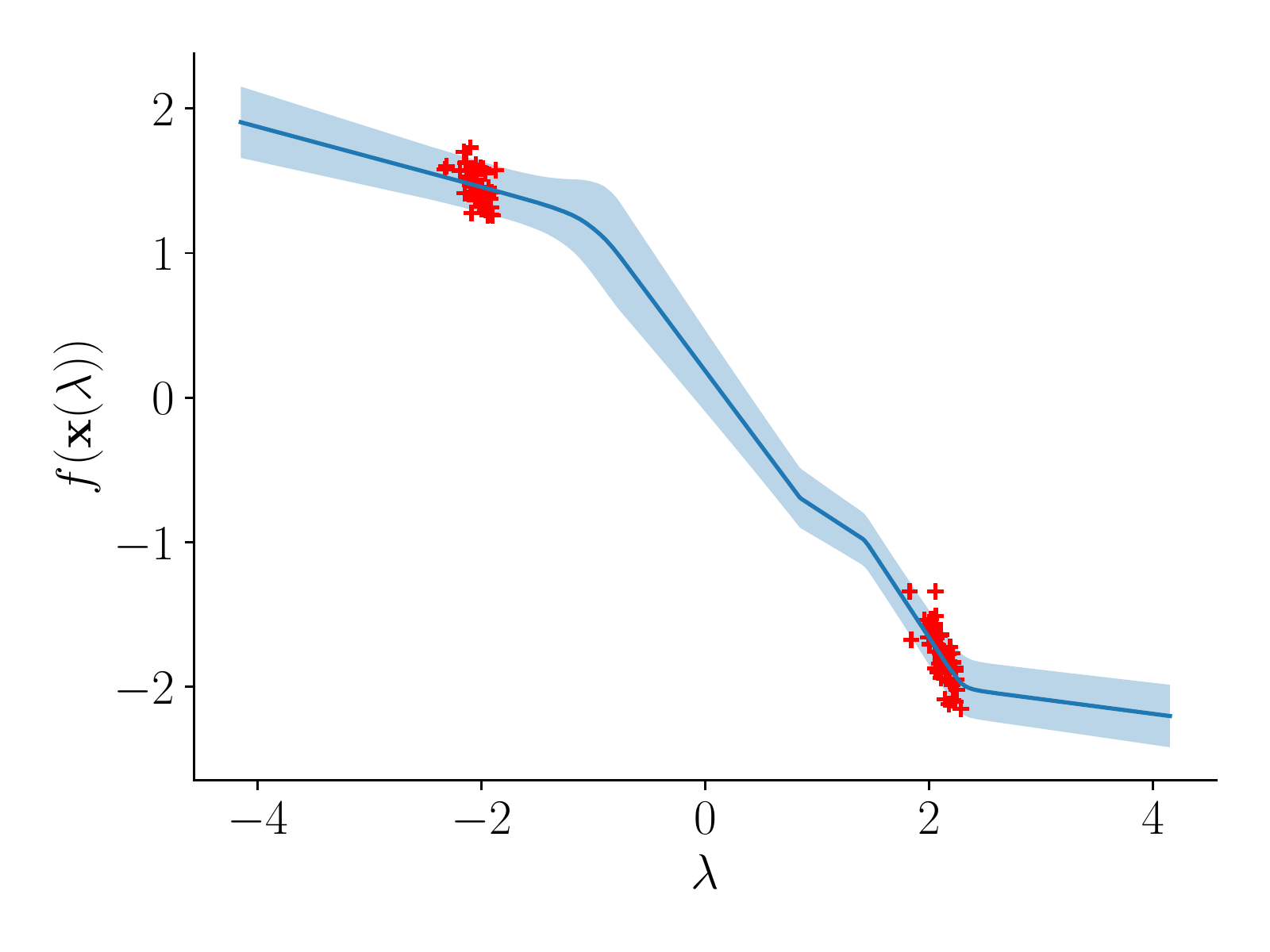}\\
   \includegraphics[trim=.15cm 0cm 0cm 0cm, clip, width=.30\textwidth]{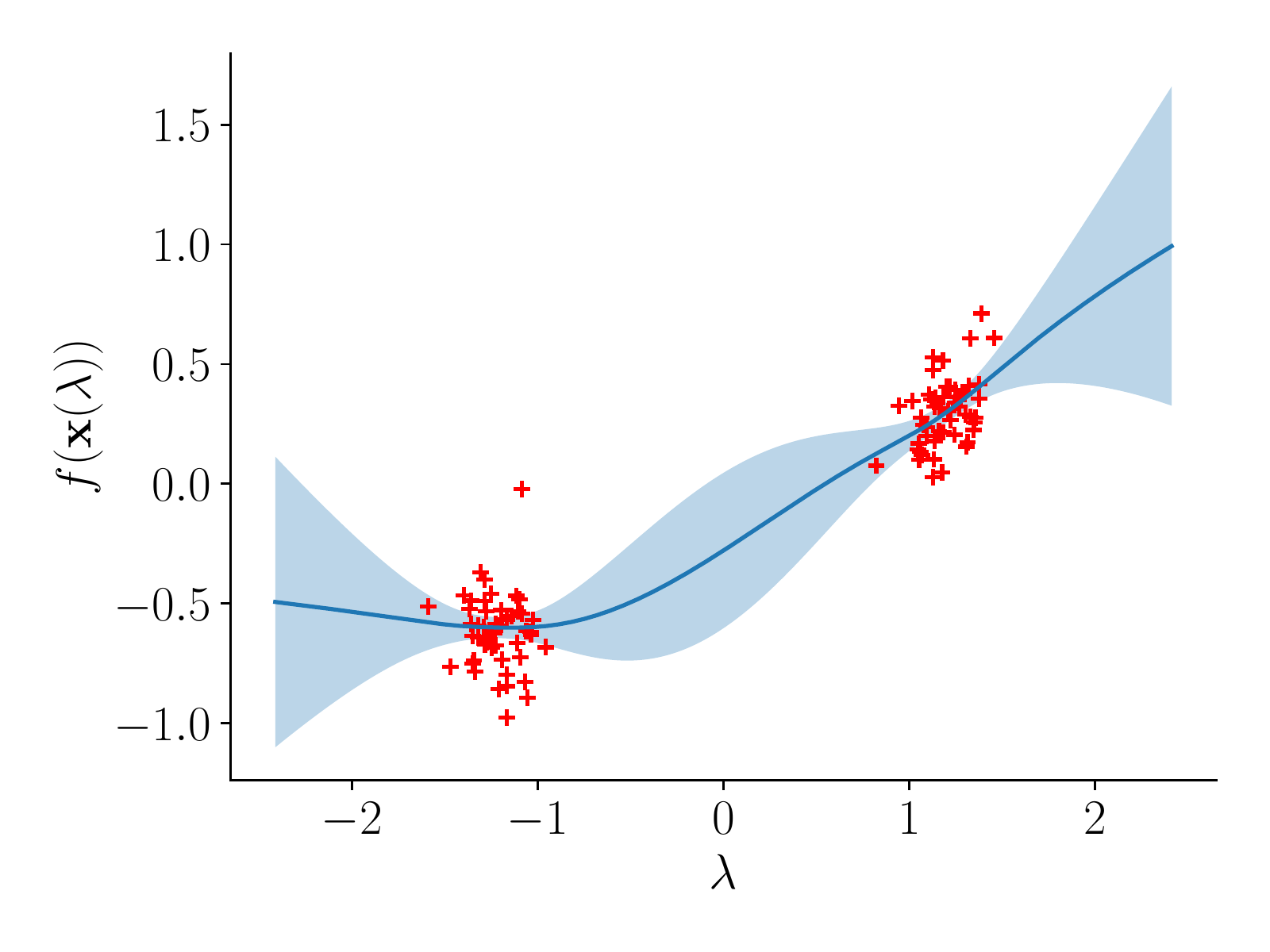}  & \includegraphics[trim=.15cm 0cm 0cm 0cm, clip, width=.30\textwidth]{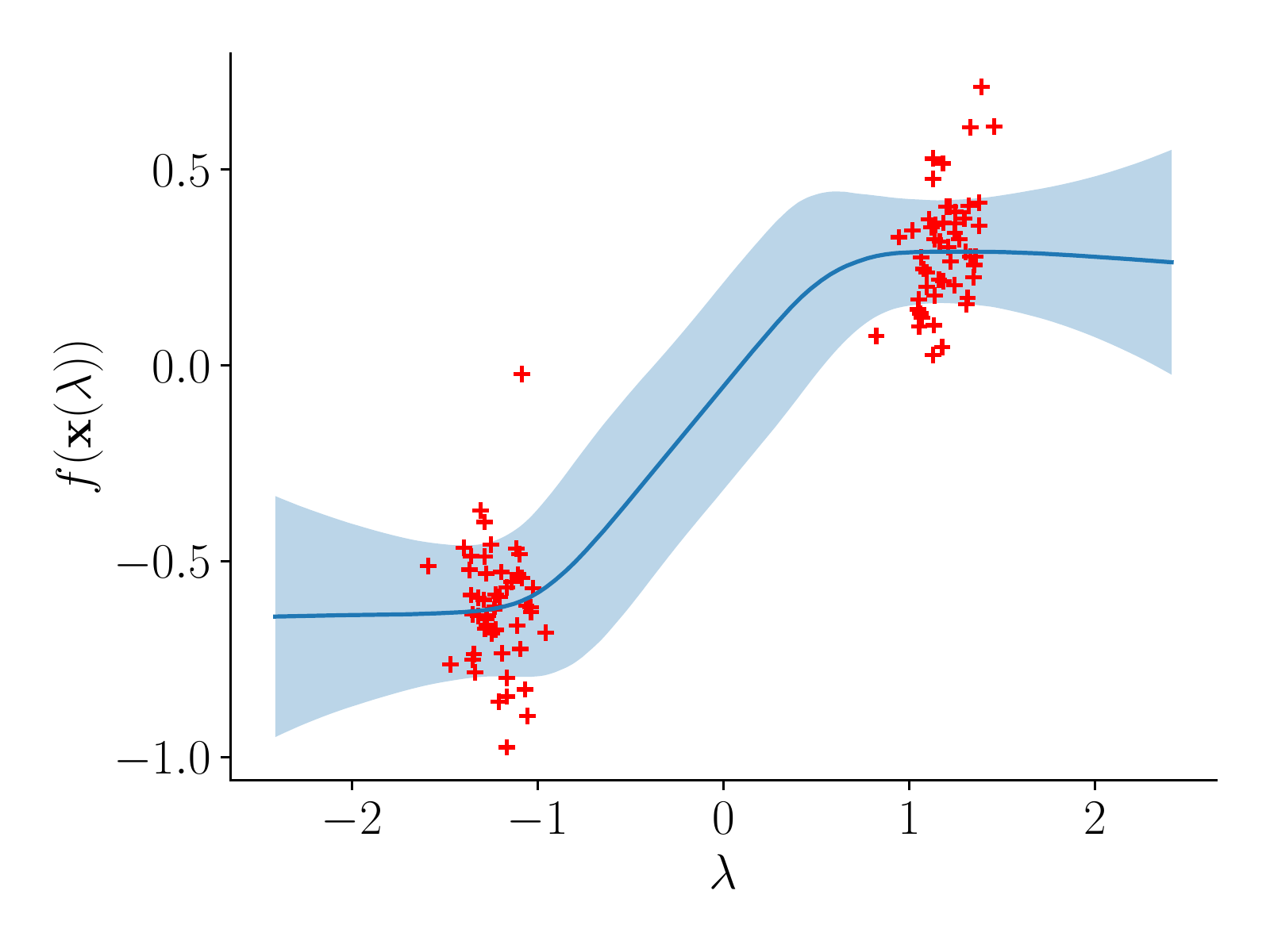}&\includegraphics[trim=.15cm 0cm 0cm 0cm, clip, width=.30\textwidth]{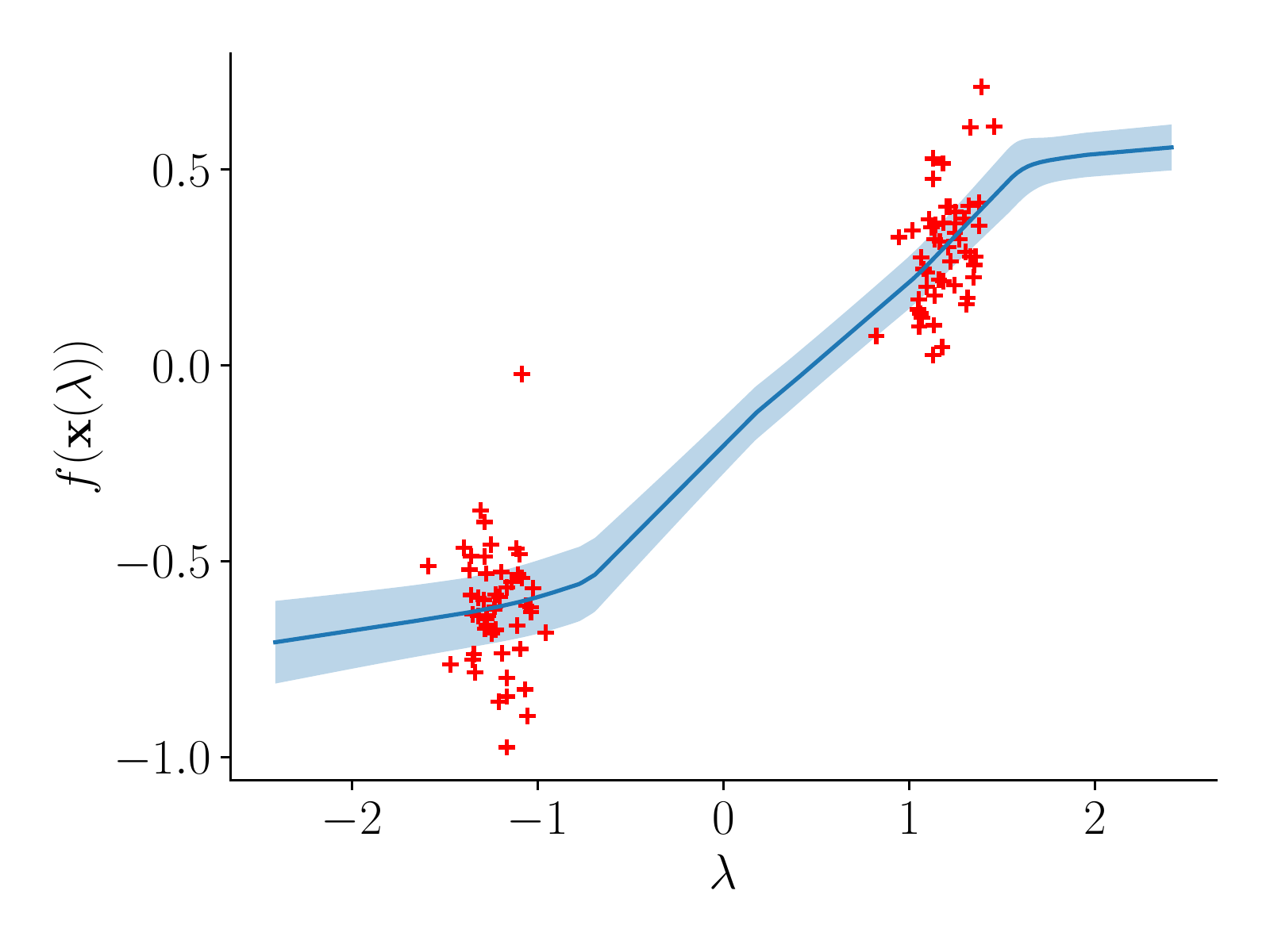}
\end{tabular}
\caption{Same experimental set-up as in \cref{fig:1HL_random_regression} for the 3HL case.}\label{fig:3HL_random_regression}
\end{figure}

\section{General Statements and Proofs of Theorems 1 and 2}\label{app:SHL-proofs}
In \cref{sec:SHL-theory} we stated simplified versions of bounds concerning the variance of single-hidden layer networks with certain approximating families. The two main results we prove in this section are the following generalisations of \cref{thm:mfgaussian,thm:dropout} respectively:

\begin{thm}\label{thm:mfgaussian-general}
    Consider a single-hidden layer ReLU neural network mapping from $\R^D \to \R^K$ with $I \in \N$ hidden units. The corresponding mapping is given by $f^{(k)}(\bfx) =  \sum_{i=1}^I w_{k,i}\psi \left( \sum_{d=1}^D u_{i,d}x_d+v_i \right)+b_k$ for $1\leq k \leq K$, where $\psi(a) = \max(0, a)$. Suppose we have a distribution over network parameters with density of the form:
\begin{align}
    q(\bfW,\bfb,\bfU,\bfv) \! = \! \prod_{i=1}^I q_{i}(\bfw_{i}|\bfU, \bfv)q(\bfb|\bfU, \bfv)\prod_{i=1}^I\prod_{d=1}^D\mcN(u_{i,d};\mu_{u_{i,d}},\sigma^2_{u_{i,d}})\prod_{i=1}^I\mcN(v_i;\mu_{v_{i}},\sigma^2_{v_{i}}) , \label{eqn:approx_family}
\end{align}
where $\bfw_i = \{ w_{k,i}\}_{k=1}^K$ are the weights out of neuron $i$ and $\bfb = \{b_k\}_{k=1}^K$ are the output biases, and $q_i(\bfw_i|\bfU, \bfv)$ and $q(\bfb|\bfU,\bfv)$ are arbitrary probability densities with finite first two moments. Consider a line in $\R^D$ parameterised by $\bfx(\lambda)_d=\gamma_d \lambda + c_d$
for $\lambda \in \R$ such that $\gamma_dc_d = 0$ for $1 \leq d \leq D.$ Then for any $\lambda_1 \leq 0 \leq \lambda_2,$ and any $\lambda_*$ such that $|\lambda_*|\leq \min(|\lambda_1|,|\lambda_2|),$
\begin{align}
\Var\lbrack f^{(k)}(\bfx(\lambda_*))\rbrack \leq \Var\lbrack f^{(k)}(\bfx(\lambda_1))\rbrack +\Var\lbrack f^{(k)}(\bfx(\lambda_2))\rbrack \text{\quad \emph{for} \, \,} 1\leq k \leq K. \label{eqn:variance_bound}
\end{align}
\end{thm}
We now briefly show how the statement of \cref{thm:mfgaussian-general} in the main text can be deduced from this more general version. The fully factorised Gaussian family $\Qffg$ is of the form in \cref{eqn:approx_family}. It remains to show that both conditions $i.$ and $ii.$ imply that $\gamma_dc_d=0$. Consider any line intersecting the origin (i.e. satisfying condition i)). Such a line can be written in the form $\bfx(\lambda)_d=\gamma_d \lambda$ by choosing the origin to correspond to $\lambda = 0$. As $c_d=0$ for all $d$, $\gamma_dc_d=0$ for all $d$. In \cref{thm:mfgaussian} $\bfp= \bfx(\lambda_1)$ and $\bfq= \bfx(\lambda_2)$ are on opposite sides of the origin, hence the signs of $\lambda_1$ and $\lambda_2$ are opposite. Finally, the condition that $\bfr=\bfx(\lambda_*)$ is closer to the origin than both $\bfp$ and $\bfq$ is exactly that $|\lambda_*|\leq \min(|\lambda_1|,|\lambda_2|)$. 

In order to verify condition ii), note that any line orthogonal to a hyperplane $x_{d'} = 0$ can be parameterised as $\bfx(\lambda)_{d} = \gamma_{d}\lambda + c_{d}$, where $\gamma_{d} = 0$ for $d \neq d'$ and $c_{d'}=0$. Hence $\gamma_{d}c_{d} = 0$ for all $d$. The condition that the line segment $\overrightarrow{\bfp\bfq}$ intersects the plane, with $\bfp = \bfx(\lambda_1)$ and $\bfq = \bfx(\lambda_2)$ is exactly that the signs of $\lambda_1$ and $\lambda_2$ are opposite, and that $|\lambda_*|\leq \min(|\lambda_1|,|\lambda_2|)$.

As a corollary of \cref{thm:mfgaussian-general}, we can obtain bounds on higher-dimensional objects than lines, such as on hypercubes. For instance, consider the case where $\bfx \in \R^2$. Let $\bfp, \bfq, \bfr, \bfs$ be the four corners of a rectangle centered the origin. For any point $\bfa$ in the rectangle, we can upper bound $\Var[f(\bfa)]$ by the sum of the variances at the top and bottom edges of the rectangle. These in turn can be upper bounded by the variances at the corners of the rectangle. Hence we have that for any point $\bfa$ in the rectangle, $\Var[f(\bfa)] \leq \Var[f(\bfp)] + \Var[f(\bfq)] + \Var[f(\bfr)] + \Var[f(\bfs)]$. Similarly the variance at any point in a hypercube centered at the origin can be bounded by the sum of the variances on its vertices, and we can obtain tighter bounds on diagonals and faces of the hypercube, by repeatedly applying \cref{thm:mfgaussian-general}.

\begin{thm}[MC dropout]\label{thm:dropout-general}
    Consider a single-hidden layer ReLU neural network mapping from $\R^D \to \R^K$ with $I \in \N$ hidden units. The corresponding mapping is given by $f^{(k)}(\bfx) =  \sum_{i=1}^I w_{k,i}\psi \left( \sum_{d=1}^D u_{i,d}x_d+v_i \right)+b_k$ for $1\leq k \leq K$, where $\psi(a) = \max(0, a)$.
    Assume $\bfU,\bfv$ are set deterministically and 
\begin{align*}
    q(\mathbf{W}, \mathbf{b}) = q(\mathbf{b})\prod_{i=1}^Iq_{i}(\bfw_{i}),
\end{align*}
where $\bfw_i = \{w_{k,i}\}_{k=1}^K$ are the weights out of neuron $i$, $\bfb = \{b_k\}_{k=1}^K$ are the output biases and $q(\bfb)$ and $q_i(\bfw_i)$ are arbitrary probability densities with finite first two moments. Then, $\Var[f^{(k)}(\bfx)]$ is convex in $\bfx$ for $1 \leq k \leq K$.
\end{thm}
\begin{proof}
The theorem follows immediately from \cref{lem:convexconditioned} since $\bfU$ and $\bfv$ are deterministic.
\end{proof}

\begin{rem}
\Cref{thm:dropout-general} applies for any activation function $\psi$ such that $\psi^2$ is convex. This is the only property of $\psi$ used in \cref{lem:convexconditioned}.
\end{rem}

\subsection{Preliminary Lemmas}\label{sec:preliminary-lemmas}
In order to prove \cref{thm:mfgaussian,thm:dropout} we first collect a series of preliminary lemmas.
\begin{lem}\label{lem:convexconditioned}
Assume a distribution for $\bfW,\bfb | \bfU, \bfv$ with density of the form 
\begin{align*}
    q(\mathbf{W}, \mathbf{b}|\bfU, \bfv) = q(\mathbf{b}|\bfU, \bfv)\prod_{i}q_{i}(\bfw_{i}|\bfU, \bfv).
\end{align*}
Then, $\Var[f^{(k)}(\bfx)| \bfU, \bfv]$ is a convex function of $\bfx.$
\end{lem}
The proof of \cref{lem:convexconditioned} is in \cref{sec:proof-convex-conditioned}. 

\begin{lem}\label{lem:halfmonotonicfunction}
Consider the variance of a single neuron in the one dimensional case, with activation $a(x)\sim\mcN(\mu(x),\sigma^2(x)),$ $\mu(x)=\mu_ux+\mu_v$ and $\sigma^2(x)=\sigma^2_ux^2+\sigma^2_v$ . Let  \[\mcT_1=\{f\geq 0: \forall 0\leq b<a, f(a)\geq f(-a) \text{ and } f(b)\leq f(a)\}\] and \[\mcT_2=\{f\geq 0: \forall a< b \leq 0, f(a)\geq f(-a) \text{ and } f(b) \leq f(a)\}.\]
If $\mu_u \geq 0$, then $\Var[\psi(a(x))] \in \mcT_1$. If $\mu_u \leq 0$, then $\Var[\psi(a(x))] \in \mcT_2$. 
\end{lem}
The proof of \cref{lem:halfmonotonicfunction} is in \cref{sec:proofofhalfmonotonic}.

\begin{cor}[Corollary of \cref{lem:halfmonotonicfunction}]\label{cor:multivariatecase}
Consider a line in $\R^D$ parameterized by $[\bfx(\lambda)]_d=\gamma_d\lambda + c_d$ for $\lambda \in \R$ such that $\gamma_dc_d = 0$ for $1 \leq d \leq D.$  Let $a(\bfx)\coloneqq \sum_{d=1}^D u_d x_d + v$ with $\{u_d\}_{d=1}^D$ and $v$ independent and Gaussian distributed. Then, $\Var[\psi(a(\bfx(\lambda)))] \in \mcT_1 \cup \mcT_2$ (as a function of $\lambda$).
\end{cor}

\begin{proof}
The activation $a(\bfx(\lambda))$ is a linear combination of Gaussian random variables, and is therefore Gaussian distributed. Moreover the mean is linear in $\lambda.$ The variance of $a(\bfx(\lambda))$ is given by:
\begin{align*}
\Var[a(\bfx(\lambda))]&=\sum_{d=1}^D \Var[u_d] (\gamma_d\lambda+c_d)^2 + \Var[v] \\ 
&= \sum_{d=1}^D \sigma^2_{u_d} (\gamma_d\lambda+c_d)^2 + \sigma^2_{v}\\
&= \lambda^2\left(\sum_{d=1}^D \sigma^2_{u_d} \gamma_d^2\right) +2 \lambda \left(\sum_{d=1}^D \sigma^2_{u_d}\gamma_dc_d\right) + \left(\sum_{d=1}^D \sigma^2_{u_d}c_d^2+ \sigma^2_{v}\right) \\
&= \lambda^2\left(\sum_{d=1}^D \sigma^2_{u_d} \gamma_d^2\right) +\left(\sum_{d=1}^D \sigma^2_{u_d}c_d^2+ \sigma^2_{v}\right). \\
\end{align*}

Defining $\sigma^2_{\tilde{u}}=\sum_{d=1}^D \sigma^2_{u_d} \gamma_d^2$ and $\sigma^2_{\tilde{v}}=\sum_{d=1}^D \sigma^2_{u_d}c_d^2+ \sigma^2_{v}$, the corollary follows from \cref{lem:halfmonotonicfunction}.
\end{proof}

\begin{lem}\label{lem:positivecombinations}
Let $\mcC$ be the set of convex functions from $\R \to [0,\infty).$ Fix any $a<0<b$ and $c$ such that $|c|\leq \min(|a|,|b|)$. Then any function $f$ that can be written as a linear combination of functions in $\mcT_1 \cup \mcT_2 \cup \mcC$ with non-negative weights satisfies, $f(c) \leq f(a)+f(b)$.
\end{lem}
 The proof of \cref{lem:positivecombinations} can be found in \cref{app:proof-positive-combinations}.
 
\subsection{Proof of Theorem 1}
Having collected the necessary preliminary lemmas we now prove \cref{thm:mfgaussian}.  
\begin{proof}[Proof of Theorem 1]
By the law of total variance,
\[ \Var[f^{(k)}(\mathbf{x})] = \mathbb{E}[ \Var[f^{(k)}(\mathbf{x})|\mathbf{U}, \mathbf{v}]] + \Var[\mathbb{E}[f^{(k)}(\mathbf{x})|\mathbf{U}, \mathbf{v}]]. \]
Using \cref{lem:convexconditioned}, $\Var[f^{(k)}(\mathbf{x})|\mathbf{U}, \mathbf{v}]$ is convex as a function of $\bfx.$ As the expectation of a convex function is convex, the first term is a convex function of $\bfx.$ For the second term we have
\begin{align*}
    \mathbb{E}[f^{(k)}(\mathbf{x})|\mathbf{U}, \mathbf{v}] &= \mathbb{E} \left[\sum_{i=1}^I w_{k,i}\psi (a_i) + b_k \bigg\vert \mathbf{U}, \mathbf{v} \right]= \sum_{i=1}^I \mu_{w_{k,i}}\psi (a_i) + \mu_{b_k},
\end{align*}
where $\mu_{w_{k,i}} \coloneqq \mathbb{E}[w_{k,i}], \mu_{b_k} \coloneqq \mathbb{E}[b_k].$ In the second line we used linearity of expectation and that conditioned on $(\mathbf{U}, \mathbf{v})$, the $a_i$ are deterministic. Next,
\begin{align}
    \Var [\mathbb{E}[f^{(k)}(\mathbf{x})|\mathbf{U}, \mathbf{v}]] &= \Var \left[ \sum_{i=1}^I \mu_{w_{k,i}}\psi (a_i) + \mu_{b_k} \right] = \sum_{i=1}^I \mu_{w_{k,i}}^2 \Var[\psi (a_i)], \label{eq:non_convex_part} 
\end{align}
since the $a_i$ are independent of each other. 

Consider a line in $\R^D$ parameterised by $[\bfx(\lambda)]_d=\gamma_d\lambda + c_d$ for $\lambda \in \R$ such that $\gamma_dc_d = 0$ for $1 \leq d \leq D.$ 

By \cref{cor:multivariatecase}, $\Var[\psi(a_i(\bfx(\lambda)))] \in \mcT_1 \cup \mcT_2$ (as a function of $\lambda$). Since $\Var[f^{(k)}(\mathbf{x})|\mathbf{U}, \mathbf{v}]$ is convex as a function of $\bfx,$ it is also convex as a function of $\lambda.$ We have written $\Var[f^{(k)}(\mathbf{x}(\lambda))]$ in the form assumed in \cref{lem:positivecombinations}, completing the proof.
\end{proof}

\section{Proof of Lemmas}
In this section we prove the preliminary lemmas stated in \cref{sec:preliminary-lemmas}.
\subsection{Proof of Lemma \ref{lem:convexconditioned}}\label{sec:proof-convex-conditioned}
\begin{proof}
We assume a distribution for the network weights such that: 
\begin{align*}
    q(\mathbf{W}, \mathbf{b}|\bfU, \bfv) = q(\mathbf{b}|\bfU, \bfv)\prod_{i=1}^Iq_{i}(\bfw_{i}|\bfU, \bfv).
\end{align*}
By this factorisation assumption, the outgoing weights from each neuron are conditionally independent. This means the conditional variance of the output under this distribution can be written
 \begin{align}\label{eqn:sum_variance}
     \Var[f^{(k)} (\mathbf{x})\vert \bfU,\bfv] &= \sum_i \Var[w_{k,i}|\bfU, \bfv] \psi(a_i)^2 + \Var[b_k|\bfU, \bfv]. 
 \end{align}
with $a_i \coloneqq a_i(\bfx) = \sum_{d=1}^D u_{i,d}x_d+v_i.$

Since $\Var[f^{(k)} (\mathbf{x})\vert \bfU,\bfv]$ is a linear combination of the $\psi(a_i)^2$ with non-negative weights (plus a constant), to prove convexity it suffices to show that each $\psi(a_i)^2$ is convex as a function of $\bfx$. $\psi(a_i)^2$ is convex as a function of $a_i$, since it is $0$ for $a_i \leq 0$ and $a_i^2$ for $a_i > 0$. To show that it is convex as a function of $\bfx$, we write
\begin{align*}
    \psi\left(a_i(t\bfx_1 + (1-t)\bfx_2) \right)^2 &= \psi\left(\sum_d u_{i,d}\left(t[\bfx_1]_d + (1-t)[\bfx_1]_d \right) + v_i \right)^2\\
    &= \psi\left( t\left( \sum_d u_{i,d}[\bfx_1]_d + v_i  \right) + (1-t)\left( \sum_d u_{i,d}[\bfx_2]_d + v_i  \right) \right)^2\\
    &\leq t \psi \left( \sum_d u_{i,d}[\bfx_1]_d + v_i  \right)^2 + (1-t) \psi \left( \sum_d u_{i,d}[\bfx_2]_d + v_i \right)^2\\
    &= t \psi \left( a_i(\bfx_1)  \right)^2 + (1-t) \psi \left( a_i(\bfx_2) \right)^2.
\end{align*}
The inequality uses convexity of $\psi(a)$ as a function of $a$.
\end{proof}

\subsection{Proof of Lemma \ref{lem:halfmonotonicfunction}}\label{sec:proofofhalfmonotonic}

Throughout, we assume $\sigma_u,\sigma_v$ and $\mu_v$ are fixed and suppress dependence on these parameters. Let $v_{\mu_u}(x)\coloneqq \Var[\psi(a(x))]$ where the variance is taken with respect to a distribution with parameter $\mu_u$. Then, $v_{\mu_u}(x) =v_{-\mu_u}(-x)$ since $\mu(x)$ and $\sigma^2(x)$ are unchanged by the transformation $\mu_u, x \to -\mu_u, -x$.

Suppose  $v_{\mu_u} \in \mcT_1$ for $\mu_u>0$, then for $x\leq0$, 
\[
v_{-\mu_u}(x) = v_{\mu_u}(-x) \geq v_{\mu_u}(x) = v_{-\mu_u}(-x),
\]
and for $x<y\leq0$,
\[
v_{-\mu_u}(y) = v_{\mu_u}(-y) \leq v_{\mu_u}(-x) = v_{-\mu_u}(x).
\]
In words, if $v_{\mu_u} \in \mcT_1$ then $v_{-\mu_u} \in \mcT_2$. It therefore suffices to consider the case when $\mu_u\geq0$.

We first show that if $x\geq 0, v_{\mu_u}(x)\geq v_{\mu_u}(-x)$. Henceforth, we assume $\mu_u \geq 0$ is fixed and suppress it notationally. From \citet{frey1999variational},
\begin{equation}
v(x) = \sigma(x)^2 \alpha(r(x)), \label{eqn:variance-of-activation}
\end{equation}
Here $r(x)=\mu(x)/\sigma(x)$. We define $h(r)=N(r)+r\Phi(r)$, where $N$ is the standard Gaussian pdf, $\Phi$ is the standard Gaussian cdf. We define $\alpha(r)= \Phi(r)+ r h(r) -h(r)^2$. 

As $\sigma(x)^2=\sigma(-x)^2$, it suffices to show $\alpha(r(x)) \geq \alpha(r(-x))$ for $x>0$. To show this, we first show that $r(x)\geq r(-x)$ for $x>0$, then show that $\alpha(r)$ is monotonically increasing.
\[
r(x) = \mu(x)/\sigma(x) = \mu(-x)/\sigma(-x) + 2\mu_ux/\sigma(-x) \geq \mu(-x)/\sigma(-x) =r(-x). 
\]
The inequality uses that both $\mu_u$ and $x$ are non-negative. It remains to show that $\alpha(r)$ is monotonically increasing. A straightforward calculation shows that,
\[
\alpha'(r) = 2h(r)(1-\Phi(r)).
\]
As $1-\Phi(r)>0$, we must show $h(r) \geq 0$. We have $\lim_{r \to -\infty} h(r) = 0$ and $h'(r) = \Phi(r)>0$, implying $h(r)>0$. We conclude $\alpha'(r)>0$ for all $r$, showing that $v_{\mu_u}(x)\geq v_{\mu_u}(-x)$ for $x \geq 0$.

To complete the proof, we must show that $v(x)$ is monotonically increasing for $x\geq0$.  As $\sigma(x)^2$ is increasing as a function of $x$ and $\alpha(r)$ is increasing as a function of $r$, $v(x)$ is increasing as a function of $x$ whenever $r(x)$ is increasing as a function of $x$. As $r'(x)= \frac{\sigma_v^2 \mu_u-\sigma_u^2\mu_v x}{\sigma(x)^3}$, this completes the proof if $\sigma_v^2 \mu_u-\sigma_u^2\mu_vx \geq 0$.  In particular, we need only consider cases when $\mu_v>0$. In this case, we write,
\begin{align} \label{eqn:beta}
    v(x) = \mu(x)^2 \beta(r(x))
\end{align}
where $\beta(r)=\alpha(r)/r^2$. Also in this region, we have the inequality,
\[
r'(x)\sigma(x) = \frac{\sigma_v^2 \mu_u-\sigma_u^2\mu_v x}{\sigma_u^2x^2+\sigma_v^2} \leq \frac{\sigma_v^2 \mu_u}{\sigma_u^2x^2+\sigma_v^2} \leq \frac{\sigma_v^2 \mu_u}{\sigma_v^2}  = \mu_u,
\]
which leads to $r'(x) \leq \mu_u/\sigma(x)$.

Differentiating \cref{eqn:beta},
\begin{align*}
v'(x) &= 2\mu_u \mu(x)\beta(r(x)) +\mu(x)^2 \left(\frac{\sigma_v^2 \mu_u-\sigma_u^2\mu_v x}{\sigma(x)^3}\right)\beta'(r(x)) \\
&\geq  2\mu_u\mu(x)\left(\beta(r(x)) +\frac{1}{2}r(x) \beta'(r(x))\right). 
\end{align*}
The inequality uses that $r(x)>0$, so that by \cref{lem:sgn-beta-derivative}, $\beta'(r(x))<0$. It suffices to show that $\beta(r) +\frac{1}{2} r \beta'(r)>0$ for $r>0$.
\[
\beta(r) +\frac{1}{2} r \beta'(r) = \beta(r) +\frac{1}{2} r \frac{d}{dr}\left(\frac{\alpha(r)}{r^2}\right) = \frac{\alpha(r)}{r^2} +\frac{1}{2} r\frac{\alpha'(r)r^2- 2r\alpha(r)}{r^4} = \frac{\alpha'(r)}{2r}\geq 0.
\]
We conclude that $v'(x)\geq 0$ for $x \geq 0$, implying that $v(x)$ is monotonically increasing in this region. This completes the proof that $v_{\mu_u}(x) \in \mcT_1$ for $\mu_u>0$. 
\begin{lem}\label{lem:sgn-beta-derivative}
For $\beta$ defined as in the proof of \cref{lem:halfmonotonicfunction} and for $r>0$, $\beta'(r)<0$ 
\end{lem}
\begin{proof}
For $r \neq 0$, $\beta'(r)= \left(-2\Phi(r)+2N(r)^2+2N(r)\Phi(r)\right)/r^3$. As $r>0$,
\[
\beta'(r) \leq 0 \Leftrightarrow I(r):= -\Phi(r)+N(r)^2+N(r)r\Phi(r) \leq 0.
\]
Rearranging \citep[7.1.13]{abramowitz1965handbook} yields:
 \begin{equation}
  1- \frac{2}{r+\sqrt{r^2+8/\pi}} N(r)  \le \Phi(r) < 1- \frac{2}{r+\sqrt{r^2+4}} N(r).
\end{equation}
for $r\geq 0.$
\begin{align}
I(r) & = -\Phi(r)+N(r)^2+rN(r)\Phi(r) \nonumber \\
& \leq -\Phi(r)+N(r)^2+rN(r)\left( 1- \frac{2}{r+\sqrt{r^2+4}} N(r)\right) \nonumber \\
& \leq  -1 + \frac{2}{r+\sqrt{r^2+8/\pi}} N(r)+N(r)^2+rN(r)\left( 1- \frac{2}{r+\sqrt{r^2+4}} N(r)\right) \nonumber \\
& = -1 +\frac{2}{r+\sqrt{r^2+8/\pi}}N(r)+rN(r)  +N(r)^2\left(1 - \frac{2r}{r+\sqrt{r^2+4}} \right) \label{eqn:Iupper1}
\end{align}
We now make use of numerous crude bounds which hold for $r>0$:
\begin{enumerate}
     \item $N(r) \leq 1/\sqrt{2\pi},$ 
     \item $\frac{2}{r+\sqrt{r^2+8/\pi}} \leq \sqrt{\pi/2},$
    \item $rN(r) \leq 1/\sqrt{2\pi e} $ 
    \item $\frac{2r}{r+\sqrt{r^2+4}}  \geq 0$.
\end{enumerate}
Plugging these into \cref{eqn:Iupper1},
\[
I(r) \leq -1 +\frac{\sqrt{\pi/2}}{\sqrt{2\pi}}+\frac{1}{\sqrt{2\pi e}}  +\frac{1}{2\pi}  = -\frac{1}{2} + \frac{1}{\sqrt{2\pi e}}  +\frac{1}{2\pi} \approx -0.098 < 0. \qedhere
\]
\end{proof}
\subsection{Proof of Lemma \ref{lem:positivecombinations}}\label{app:proof-positive-combinations}
\begin{proof}
Recall that \[\mcT_1=\{f\geq 0: \forall 0\leq b<a, f(a)\geq f(-a) \text{ and } f(b)\leq f(a)\}\] and \[\mcT_2=\{f\geq 0: \forall a< b \leq 0, f(a)\geq f(-a) \text{ and } f(b) \leq f(a)\}.\]
First, note that $\mcT_1,\mcT_2$ and the set of non-negative convex functions, $\mcC$ are all closed under addition and positive scalar multiplication. We can therefore write $f$ as a sum of three functions, $f(x) = t_1(x)+t_2(x)+s(x)$ with $t_1 \in \mcT_1,t_2 \in \mcT_2$ and $s\in \mcC.$ We prove the case when $a\leq c \le 0 \le-c\le b.$ The case $a\le -c \le 0 \le c \le b$ follows a symmetric argument.
\begin{align*}
f(c) &=  t_1(c)+t_2(c)+s(c) \text{ \quad (def.)}\\ 
& \leq t_1(c)+t_2(a)+s(c) \text{ \quad (second condition for $\mcT_2$)}\\
& \leq t_1(-c)+t_2(a)+s(c) \text{ \quad (first condition for $\mcT_1$)}\\
& \leq t_1(b)+t_2(a)+s(c)\text{ \quad (second condition for $\mcT_1$)}\\
& \leq t_1(b)+t_2(a)+\max(s(a),s(b)) \text{\quad ($s$ convex)}\\
& \leq t_1(b)+t_2(a)+s(a)+s(b) \\ 
& \leq t_1(a)+t_1(b)+t_2(a)+t_2(b)+s(a)+s(b) \text{\quad (non-negativity)} \\
&= f(a)+f(b). \qedhere
\end{align*}
\end{proof}

\section{Proof of Theorem \ref{thm:universal}}\label{app:deep-proofs}
We now restate and prove \Cref{thm:universal} from the main body:
\begin{thm}\label{thm:universal-supplement}
Let $A \subset \R^D$ be compact, and let $C(A)$ be the space of continuous functions on $A$ to $\R$. Similarly, let $C^+(A)$ be the space of continuous functions on $A$ to $\R_{\geq 0}$. Then for any $g \in C(A)$ and $h \in C^+(A)$, and any $\epsilon > 0$, for both the mean-field Gaussian and MC dropout families, there exists a 2-hidden layer ReLU NN such that 
\[
\sup_{\bfx \in A} |\Exp{\!}{f(\bfx)} - g(\bfx)| < \epsilon \quad \mathrm{\,and \,} \quad \sup_{\bfx \in A} |\Var[f(\bfx)] - h(\bfx)| < \epsilon,
\]
where $f(\bfx)$ is the (stochastic) output of the network.
\end{thm}

Our proof will make use of the standard universal approximation theorem for deterministic NNs as given in \citet{leshno1993multilayer}:

\begin{thm}[Universal approximation for deterministic NNs]\label{thm:universal-function-approximator}
Let $\psi(a) = \max(0, a)$. Then for every $g \in C(\R^D)$ and every compact set $A \subset \R^D$, for any $\epsilon > 0$ there exists a function $f \in S$ such that $\|g - f\|_\infty < \epsilon$. Here 
\begin{align*}
    S = \left\{ \sum_{i=1}^I w_i\psi \left(\sum_{d=1}^D u_{i,d} x_d + v_i \right) : I \in \N, w_i, u_{i,d}, v_i \in \R \right\}.
\end{align*}
\end{thm}

We first prove a useful lemma.

\begin{lem}\label{lem:relu-variance-bound} 
Let $\psi(a) = \max(0, a)$. Let $a$ be a random variable with finite first two moments. Then $\Var[\psi(a)] \leq \Var[a]$.
\end{lem}
\begin{proof}
For all $x,y \in \R$, we have $|x-y|^2 \geq |\psi(x) - \psi(y)|^2$. Consider two i.i.d.~copies of any random variable with finite first two moments, denoted $a_1$ and $a_2$. Then
\begin{align*}
    \Var[a_1] = \Exp{\!}{a_1^2} - \Exp{\!}{a_1}^2
    = \frac{1}{2}\Exp{\!}{a_1^2 + a_2^2 - 2a_1a_2}
    = \frac{1}{2}\Exp{\!}{|a_1 - a_2|^2}
    &\geq \frac{1}{2}\Exp{\!}{|\psi(a_1) - \psi(a_2)|^2}\\
    &= \Var[\psi(a_1)]. \hspace{.8cm} \qedhere
\end{align*}
\end{proof}
\subsection{Proof of Theorem 3 for $\Qffg$}

We prove \cref{thm:universal-supplement} for the fully-factorised Gaussian approximating family. We begin by proving results about 1HL networks within this family. The overall goal of these results is \cref{lem:mean-close-to-map}, which informally says that for any set of mean parameters for the weights, we can find a setting of the standard deviations of the weights, such that the mean output of the network is close to the output of the deterministic network, with weights equal to the mean parameters. Our proof of this proceeds in 3 parts: First, in \cref{lem:dropout-lowvariance}, we show that by making the standard deviation parameters sufficiently small, we can ensure that the variance of the output of the network is uniformly small on some compact set $A$. Next, in \cref{lem:samples-near-mean-param}, we show that again by choosing the standard deviation sufficiently small, we can show that most of the sample functions of the 1HL network are close to the function that would be obtained by using the mean parameters. Finally, in the proof of \cref{lem:mean-close-to-map}, we use Chebyshev's inequality and the triangle inequality to conclude that the mean of the network must also be close to the function defined by the mean parameters. 

These networks will be used to construct the desired 2HL network. 

\paragraph{Notation} Consider a 1HL ReLU NN with input $\bfx \in \R^D$ and output $\bff \in \R^K$. Let the network have $I$ hidden units and be parameterised by input weights $U \in \R^{I \times D}$, input biases $v \in \R^{I}$, output weights $W \in \R^{K \times I}$ and output biases $b \in \R^K$. Let $\theta  = (U, v, W, b)$. Denote the $k$\textsuperscript{th} output of the network by $f^{(k)}_\theta(\bfx)$. Consider a factorised Gaussian distribution over the parameters $\theta$ in the network. Let the means of the Gaussians be denoted $\bm{\mu} = (\mu_U, \mu_v, \mu_W, \mu_b)$, where e.g.~$\mu_U$ is a matrix whose elements are the means of $U$. Each mean is always taken to be $\in \R$. Let the standard deviations be denoted $\bm{\sigma} = (\sigma_U, \sigma_v, \sigma_W, \sigma_b)$. Each standard deviation is always taken to be $\in \R_{>0}$. 

The following lemma states that we can make the output of a 1HL BNN have low variance by setting the standard deviation of the weights to be small.
\begin{lem}\label{lem:variance-small}
Let $A \subset \R^D$ be a compact set and $f^{(k)}_\theta(\bfx)$ be the $k$\textsuperscript{th} output of a 1HL ReLU NN with a mean-field Gaussian distribution mapping from $A \to \R$. Fix any $\bm{\mu}$ and any $\epsilon > 0$. Let all the standard deviations in $\bm{\sigma}$ be equal to a shared constant $\sigma > 0$. Then there exists $\sigma'>0$ such that for all $\sigma < \sigma'$ and for all $\bfx \in A$, $\Var[\psi(f^{(k)}_\theta(\bfx))] < \epsilon$ for all $1\leq k \leq K$. 
\end{lem}

\begin{proof}
Define $a_i = \sum_{d=1}^D u_{i,d} x_d + v_i$, so that $f^{(k)}_\theta(\bfx) = \sum_{i=1}^I w_{k,i} \psi(a_i) + b_k$. Then
\begin{align*}
    \Var[f^{(k)}_\theta(\bfx)] &= \Var\left[\sum_{i=1}^I w_{k,i} \psi(a_i)\right] + \sigma^2\\
    &= \sum_{i=1}^I \sum_{j=1}^I \mathrm{Cov}\left(w_{k,i}\psi(a_i),w_{k,j}\psi(a_j)\right) + \sigma^2\\
    & \leq \sum_{i=1}^I \sum_{j=1}^I \left|\mathrm{Cov}\left(w_{k,i}\psi(a_i),w_{k,j}\psi(a_j)\right)\right| + \sigma^2\\
    &\leq  \sum_{i=1}^I \sum_{j=1}^I \sqrt{\Var[w_{k,i}\psi(a_i)] \Var[w_{k,j}\psi(a_j)]} + \sigma^2,\\ 
\end{align*}
where the final line follows from the Cauchy--Schwarz inequality. We now analyse each of the constituent terms. Since $w_{k,i}$ and $\psi(a_i)$ are independent,
\begin{align*}
    \Var[w_{k,i}\psi(a_i)] = \mu_{w_{k,i}}^2 \Var[\psi(a_i)] + \Exp{\!}{\psi(a_i)}^2 \sigma^2 + \sigma^2\Var[\psi(a_i)].
\end{align*}
As $A$ is compact, it is bounded, so there exists an $M$ such that $|x_d| \leq M$ for all $1 \leq d \leq D$. Using \cref{lem:relu-variance-bound}, and the mean-field assumptions, 
 \[
 \Var[\psi(a_i)] \leq \Var[a_i] = \sigma^2 \left( \sum_{d=1}^D x_d^2 + 1 \right) \leq \sigma^2 ( DM^2 +1 ).
 \]

Since $a_i$ is a linear combination of Gaussian random variables, we have that $a_i \sim \mathcal{N}(\mu_{a_i}, \sigma^2_{a_i})$, where $\mu_{a_i} = \sum_{d=1}^D \mu_{u_{i,d}}x_d + \mu_{v_i}$ and $\sigma^2_{a_i} = \sigma^2 \left(\sum_{d=1}^D x_d^2 +1 \right)$. Therefore, we have that \cite{frey1999variational}
\begin{align*}
    \Exp{\!}{\psi(a_i)}^2 &= \left(\mu_{a_i} \Phi \left(\frac{\mu_{a_i}}{\sigma_{a_i}}\right) + \sigma_{a_i} N \left(\frac{\mu_{a_i}}{\sigma_{a_i}}\right) \right)^2
    \leq \left( |\mu_{a_i}|\Phi \left(\frac{\mu_{a_i}}{\sigma_{a_i}}\right) + \sigma_{a_i} N \left(\frac{\mu_{a_i}}{\sigma_{a_i}}\right) \right)^2 \\
    & \leq \left( |\mu_{a_i}| + \frac{\sigma_{a_i}}{\sqrt{2\pi}} \right)^2.
\end{align*}
We can then upper bound $\Var[w_{k,i}\psi(a_i)]$ as follows:
\begin{align*}
    \Var[&w_{k,i}\psi(a_i)] \leq \mu_{w_{k,i}}^2 \sigma^2 ( DM^2 +1 ) + \left( |\mu_{a_i}| + \frac{\sigma_{a_i}}{\sqrt{2\pi}} \right)^2 \sigma^2 + \sigma^4 ( DM^2 +1 )\\
    &\leq  \mu_{w_{k,i}}^2 \sigma^2 ( DM^2 +1 ) + \left( M\sum_{d=1}^D |\mu_{u_{i,d}}| +|\mu_{v_i}| + \frac{\sqrt{\sigma^2(M^2D+1)}}{\sqrt{2\pi}} \right)^2 \sigma^2 + \sigma^4 ( DM^2 +1 )\\
    &\coloneqq v_{k,i}(\sigma). 
\end{align*}
The second inequality follows since $A$ is compact and we have $|\mu_{a_i}| \leq M\sum_{d=1}^D |\mu_{u_{i,d}}| +|\mu_{v_i}|$. Note that the upper bound $v_{k,i}(\sigma)$ is continuous and monotonically increasing in $\sigma$, and $v_{k,i}(0) = 0$. We can then upper bound the variance of the output:
\begin{align*}
    \Var[f^{(k)}_\theta(\bfx)] &\leq \sum_{i=1}^I \sum_{j=1}^I \sqrt{v_{k,i}(\sigma)v_{k,j}(\sigma)} + \sigma^2.
\end{align*}
We then choose $\sigma'$ such that for all $1 \leq k \leq K$ and for all $1 \leq i \leq I$, $v_{k,i}(\sigma') < \frac{\epsilon}{2I^2}$, and such that $\sigma'^2 < \frac{\epsilon}{2}$. Then 
\begin{align*}
    \Var[f^{(k)}_\theta(\bfx)] \leq I^2 \frac{\epsilon}{2I^2} + \sigma'^2 < \epsilon
\end{align*}
for $1 \leq k \leq K$. Finally, applying \cref{lem:relu-variance-bound}, we have $\Var[\psi(f^{(k)}_\theta(\bfx))] < \epsilon$ for $1 \leq k \leq K$.  
\end{proof}
The following lemma states that by setting the standard deviation of the weights to be sufficiently small, we can with high probability make the sampled BNN output close to the BNN output evaluated at the mean parameters.
\begin{lem}\label{lem:samples-near-mean-param}
Let $A \subset \R^D$ be any compact set. Fix any $\bm{\mu}$ and any $\epsilon, \delta > 0$. Let all the standard deviations in $\bm{\sigma}$ be equal to a shared constant $\sigma>0$. Then there exists $\sigma'>0$ such that for all $\sigma < \sigma'$, and for any $\bfx \in A$, 
\[\mathrm{Pr}\left( |\psi(f^{(k)}_{\bm{\mu}}(\bfx)) - \psi(f^{(k)}_\theta(\bfx))| > \epsilon \right) < \delta\] 
for all $1\leq k \leq K$.
\end{lem}
\begin{proof}
Let $\theta \in \R^P$. We first note that $\psi(f^{(k)}_\theta(\bfx))$ is continuous as a function from $A \times \R^P \to \R$, under the metric topology induced by the Euclidean metric on $A \times \R^P$. Next, define a ball in parameter space
\[
B_{\gamma} = \{ \theta : \| \theta - \bm{\mu} \|_2 < \gamma \}.
\]
Consider the closed ball of unit radius around $\bm{\mu}$, $\bar{B}_1$. Note that $\bar{B}_1$ is compact, and therefore $A \times \bar{B}_1$ is compact as a product of compact spaces. 

Since a continuous map from a compact metric space to another metric space is uniformly continuous, given $\epsilon > 0$, there exists a $0 < \tau < 1$ such that for all pairs $(\bfx_1, \theta_1), (\bfx_2, \theta_2) \in A \times \bar{B}_1$ such that $d((\bfx_1, \theta_1), (\bfx_2, \theta_2)) < \tau$, $|\psi(f^{(k)}_{\theta_1}(\bfx_1)) - \psi(f^{(k)}_{\theta_2}(\bfx_2))| < \epsilon$. Here $d(\cdot, \cdot)$ is the Euclidean metric on $A \times \R^P$. Since this is true for all $1 \leq k \leq K$, we can find a $0 < \tau < 1$ such that $|\psi(f^{(k)}_{\theta_1}(\bfx_1)) - \psi(f^{(k)}_{\theta_2}(\bfx_2))| < \epsilon$ holds for all $k$ simultaneously, by taking the minimum of the $\tau$ over $k$.

Now choose $\sigma'>0$ such that for all $\sigma < \sigma'$, $\mathrm{Pr}(\theta \in B_{\tau}) > 1 - \delta.$ This event implies $d((\bfx, \theta), (\bfx, \bm{\mu})) = \| \theta - \bm{\mu} \|_2 < \tau$. Furthermore, $\theta \in \bar{B}_1$, since $\tau < 1$. Hence $|\psi(f^{(k)}_{\bm{\mu}}(\bfx)) - \psi(f^{(k)}_\theta(\bfx))| < \epsilon$ holds for all $1 \leq k \leq K$.

\end{proof}
The following lemma shows that for 1HL networks, we can make $\Exp{\!}{\psi(f^{(k)}_\theta)}$ (the mean BNN output) close to $\psi(f^{(k)}_{\bm{\mu}})$ (the BNN output evaluated at the mean parameter settings) by choosing the standard deviation of the weights to be sufficiently small.

\begin{lem}\label{lem:mean-close-to-map}
Let $A \subset \R^D$ be any compact set. Then, for any $\epsilon > 0$ and any $\bm{\mu}$, there exists a $\sigma_1>0$ such that for any shared standard deviation $\sigma< \sigma_1$,
\[ \left\| \Exp{\!}{\psi(f^{(k)}_\theta)} - \psi(f^{(k)}_{\bm{\mu}})\right\|_\infty < \epsilon \]
for all $1\leq k \leq K$.
\end{lem}

\begin{proof}
For all $\bfx \in A$ and any $\theta^*$, by the triangle inequality 
\[
\left| \Exp{\!}{\psi(f^{(k)}_\theta(\bfx))} - \psi(f^{(k)}_{\bm{\mu}}(\bfx)) \right| \leq \left| \Exp{\!}{\psi(f^{(k)}_\theta(\bfx))} - \psi(f^{(k)}_{\theta^*}(\bfx)) \right| +\left| \psi(f^{(k)}_{\bm{\mu}}(\bfx)) - \psi(f^{(k)}_{\theta^*}(\bfx)) \right|.
\]
Applying \cref{lem:samples-near-mean-param} with $\epsilon' = \epsilon/2$ and $\delta=1/4$, we can find a $\sigma'$ such that for all $\sigma < \sigma'$, $\left| \psi(f^{(k)}_{\bm{\mu}}(\bfx)) - \psi(f^{(k)}_\theta(\bfx)) \right| \leq \epsilon/2$ with probability at least $3/4$. By \cref{lem:variance-small}, we can find a $\sigma''$ such that for all $\sigma < \sigma''$, $\Var[\psi(f^{(k)}_{\theta}(\bfx))] < \frac{\epsilon^2}{16K}$. Choose $0 < \sigma < \min(\sigma',\sigma'')$. We can apply Chebyshev's inequality to each random variable $\psi(f^{(k)}_{\theta}(\bfx))$,
\[
\mathrm{Pr}\left[\left|\psi(f^{(k)}_{\theta}(\bfx)) - \Exp{\!}{\psi(f^{(k)}_{\theta}(\bfx))}\right| >  \epsilon/2\right] < \frac{1}{4K}.
\]
Applying the union bound, the probability that $|\psi(f^{(k)}_{\theta}(\bfx)) - \Exp{\!}{\psi(f^{(k)}_{\theta}(\bfx))}| \leq  \epsilon/2$ for all $k$ simultaneously is at least $3/4$.
Therefore, for any $\bfx$ we can find a $\theta^*$ such that $|\psi(f^{(k)}_{\theta^*}(\bfx)) - \Exp{\!}{\psi(f^{(k)}_{\theta}(\bfx))}| \leq  \epsilon/2$ and $\left| \psi(f^{(k)}_{\bm{\mu}}(\bfx)) - \psi(f^{(k)}_{ \theta^*}(\bfx)) \right| \leq \epsilon/2$ simultaneously because both events occur with probability at least $1/2$ and therefore have a non-empty intersection. Therefore for all $\bfx$ and all $k$
\[
\left| \Exp{\!}{\psi(f^{(k)}_{\theta}(\bfx))} - \psi(f^{(k)}_{\bm{\mu}}(\bfx)) \right| \leq \epsilon. \qedhere
\]
\end{proof}

We can now complete the proof of theorem 3 for $\Qffg$.
\begin{proof}[Proof of \cref{thm:universal-supplement}]
Consider the case of a 2-hidden layer ReLU Bayesian neural network with 2 units in the second hidden layer. Denote the inputs to these units as $f^{(1)}_\theta(\bfx)$ and $f^{(2)}_\theta(\bfx)$ respectively, where $\theta$ are the parameters in the bottom two weight matrices and biases of the network. The output of the network can then be written as,
\begin{equation}\label{eqn:network-ouput}
f(\bfx) = s_1 \psi(f^{(1)}_{\theta}(\bfx))+s_2 \psi(f^{(2)}_{\theta}(\bfx)) +t
\end{equation}
where the $s_i$ are the weights in the final layer and $t$ is the bias. Taking expectations on both sides,
\[
\Exp{\!}{f(\bfx)} = \Exp{\!}{s_1 \psi(f^{(1)}_{\theta}(\bfx))}+\Exp{\!}{s_2 \psi(f^{(2)}_{\theta}(\bfx))} +\Exp{\!}{t}
\]
Choose $\mu_{s_1}=1, \mu_{s_2}= 0$, and note that $s_1$ is independent of $\theta$ by the mean field assumption. Then
\begin{equation}
\Exp{\!}{f(\bfx)} = \Exp{\!}{\psi(f^{(1)}_{\theta}(\bfx))}+\Exp{\!}{t}. \label{eqn:g_to_gtilde}
\end{equation}
Define $\mu_t= - \min_{\bfx' \in A} g(\bfx')$ (as $A$ is compact and $g$ is continuous, this minimum is well-defined). Define $\tilde{g}(\bfx)\geq 0$ to be $g(\bfx) -\min_{\bfx' \in A} g(\bfx')$. By the universal approximation theorem (\cref{thm:universal-function-approximator}) we can find a setting of the mean parameters, $\bm{\mu}$ in the first two layers (i.e.~excluding the parameters of the distributions on $s_1, s_2$ and $t$) such that
\[
 \|f^{(1)}_{\bm{\mu}}- \tilde{g}\|_\infty < \epsilon/2 \quad \mathrm{\, and \,} \quad \|f^{(2)}_{\bm{\mu}}- \sqrt{h}\|_\infty < \epsilon/2.
\]
This can be done by splitting the neurons in the first hidden layer into two sets, where the first and second set are responsible for $f^{(1)}, f^{(2)}$ respectively, and the weights from each set to the output of the other set are zero. Since $\tilde{g}(\bfx)>0$, applying the ReLU can only make $f^{(1)}$ closer to $\tilde{g}$. Hence $\|\psi(f^{(1)}_{\bm{\mu}})- \tilde{g}\|_\infty < \epsilon/2.$

By \cref{lem:mean-close-to-map}, we can find a $\sigma_1 > 0$ for this $\bm{\mu}$ such that when the standard deviations in the first two layers are set to any shared constant $\sigma < \sigma_1$,
\[ \left\| \Exp{\!}{\psi(f^{(1)}_\theta)} - \psi(f^{(1)}_{\bm{\mu}} )\right\|_\infty < \epsilon/2. \]
By the triangle inequality, $\left\| \Exp{\!}{\psi(f^{(1)}_\theta)}- \tilde{g} \right\|_\infty < \epsilon$. Combining with \cref{eqn:g_to_gtilde}, it follows that the expectation can approximate any continuous function $g$.

We now consider the variance of \cref{eqn:network-ouput}.
\begin{align*}
\Var[f(\bfx)] &= \Var[s_1 \psi(f^{(1)}_\theta(\bfx))+s_2 \psi(f^{(2)}_\theta(\bfx))] +\Var[t]\\
&= \Var[s_1 \psi(f^{(1)}_\theta(\bfx))]+\Var[s_2 \psi(f^{(2)}_\theta(\bfx))] + 2 \Cov(s_1 \psi(f^{(1)}_\theta(\bfx)),s_2 \psi(f^{(2)}_\theta(\bfx))) + \sigma_t^2.
\end{align*}
Choose $\sigma_t^2=\epsilon$. We now consider $\Var[s_1 \psi(f^{(1)}_\theta(\bfx))]$. As $s_1$ is independent of $\theta$,
\[
\Var[s_1 \psi(f^{(1)}_\theta(\bfx))]= \mu_{s_1}^2\Var[\psi(f^{(1)}_\theta(\bfx))]+\sigma_{s_1}^2 \Exp{\!}{\psi(f^{(1)}_\theta(\bfx))}^2 + \Var[\psi(f^{(1)}_\theta(\bfx))]\sigma^2_{s_1}.
\]
Recall $\mu_{s_1} =1$ and choose $\sigma_{s_1}^2 = \min\left(1,\epsilon \big/\left(\max_{x \in A}\Exp{\!}{\psi(f^{(1)}_\theta(\bfx))}^2\right)\right)$, then 
\[
\Var[s_1 \psi(f^{(1)}_\theta(\bfx))]\leq 2\Var[\psi(f^{(1)}_\theta(\bfx))] + \epsilon.
\]
By \cref{lem:variance-small}, we can find a $\sigma_2$ such that for any $\sigma<\sigma_2$, $\Var[\psi(f^{(1)}_\theta(\bfx))] \leq \epsilon$. For any such $\sigma$, $\Var[s_1 \psi(f^{(1)}_\theta(\bfx))]\leq 3 \epsilon$.

We now choose $\sigma^2_{s_2}=1$ and consider 
\begin{align*}
\Var[s_2 \psi(f^{(2)}_\theta(\bfx))] &= \mu_{s_2}^2\Var[\psi(f^{(2)}_\theta(\bfx))]+\sigma_{s_2}^2 \Exp{\!}{\psi(f^{(2)}_\theta(\bfx))}^2 + \sigma_{s_2}^2\Var[\psi(f^{(2)}_\theta(\bfx))]\\
&=\Exp{\!}{\psi(f^{(2)}_\theta(\bfx))}^2 + \Var[\psi(f^{(2)}_\theta(\bfx))].
\end{align*}

By \cref{lem:variance-small}, we can find a $\sigma_3$ such that for any $\sigma<\sigma_3$, $\Var[\psi(f^{(2)}_\theta(\bfx))]<\epsilon$.

By the universal function approximator theorem (\cref{thm:universal-function-approximator}) we can find a setting of the mean parameters, $\bm{\mu}$ in the first two layers such that $\|f^{(2)}_{\bm{\mu}}- \sqrt{h}\|_\infty < \epsilon/2$. Since $\sqrt{h(\bfx)}>0$, the ReLU can only make $f^{(2)}$ closer to $\sqrt{h}$, $\|\psi(f^{(2)}_{\bm{\mu}})- \sqrt{h})\|_\infty < \epsilon/2$. 

By \cref{lem:mean-close-to-map}, we can find a setting of $\sigma$ for this $\bm{\mu}$ such that
\[\left\| \Exp{\!}{\psi(f^{(2)}_\theta)} - \psi(f^{(2)}_{\bm{\mu}} )\right\|_\infty < \epsilon/2. \]
By the triangle inequality,
\[
\left\| \Exp{\!}{\psi(f^{(2)}_\theta)}- \sqrt{h} \right\|_\infty < \epsilon.
\]
This implies, 
\begin{align*}
\left\| \Exp{\!}{\psi(f^{(2)}_\theta)}^2- h \right\|_\infty &=  \left\| \left(\Exp{\!}{\psi(f^{(2)}_\theta)}- \sqrt{h}\right)\left(\Exp{\!}{\psi(f^{(2)}_\theta)}+ \sqrt{h}\right) \right\|_\infty \\
& \leq \epsilon \left\|\Exp{\!}{\psi(f^{(2)}_\theta)}+ \sqrt{h}\right\|_\infty\\
& \leq \epsilon (2\|\sqrt{h}\|_\infty+\epsilon)
\end{align*}

We therefore have,
\begin{align*}
\|\Var[f]- h\|_\infty &\leq E(\epsilon) + 2 \Cov(s_1 \psi(f^{(1)}_\theta(\bfx)),s_2 \psi(f^{(2)}_\theta(\bfx))) \\
&\leq E(\epsilon)+2 \sqrt{\Var[s_1 \psi(f^{(1)}_\theta(\bfx))]\Var[s_2 \psi(f^{(2)}_\theta(\bfx))]} \\
&\leq E(\epsilon)+C\sqrt{\epsilon} 
\end{align*}
where the first inequality is Cauchy-Schwarz, and $E(\epsilon)$ is a function that tends to zero with $\epsilon$ and $C$ is a constant. The theorem follows by choosing $\sigma < \min\{\sigma_1, \sigma_2, \sigma_3 \}$.
\end{proof}

The construction in our proof used a 2HL BNN with only two neurons in the second hidden layer. The construction still works for wider hidden layers, by setting the unused neurons to have zero mean and sufficiently small variance.

An analogous statement to \cref{thm:universal} for networks with more than two hidden layers can be proved inductively: applying \cref{thm:universal} for 2HL BNNs we can choose the variance to be uniformly small, thus satisfying the condition stated in \cref{lem:variance-small}. The proof of \cref{lem:samples-near-mean-param} applies equally for the output of 2HL BNNs. The rest of the proof then follows as stated.

\subsection{Proof of \cref{thm:universal-supplement} for MCDO}
In order to prove the universality result for deep dropout, we first prove two lemmas about 1HL dropout networks. The following lemma states that the mean of a 1HL dropout network is a universal function approximator, while its variance can simultaneously be made arbitrarily small.

\begin{lem} \label{lem:dropout-lowvariance}
Consider any $\epsilon>0$ and any continuous function, $m$ mapping from a compact subset, $A$ of $\R^D$ to $\R$. Then there exists a (random) ReLU neural network of the form
\[f(x)  = \sum_{i=1}^I w_{i}\gamma_i \psi\left( \sum_{d=1}^Du_{i,d}x_d+v_i\right)+b\]
with $\gamma_i \stackrel{\mathrm{i.i.d.}}{\sim} \textup{Bern}(1-p)$ such that $\|\Exp{\!}{f}-m\|_\infty <\epsilon$ and $\| \Var[f] \|_\infty \leq \epsilon$.
\end{lem}
\begin{proof}
By the universal approximation theorem \citet{leshno1993multilayer}, there exists a $J$ and 1HL network of the form,
\[
g(x) = \sum_{j=1}^J \tilde{w_j} \psi\left(\sum_{d=1}^D \tilde{u}_{j,d}x_d+v_j\right) +b,
\]
such that $\|g-m\|_\infty \leq \epsilon$.
Define the dropout network,
\[
f^{(1)}(x) = \sum_{j=1}^J \frac{\tilde{w_j}}{1-p} \psi\left(\sum_{d=1}^D \tilde{u}_{j,d}x_d+v_j\right) +b.
\]
Then $\Exp{\!}{f^{(1)}}=g$, so that $\|\Exp{\!}{f^{(1)}}-m\|_\infty \leq \epsilon$. Let $S=\|\Var[f^{(1)}]\|_\infty < \infty$. 

Define $f = \frac{1}{L}\sum_{\ell=1}^L f^{(1,\ell)}$ where each $f^{(1,\ell)}$ is an independent realisation of $f^{(1)}$. Then $\Exp{\!}{f} = g$ and $\Var[f] =\frac{\Var[f^{(1)}]}{\sqrt{L}} \leq \frac{S}{\sqrt{L}}$. $f$ can be realised by a dropout network by combining $L$ copies of $f^{(1)}$ together with identical weights within each copy and $0$ weights connecting the various copies. Choosing $L=(S/\epsilon)^2$ completes the proof. \qedhere

\end{proof}

The following lemma states that the mean of the MCDO network can approximate any continuous \emph{positive} function, after application of the ReLU non-linearity.

\begin{lem}\label{lem:dropout-variance-phi}
Given a positive mean function $m$ with $0< \delta \leq \|m\|_\infty \leq \Delta$ and a stochastic process $f$ such that $\|\Exp{\!}{f}-m\|_\infty \leq \epsilon \leq \delta$ and $\|\Var[f]\|_\infty \leq \epsilon$, 
\[
\|\Exp{\!}{\psi(f)} - m\|_\infty \leq \epsilon + \frac{\sqrt{\epsilon^2 +\epsilon\left(\Delta+\epsilon\right)^2}}{\delta-\epsilon} = \mcO(\Delta\sqrt{\epsilon}/(\delta-\epsilon))
\] 
and $\|\Var[\psi(f)]\|_\infty \leq \epsilon$. In the big-O notation, we assume $\Delta$ is bounded below by a constant and $\epsilon,\delta$ are bounded above by a constant.
\end{lem}
\begin{proof}
The bound $\|\Var[\psi(f)]\|_\infty \leq \epsilon$ follows from \cref{lem:relu-variance-bound}. We consider the expectation of $\psi(f(\bfx))$ for some arbitrary fixed $\bfx$,
\begin{align*}
    \left\lvert \Exp{\!}{\psi(f(\bfx))} - m(\bfx)\right\rvert &= \left\lvert  \Exp{\!}{f(\bfx)} - m(\bfx) -  \Exp{\!}{\min(0,f(\bfx))}\right\rvert 
    \\&\leq \left\lvert \Exp{\!}{f(\bfx)}  - m(\bfx) \right\rvert + \left\lvert\Exp{\!}{\min(0,f(\bfx))} \right\rvert  \\&
    \leq \epsilon +  \left\lvert\Exp{\!}{\min(0,f(\bfx))}\right\rvert.
\end{align*}
 We therefore bound $\left\lvert\Exp{\!}{\min(0,f(\bfx))}\right\rvert$.
\begin{align*}
    \left\lvert\Exp{\!}{\min(0,f(\bfx))}\right\rvert = \left\lvert\Exp{\!}{f(\bfx) \mathbf{1}\{\bfx: f(\bfx)<0\}}\right\rvert \leq \sqrt{\Exp{\!}{f(\bfx)^2}\mathrm{Pr}(f(\bfx)<0)}.
\end{align*}
The inequality uses Cauchy-Schwarz, that the square of an indicator function is itself and reinterprets the expectation of an indicator function as a probability.  We bound the two terms on the RHS separately.
\begin{align*}
    \Exp{\!}{f(\bfx)^2} = \Var[f(\bfx)] + \Exp{\!}{f(\bfx)}^2  \leq \epsilon + \Exp{\!}{f(\bfx)}^2 \leq \epsilon +\left(m(\bfx)+\epsilon\right)^2\leq \epsilon +\left(\Delta+\epsilon\right)^2
\end{align*}
We use Chebyshev's inequality to bound the probability $f(x)<0$,
\begin{align*}
    \mathrm{Pr}(f(\bfx)<0) &\leq \mathrm{Pr}\left(|f(\bfx)-\Exp{\!}{f(\bfx)}| > m(\bfx)-\epsilon\right)\\
    &\leq \frac{\Var[f(\bfx)]}{(m(\bfx)-\epsilon)^2}\\
    &\leq \frac{\epsilon}{(m(\bfx)-\epsilon)^2}\\
    &\leq \frac{\epsilon}{(\delta-\epsilon)^2}. \hspace{1cm} \qedhere
\end{align*}
\end{proof}
Having collected the necessary lemmas, we provide a construction that proves \cref{thm:universal-supplement}. 

\begin{proof}[Proof of \cref{thm:universal-supplement}]
Consider a 2HL dropout NN. Let the pre-activations in the first hidden layer be collectively denoted $\bfa_1$, and the random dropout masks by $\bm{\epsilon_1}$. Let the second hidden layer have $I+2$ hidden units. Let $\odot$ denote the elementwise product of two vectors of the same length.  Define the pre-activations of two of the second hidden layer units by $a_v = {\bfw^{\mathsf{T}}_v}(\bm{\epsilon_1} \odot \psi(\bfa_1))$, i.e.~both these hidden units have identical weight vectors $\bfw_v$ and dropout masks, and are hence the same random variable. Similarly, let the remaining $I$ second hidden layer pre-activations be defined by $a_m = \bfw_m^{\mathsf{T}}(\bm{\epsilon_1} \odot \psi(\bfa_1))$, again all being the same random variable. Furthermore, let $(\bfw_{v})_i = 0$ whenever $(\bfw_{m})_i \neq 0$ and vice versa, so that the first hidden layer neurons that influence $a_v$ and those that influence $a_m$ form disjoint sets. Then the output of the 2HL network is:
\[
f = \epsilon_a w_{2,a} \psi(a_v) + \epsilon_b w_{2,b} \psi(a_v) + \sum_{i=1}^I \epsilon_{i} w_{2,i} \psi(a_m) + b_{2},
\]
where $\epsilon_a, \epsilon_b, \{ \epsilon_i \}_{i=1}^I$ are the final layer dropout masks and $\{w_{2,i}\}_{i=1}^I, b_{2}$ are the final layer weights and bias. 

We now make the choices $ w_{2,a} = 1, w_{2,b} = -1, w_{2,i} = \alpha$, where $\alpha I = 1/(1-p)$. Then $\Exp{\!}{f} = \Exp{\!}{\psi(a_m)} + b_2$. Let $b_2 = \min_{\bfx \in A} g - \delta$, where $\delta >0$ and the min exists due to compactness of $A$. Define $g' = g - b_2$. Since $a_m$ is just the output of a single-hidden layer dropout network, for any $\gamma' >0$ we can use \cref{lem:dropout-lowvariance} to choose $\|\Exp{\!}{a_m} - g'\|_{\infty} < \gamma'$ and $\|\Var[a_m] \|_{\infty} < \gamma'$. Since $g'$ is bounded below by $\delta$ and bounded above by some $\Delta \in \R$ (by continuity of $g$ and compactness of $A$), we can then apply \cref{lem:dropout-variance-phi} to obtain $\| \Exp{\!}{a_m} - g' \|_{\infty} = \mcO(\Delta \sqrt{\epsilon'}/(\delta- \epsilon'))$ and $\| \Var[\psi(a_m)] \|_\infty < \gamma'$. We can use this to bound the error in the mean of the 2HL network output:
\begin{align*}
    \|\Exp{\!}{f} - g\|_{\infty} &= \| \Exp{\!}{\psi(a_m)} + b_2 - g\|_{\infty}
    =  \|\Exp{\!}{\psi(a_m)} - g'\|_{\infty}
    = \mcO(\Delta \sqrt{\gamma'}/(\delta- \gamma')).
\end{align*}
We can choose $\gamma'$ to depend on $\delta, \Delta$ such that $ \|\Exp{\!}{f} - g\|_{\infty} <\gamma $, proving the first part of the theorem. Next, calculating the variance,
\begin{align}
    \Var[f] &= \Var\left[(\epsilon_a - \epsilon_b)\psi(a_v) + \alpha \psi(a_m) \sum_{i=1}^I \epsilon_i\right] = \Var[(\epsilon_a - \epsilon_b)\psi(a_v)] + \alpha^2\Var\left[\psi(a_m) \sum_{i=1}^I \epsilon_i\right]. \label{eqn:dropout-variance-terms}
\end{align}
Next we show that by taking $I$ sufficiently large, we can make the second term arbitrarily small. We have,
\begin{align*}
    \Var\left[\psi(a_m) \sum_{i=1}^I \epsilon_i\right] &=  \Var[\psi(a_m)] \Var\left[\sum_{i=1}^I \epsilon_i \right] + \Var[\psi(a_m)] \Exp{\!}{\sum_{i=1}^I \epsilon_i}^2 +  \Var\left[\sum_{i=1}^I \epsilon_i \right] \Exp{\!}{\psi(a_m)}^2 \\
    &= \Var[\psi(a_m)] Ip(1-p) + \Var[\psi(a_m)] I^2(1-p)^2 + Ip(1-p) \Exp{\!}{\psi(a_m)}^2\\
    &\leq \gamma'  Ip(1-p) + \gamma' I^2(1-p)^2 + Ip(1-p) \Exp{\!}{\psi(a_m)}^2
\end{align*}
The first two of these three terms can be made arbitrarily small by choosing $\gamma'$ sufficiently small. The third term, upon multiplying by $\alpha^2$, becomes
\begin{align*}
    \alpha^2 Ip(1-p) \Exp{\!}{\psi(a_m)}^2 &= \frac{p}{I(1-p)}\Exp{\!}{\psi(a_m)}^2,
\end{align*}
which can also be made arbitrarily small by choosing $I \in \N$ sufficiently large. We now show that the first term in \cref{eqn:dropout-variance-terms} can well approximate our target variance function $h$.
\begin{align}
    \Var[(\epsilon_a - \epsilon_b)\psi(a_v)] &= \Var[\epsilon_a - \epsilon_b] \Var[\psi(a_v)] + \Var[\epsilon_a - \epsilon_b] \Exp{\!}{\psi(a_v)}^2 + \Var[\psi(a_v)] \Exp{\!}{\epsilon_a - \epsilon_b}^2\\
    &= 2p(1-p) \Var[\psi(a_v)] + 2p(1-p) \Exp{\!}{\psi(a_v)}^2 \label{eqn:dropout-target-variance-approx}
\end{align}
Define
\begin{align*}
    h' = \sqrt{\frac{h}{2p(1-p)}} + \delta',
\end{align*}
for some $\delta' > 0$. Again applying \cref{lem:dropout-lowvariance} (which we can do independently of the choice of $a_m$ since neurons influencing $a_v$ and $a_m$ are disjoint), for any $\gamma'' > 0$ we can choose $\left\|\Exp{\!}{a_v} - h' \right\|_{\infty} < \gamma''$ and $\| \Var[a_v] \|_{\infty} < \gamma''$. The first term in \cref{eqn:dropout-target-variance-approx} can be made arbitrarily small by choosing $\gamma''$ small enough. We can again apply \cref{lem:dropout-variance-phi} so that $\left\|\Exp{\!}{\psi(a_v)} - h' \right\|_{\infty} = \mcO(\Delta'\sqrt{\gamma''}/(\delta' - \gamma''))$. We then bound the difference between the second term in \cref{eqn:dropout-target-variance-approx} and our target variance function:
\begin{align}
    \left\| 2p(1-p) \Exp{\!}{\psi(a_v)}^2 - h \right\|_{\infty} &\!\leq \left\| \sqrt{2p(1-p)} \Exp{\!}{\psi(a_v)} \!+\! \sqrt{h} \right\|_{\infty} \! \left\| \sqrt{2p(1-p)} \Exp{\!}{\psi(a_v)} - \sqrt{h} \right\|_{\infty} \label{eqn:submultiplicativity}\\
    &\hspace{-3.5cm}\leq  \left( \left\|2 \sqrt{h} \right\|_\infty + \left\| \sqrt{2p(1-p)} \Exp{\!}{\psi(a_v) } - \sqrt{h}\right\|_\infty  \right) \left\| \sqrt{2p(1-p)} \Exp{\!}{\psi(a_v)} - \sqrt{h} \right\|_{\infty} \label{eqn:dropout-final}
\end{align}
where \cref{eqn:submultiplicativity} follows from sub-multiplicativity of the infinity norm. Expanding the second term in \cref{eqn:dropout-final},
\begin{align*}
    \left\| \sqrt{2p(1-p)} \Exp{\!}{\psi(a_v)} - \sqrt{h} \right\|_{\infty} &= \sqrt{2p(1-p)}\left\|  \Exp{\!}{\psi(a_v)} - h' + \delta' \right\|_{\infty} \\ &= \mcO(\delta'+\Delta'\sqrt{\gamma''}/(\delta' - \gamma''))
\end{align*}
By first choosing $\delta'$ sufficiently small, and then choosing $\gamma''$ depending on $\delta'$, we can make this error term arbitrarily small. Since all the other contributions to $\Var[f]$ were made arbitrarily small, this allows us to set $\|\Var[f] - h \| < \gamma$, for any $\gamma>0$, completing the proof. \qedhere

\end{proof}

In order to provide an analogous construction for MCDO BNNs with more than 2 hidden layers, we note that the above proof only requires a BNN output with a universal mean function and an arbitrarily small variance function in \cref{lem:dropout-lowvariance}. Instead of a 1HL network, we can apply \cref{thm:universal} to construct a 2 or more hidden layer network to provide these mean and variance functions. The rest of the proof then follows as in the 2HL case.

\section{Dropout With Inputs Dropped Out}\label{app:dropout-drop-inputs}

The behaviour of MC dropout with inputs dropped out is somewhat different, both theoretically and empirically, from the case when inputs are not dropped out as discussed in the main body. 

\subsection{Single-Hidden Layer Networks}
In the single-hidden layer case, the variance is no longer convex as a function of $\bfx$. On the other hand, this approximating family still struggles to represent in-between uncertainty:
\begin{thm}[MC dropout, dropped out inputs]\label{thm:dropout-fixed}
    Consider a single-hidden layer ReLU neural network mapping from $\R^D \to \R^K$ with $I \in \N$ hidden units. The corresponding mapping is given by $f^{(k)}(\bfx) =  \sum_{i=1}^I w_{k,i}\psi \left( \sum_{d=1}^D u_{i,d}x_d+v_i \right)+b_k$ for $1\leq k \leq K$, where $\psi(a) = \max(0, a)$.
    Assume $\bfv$ is set deterministically and 
\begin{align*}
    q(\mathbf{W}, \mathbf{b},\bfU) = q(\bfU)q(\mathbf{b}|\bfU)\prod_{i}q_{i}(\bfw_{i}|\bfU),
\end{align*}
where $\bfw_i = \{w_{k,i}\}_{k=1}^K$ are the weights out of neuron $i$, $\bfb = \{b_k\}_{k=1}^K$ are the output biases and $q(\bfU),$  $q(\bfb|\bfU)$ and $q_i(\bfw_i|\bfU)$ are arbitrary probability densities with finite first two moments. Then, for any finite set of points $\mcS \subset \R^D$ such that $\mathbf{0}$ is in the convex hull of $\mcS$,
\begin{equation}
\Var[f^{(k)}(\mathbf{0})] \leq \max\limits_{\bfs \in \mcS}\left\{\Var[f^{(k)}(\bfs)]\right\} \text{\quad \emph{for} \, \,} 1\leq k \leq K.
\end{equation}
\end{thm}
In order to prove \cref{thm:dropout-fixed} we use the following simple lemma,
\begin{lem}\label{lem:convexhullbound}
 Let $f: \R^D \to \R$ be a convex function and consider a finite set of points $\mcS \subset \R^D$. Then for any point $\bfr$ in the convex hull of $\mcS$, $f(\bfr) \leq \max \limits_{\bfs \in \mcS}\{f(\bfs)\}$. 
\end{lem}
\begin{proof}[Proof of \cref{lem:convexhullbound}]
Let $\{ \bfs_n \}_{n=1}^N = \mcS_N \subset \R^D$. We proceed by induction. The lemma is true for $N=2$ by the definition of convexity. Assume it is true for $N$. Let $\mathrm{Conv}(\mcS_{N+1})$ denote the convex hull of $\mcS_{N+1}$. Consider a point $\bfr_{N+1} \in \mathrm{Conv}(\mcS_{N+1})$. Then
\begin{align}
    f(\bfr_{N+1}) = f\left( \sum_{n=1}^{N+1} \alpha_n \bfs_n \right)
\end{align}
with $\sum_{n=1}^{N+1} \alpha_n = 1$ and $\alpha_n \geq 0$ for $1 \leq n \leq N+1$. We can write
\begin{align}
    f(\bfr_{N+1}) &= f\left( \left(\sum_{n=1}^N \alpha_{n}\right)\bft_N  + \alpha_{N+1} \bfs_{N+1} \right) \leq \max \{ f(\bft_N), f(\bfs_{N+1}) \} \label{eqn:convexhullbound}
\end{align}
where $\bft_N \coloneqq \sum_{n=1}^N  \alpha_n \bfs_n \Big/ \sum_{n=1}^N \alpha_{n}  $, and we have used the convexity of $f$. By the induction assumption, $f(\bft_N) \leq \max \limits_{\bfs \in \mcS_N}\{f(\bfs)\}$, since $\bft_N \in \mathrm{Conv}(\mcS_N)$. Combining this with \cref{eqn:convexhullbound} completes the proof. \qedhere
\end{proof}

\begin{proof}[Proof of \cref{thm:dropout-fixed}]
By the law of total variance,
\[ \Var[f^{(k)}(\mathbf{x})] = \mathbb{E}[ \Var[f^{(k)}(\mathbf{x})|\mathbf{U}]] + \Var[\mathbb{E}[f^{(k)}(\mathbf{x})|\mathbf{U}]]. \]
Using \cref{lem:convexconditioned}, $\Var[f^{(k)}(\mathbf{x})|\mathbf{U}]$ is convex as a function of $\bfx.$ As the expectation of a convex function is convex, the first term is a convex function of $\bfx.$ This implies 
\[
 \mathbb{E}[ \Var[f^{(k)}(\mathbf{0})|\mathbf{U}]] \leq \max\limits_{\bfs \in \mcS}\left\{\mathbb{E}[\Var[f^{(k)}(\bfs)|\bfU]]\right\},
\]
by \cref{lem:convexhullbound}. $\Var[\mathbb{E}[f^{(k)}(\mathbf{x})|\mathbf{U}]]$ is non-negative everywhere. As the output of the first layer is independent of the matrix $\bfU$ at $\bfx=\mathbf{0},$ $\mathbb{E}[f^{(k)}(\mathbf{0})|\mathbf{U}]$ is deterministic. So $\Var[\mathbb{E}[f^{(k)}(\mathbf{0})|\mathbf{U}]] = 0$, completing the proof.
\end{proof}
\begin{figure}[ht]
\begin{center}
\includegraphics[width=0.2\columnwidth,trim={0 .22cm 0 0},clip]{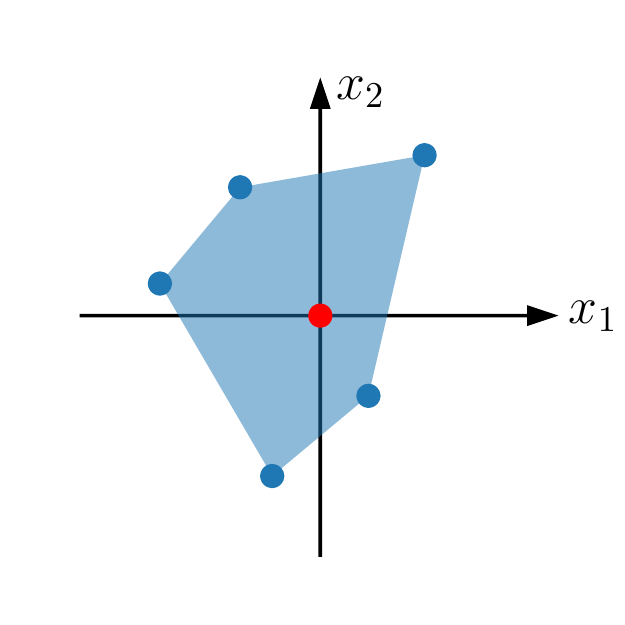}
\caption{Schematic illustration of the bound in \cref{thm:dropout}, showing the input domain of a single-hidden layer MC dropout BNN, for the case $\bfx \in \R^2$. The convex hull (in blue) of the blue points contains the origin. Hence the variance at the origin cannot exceed the variance at any of the blue points.}
\label{fig:dropout-illustration}
\end{center}
\end{figure}
\subsection{Deep Networks}
In the case when the network has several hidden layers, dropout with inputs dropped defines a posterior with somewhat strange properties, as observed in \citet[Section 4.2.1]{gal2016uncertainty}. In particular, in $D$ dimensions, a typical sample function from the approximate posterior will be constant as a function of roughly $pD$ of the input dimensions. However, which dimensions it is constant along depends on the particular sample. This behaviour is unlikely to be shared by the exact posterior. We are able to exploit this type of behaviour to show that if inputs are dropped out, there are simple combinations of mean and variance functions that cannot be simultaneously approximated by the corresponding approximating family. 

\begin{prop}
Consider $f$ the (stochastic) output of an MC dropout network of arbitrary depth with inputs dropped out. For any $x,x' \in \R$ such that $\Var[f(x)],\Var[f(x')]<\epsilon^2$, $|\Exp{\!}{f(x)}-\Exp{\!}{f(x')}| \leq 2\epsilon\sqrt{2/p}$. 
\end{prop}
\begin{proof}
With probability $p$, the input is dropped out, so $\mathrm{Pr}(f(x)=f(x')) \geq p$. We apply Chebyshev's inequality giving the bounds,
\begin{equation*}
    \mathrm{Pr}(|f(x) - \Exp{\!}{f(x)} | \leq r \epsilon) \geq 1-1/r^2 \quad \text{ and } \quad  \mathrm{Pr}(|f(x') - \Exp{\!}{f(x')} | \leq r \epsilon) \geq 1-1/r^2. 
\end{equation*}
for any $r>0$. Choose $r=\sqrt{2/p}+\delta$ for any $\delta>0$, then there exists a realisation of the dropout network such that $|f(x) - \Exp{\!}{f(x)} | \leq r \epsilon$, $|f(x') - \Exp{\!}{f(x')} | $ and $f(x)=f(x')$ simultaneously. Consequently,
\begin{align*}
    |\Exp{\!}{f(x)}-\Exp{\!}{f(x')}| & = |\Exp{\!}{f(x)} - f(x) + f(x) - \Exp{\!}{f(x')}| \\
    &= |\Exp{\!}{f(x)} - f(x) + f(x') - \Exp{\!}{f(x')}| \\
    & \leq |\Exp{\!}{f(x)} - f(x)| + |f(x') - \Exp{\!}{f(x')}| \\
    & \leq 2r \epsilon = 2\epsilon\sqrt{2/p}+2\epsilon \delta. 
\end{align*}
Taking the limit as $\delta \to 0$ completes the proof.\qedhere
\end{proof}

In other words we can bound the difference in the mean output at two points in terms of the uncertainty at those points and the dropout probability. 

In $D>1$ dimensions, we can get similarly tight bounds on lines parallel to a coordinate axis: for $\bfx,\bfx'$ on such a line $\mathrm{Pr}(f(\bfx)=f(\bfx')) \geq p$ still holds. If the dimension on which $\bfx$ and $\bfx'$ differ is dropped out $f(\bfx)=f(\bfx')$.

Alternatively in $D$ dimensions for arbitrary $\bfx,\bfx'\in \R^D$, $\mathrm{Pr}(f(\bfx)=f(\bfx')) \geq p^D$. This comes from noting that with probability $p^D$ the output of the network is a constant function. However, we note this bound becomes exponentially weak as the input dimension increases.

\subsection{Details of Experiments Minimising Squared Loss}\label{app:approximator-details}
We generated a dataset that consisted of two separated clusters in one dimension. We then fit a Gaussian process to the dataset and computed the predictive mean and variance on a one-dimensional grid of $N=40$ points, call these point $X$. Let $\mu(X) \in \R^N$ denote the mean of the GP regression at these points $\sigma^2(X) \in \R^N$ denote its variance. We define a loss function as
\[
\mcL(\phi) = \|\Exp{q_\phi}{f(X)}-\mu(X)\|_2^2+ \|\Var_{q_\phi}[f(X)]-\sigma^2(X)\|_2^2.
\]
The expectation and variance are Monte Carlo estimated using 128 samples. We use full batch optimisation with ADAM with learning rate $1 \times 10^{-3}$ for 50,000 iterations. A dropout rate of $0.05$ is used for MCDO. Weights and biases are initialized at the prior for MFVI.

\section{Details and Additional Figures for Section \ref{sec:deep-empirical}}\label{app:details-effect-of-depth}
In this appendix, we provide details of the protocol used to generate \cref{fig:deep-box}.

\subsection{Experimental Details}

\paragraph{Data:} We consider the dataset from \cref{fig:2D_dataset} with $\bfx \in \overrightarrow{\bfp\bfq}$, where $\bfp=(-1.2,-1.2)$ and $\bfq=(1.2,1.2)$ i.e.~the line segment between and including the two data clusters. We evaluate the overconfidence ratio on a discretisation of $\overrightarrow{\bfp\bfq}$.

\paragraph{Choosing the Prior:} For each depth a fully-connected ReLU network with 50 hidden units per layer is used. The prior mean for all parameters is chosen to be $0$. The prior standard deviation for the bias parameters is chosen as $\sigma_b=1$ for all experiments. In \cref{fig:deep-box}, the prior weight standard deviation is selected so that the prior standard deviation in function space at the region containing data is approximately constant. In particular, let $\sigma_w/\sqrt{H}$ be the prior standard deviation of each weight, where $H$ is the number of inputs to the weight matrix. We choose $\sigma_w = \{4, 3, 2.25, 2, 2, 1.9, 1.75, 1.75, 1.7, 1.65\}$ for depths 1-10 respectively, which ensures the prior standard deviations (of both the GP and the BNN) in function space at the points $(1,1)$ and $(-1,-1)$ (the centres of the data clusters) are between $10$ and $15$. Choosing a fixed $\sigma_w$ such as $\sigma_w=4$ for all depths would have caused the prior standard deviation in function space to grow unreasonably large with increasing depth; see \citet{schoenholz2016deep}. All models used a fixed Gaussian likelihood with standard deviation $0.1$.

\paragraph{Fitting the GP:} The Gaussian process was implemented using GPFlow \citep{GPflow2017} with the infinite-width ReLU BNN kernel implemented following \cite{lee2017deep}. All hyperparameters were fixed and exact inference was performed using Cholesky decomposition. 

\paragraph{Fitting MFVI:} We initialize the standard deviations of weights to be small and train for many epochs, following \citet{tomczak2018neural, swaroop2019improving} who found this led to good predictive performance. The weight means in each weight matrix were initialised by sampling from $\mcN(0, 1/\sqrt{2n_{\mathrm{out}}})$, where $n_{\mathrm{out}}$ is the number of outputs of the weight matrix. The weight standard deviations were all initialised to a very small value of $1 \times 10^{-5}$, (we tried a larger initialization with weight standard deviations initialized to $1\times 10^{-2.5}$ and found no significant difference). Bias means were initialised to zero, with the variances initialised to the same small value as the weight variance. 100,000 iterations of full batch training on the dataset were performed using ADAM with a learning rate of $1 \times 10^{-3}$. The ELBO was estimated using 32 Monte Carlo samples during training. The local reparameterisation trick was used \citep{kingma2015variational}. The predictive distribution at test time was estimated using 500 samples from the approximate posterior.

\paragraph{Fitting MCDO:} The weights and biases were initialised using the default \verb torch.linear ~initialisation. The dropout rate was fixed at $p=0.05$. The $\ell^2$ regularisation parameter was set following \citet[Section 3.2.3]{gal2016uncertainty} for the given prior, in such a way that the `KL condition' is met, in the interpretation of dropout training as variational inference. 100,000 iterations of full batch training on the dataset were performed using ADAM with a learning rate of $1 \times 10^{-3}$. The dropout objective was estimated using 32 Monte Carlo samples during training. The predictive distribution at test time was estimated using 500 samples from the approximate posterior. 

\paragraph{Fitting HMC:} For HMC on the 1HL BNN, 250,000 samples of HMC were taken using the NUTS implementation in Pyro \citep{hoffman2014no,bingham2018pyro} after 10,000 warmup steps. For the 2HL case, 1,000,000 samples of HMC were taken after 20,000 warmup steps. We set the maximum tree depth in NUTS to 5, and adapt the step size and mass matrix during warmup.

\subsection{Additional Figures}
In order to assess the robustness of our findings to different choices of prior, we also consider the same experiment with $\sigma_w=\sqrt{2}$ for all depths. We choose this prior as it leads to similar variances in function space as depth increases \citep{schoenholz2016deep}. We note that the variance of this prior is significantly smaller than the variance of the prior in the previous setting. The corresponding box plot is shown in \cref{fig:deep-box-root2}. With this prior both methods tend to be less over-confident between data clusters, but more underconfident at the data, especially as depth increases (see \cref{fig:ratio-function}). 
\begin{figure}
    \centering
    \includegraphics[width=.5\textwidth]{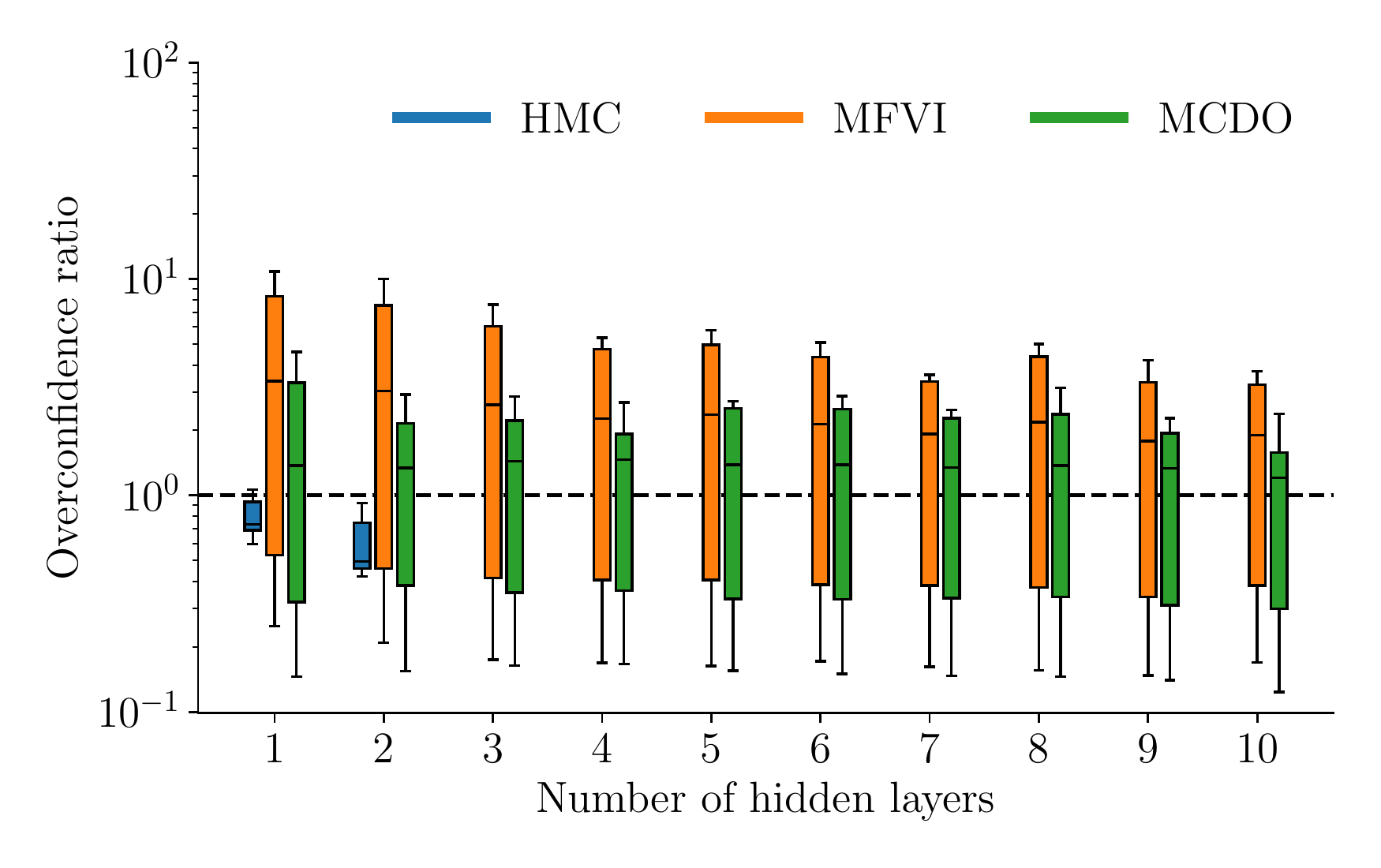}
    \caption{Boxplots of the overconfidence ratios of HMC, MFVI and MCDO  relative to exact inference in an infinite width BNN (GP) with $\sigma_w=\sqrt{2}$.}
    \label{fig:deep-box-root2}
\end{figure}
\begin{figure}
    \centering
    \begin{subfigure}[b]{0.4\textwidth}
    \includegraphics[width=\textwidth]{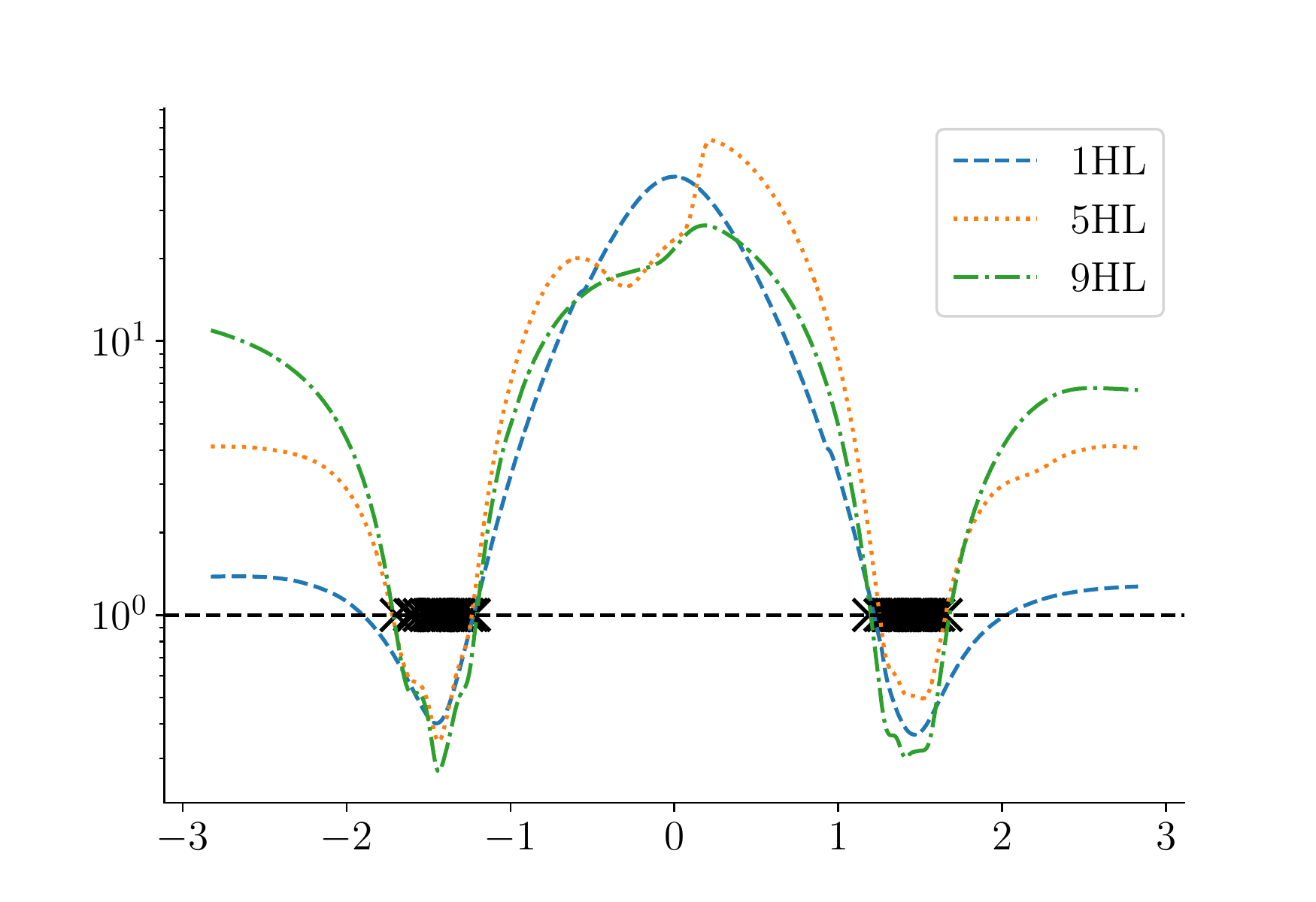}
    \caption{Mean Field VI}
    \end{subfigure}
    \begin{subfigure}[b]{0.4\textwidth}
    \includegraphics[width=\textwidth]{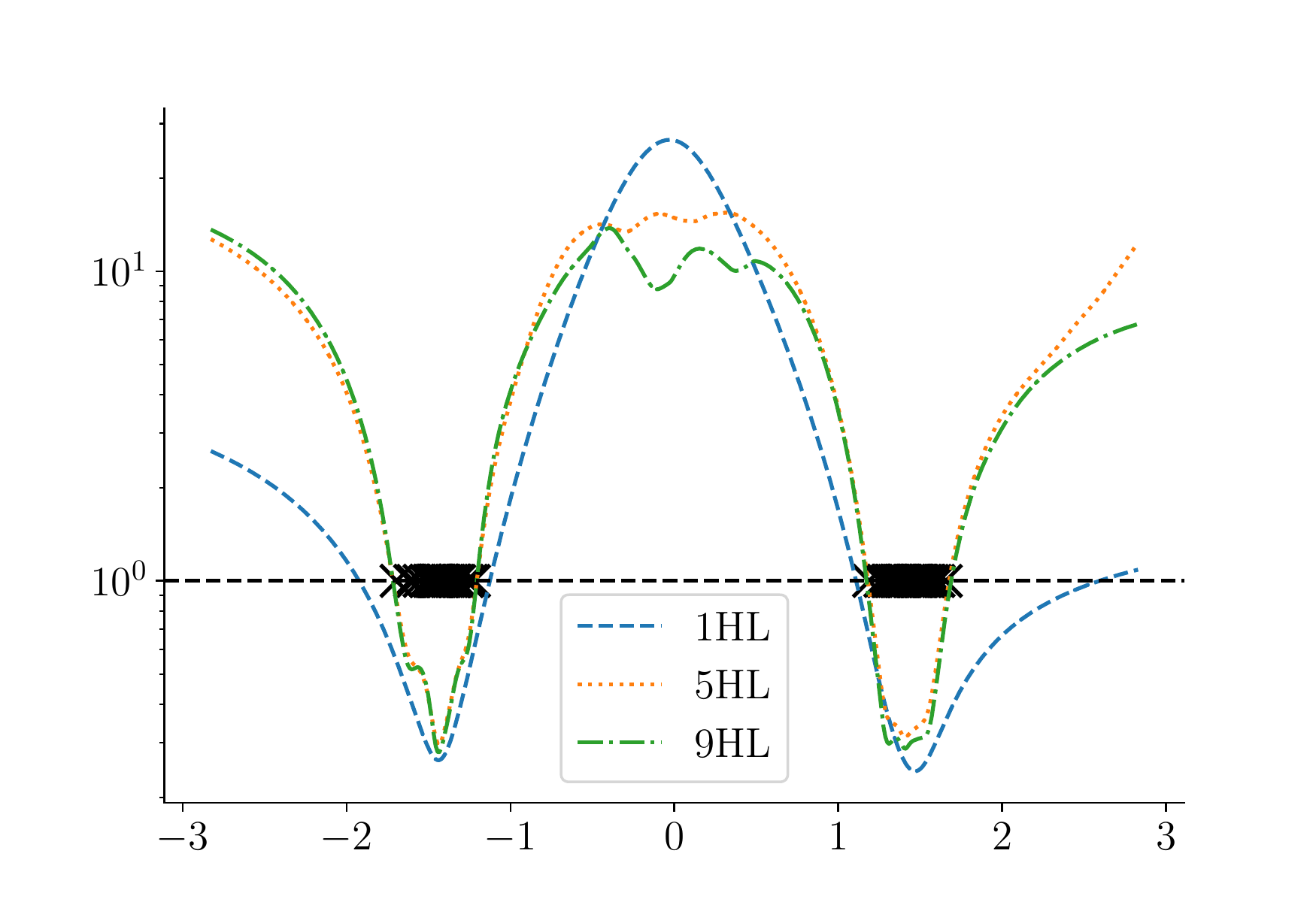}
        \caption{MC Dropout}
    \end{subfigure}
    \begin{subfigure}[b]{0.4\textwidth}
    \includegraphics[width=\textwidth]{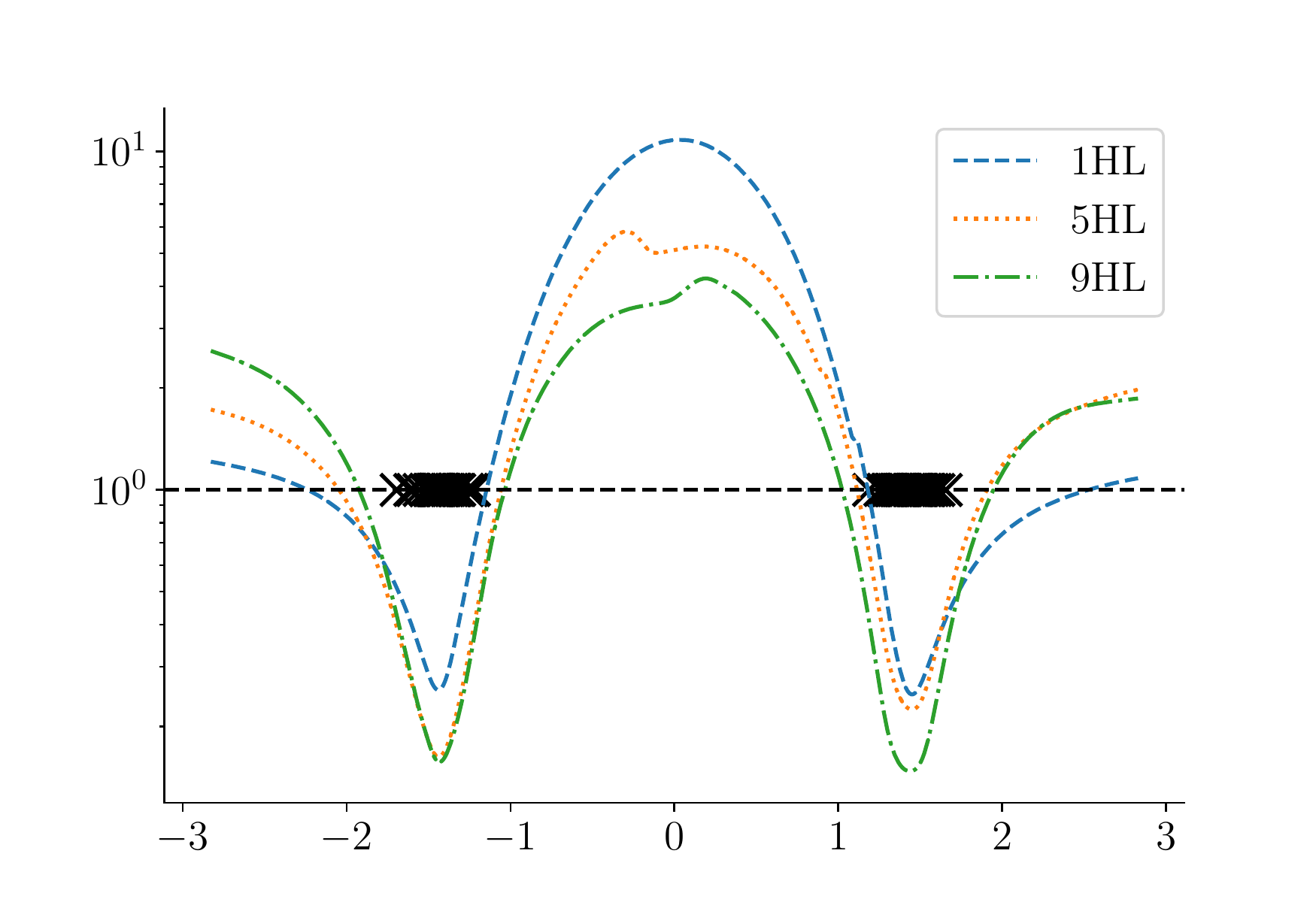}
    \caption{Mean Field VI ($\sigma_w=\sqrt{2}$ prior)}
    \end{subfigure}
    \begin{subfigure}[b]{0.4\textwidth}
    \includegraphics[width=\textwidth]{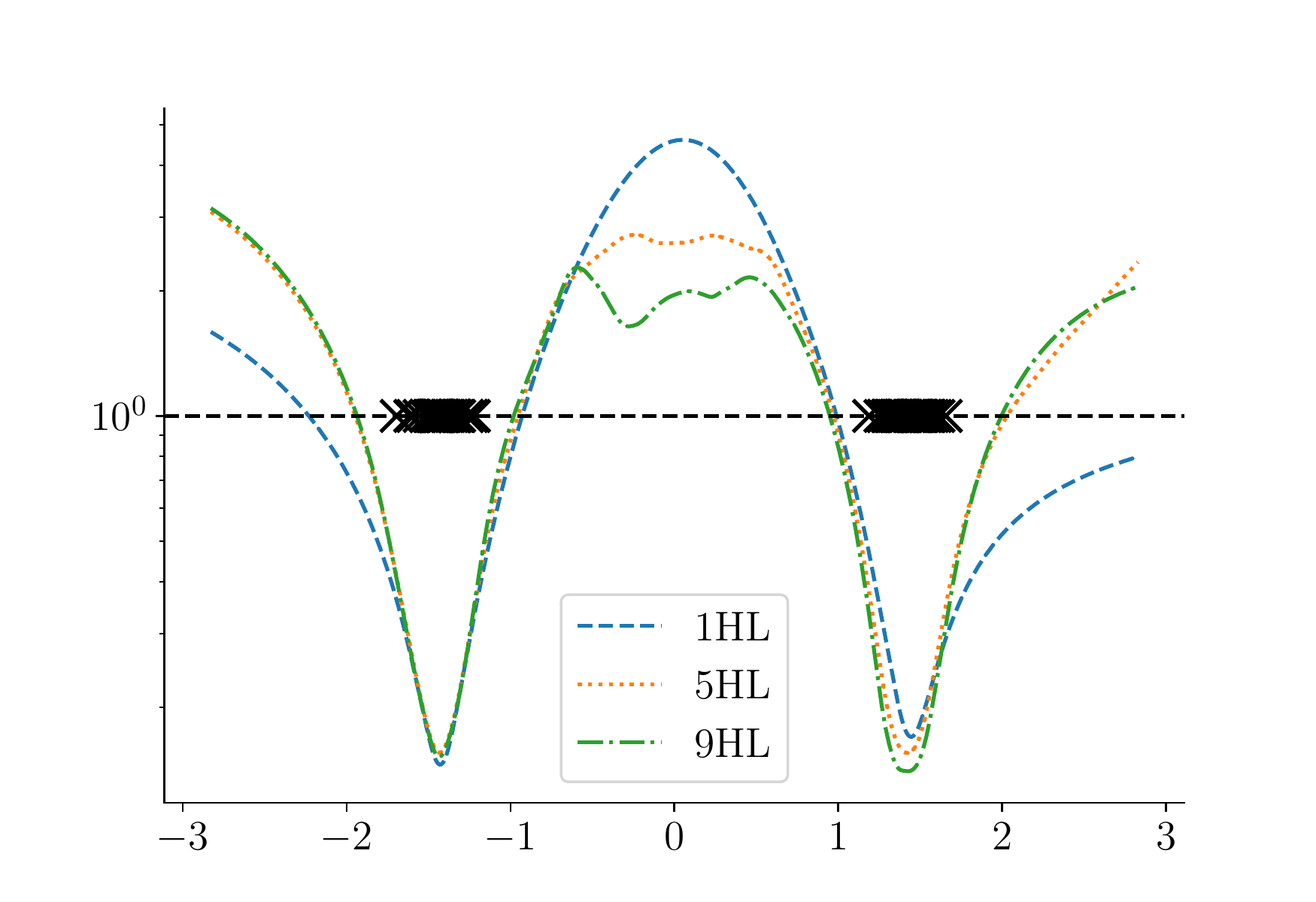}
        \caption{MC Dropout ($\sigma_w=\sqrt{2}$ prior)}
    \end{subfigure}
        \caption{Plots of the overconfidence ratio against $\lambda$ (where $\lambda$ is defined as in \cref{fig:2D_dataset}) for several depths of neural networks with $\sigma_w = 4, 2, 1.7$ for 1, 5 and 9 hidden layers respectively (top), $\sigma_w=\sqrt{2}$ for all depths (bottom). Projections of the datapoints onto the diagonal slice between the clusters are shown as black crosses (\ding{53}). We see that both MCDO and MFVI are overconfident ($>1$) in between data, and underconfident ($<1$) at the locations where we have observed data, relative to the GP reference.}
    \label{fig:ratio-function}
\end{figure}

\section{Initialisation of VI}\label{app:initialise-VI}
In order to assess whether the variational objective (ELBO) or optimisation is primarily responsible for the lack of in-between uncertainty when performing MFVI and MCDO, we considered the effect of initialisation on the quality of the posterior obtained after variational inference. In order to find setting of the weights so that the posterior distribution in function space was close to the exact posterior in function space, we initialised the weights of the network by training the network using mean squared loss between the mean and variance functions of a reference posterior and the approximate posterior (as in \cref{fig:universal}). The reference posterior was obtained by fitting the limiting GP on the dataset (shown in crosses). We used these weights as an initialisation for variational inference. The noise variance was fixed to the true noise variance that generated the data. The data itself was sampled from the limiting GP prior, so that the model should be able to fit the data well. 

Two-hidden layer MFVI and MCDO networks were used, with 50 hidden units in both layers. The solution found by minimising squared loss for 50,000 iterations between the mean and variance functions and a reference posterior may lead to distributions over weights such that the KL from these distributions to the prior is high. This can lead to very high values of the variational objective function. To alleviate this behaviour, we gradually interpolate between the squared-error loss and the variational objective, by taking convex combinations of the losses. Call the function space squared loss $L_1$ and the standard variational objective $L_2$. Then after the first 50,000 iterations of using $L_1$, we train for 10,000 iterations using $.9 L_1+.1L_2$, 10,000 iterations using $.8 L_1+ .2 L_2$ and so on until we are only training using $L_2$. We then train for 100,000 iterations using just $L_2$, to ensure the variational objective has converged. \Cref{fig:loss-initialization} shows that the obtained posterior still lacks in-between uncertainty, providing evidence that this may be due to the objective function rather than overfitting. 

\begin{figure}
    \centering
    \begin{subfigure}[b]{0.4\textwidth}
    \includegraphics[width=\textwidth]{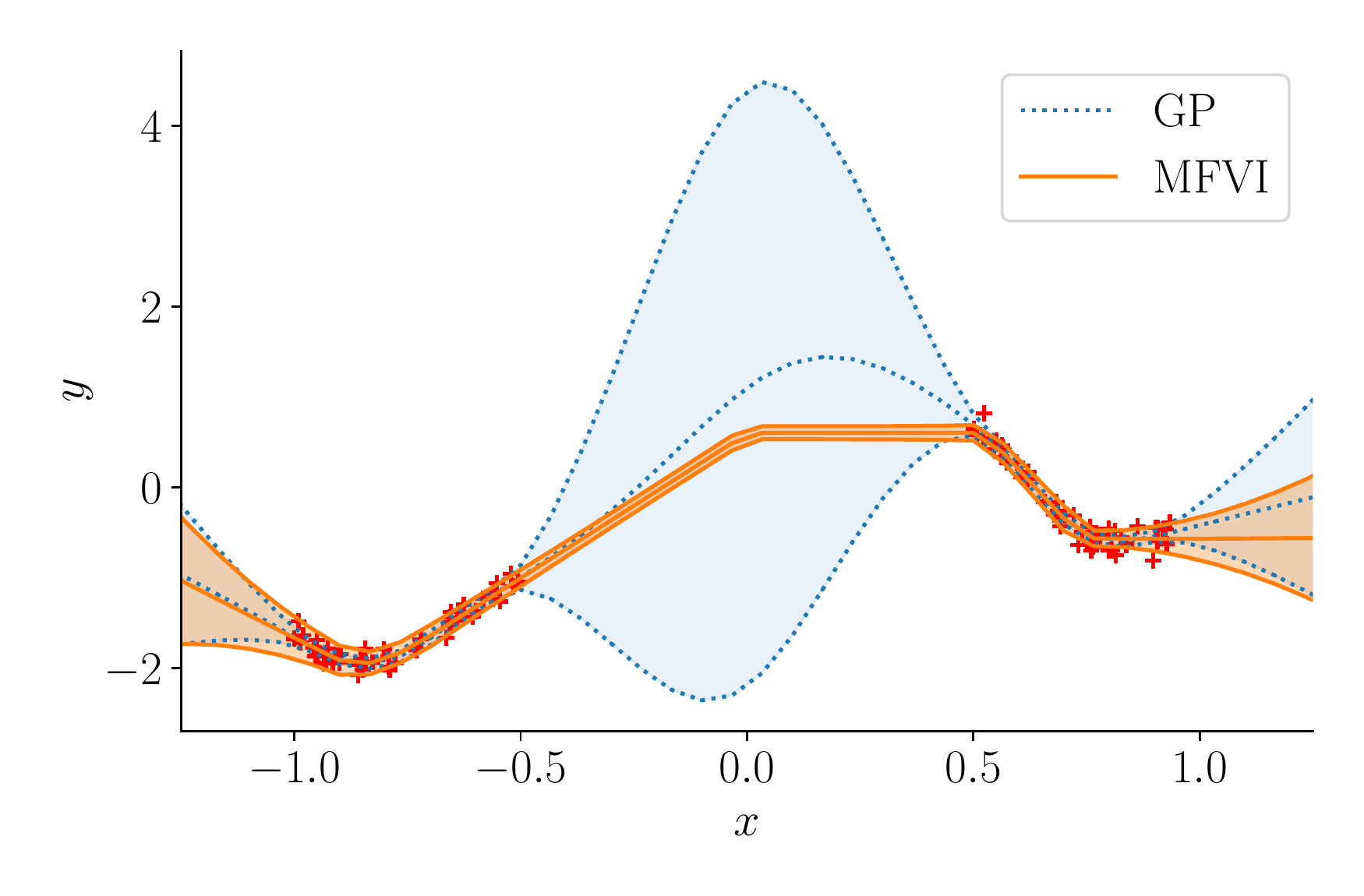}
    \caption{Mean Field VI}
    \end{subfigure}
    \begin{subfigure}[b]{0.4\textwidth}
    \includegraphics[width=\textwidth]{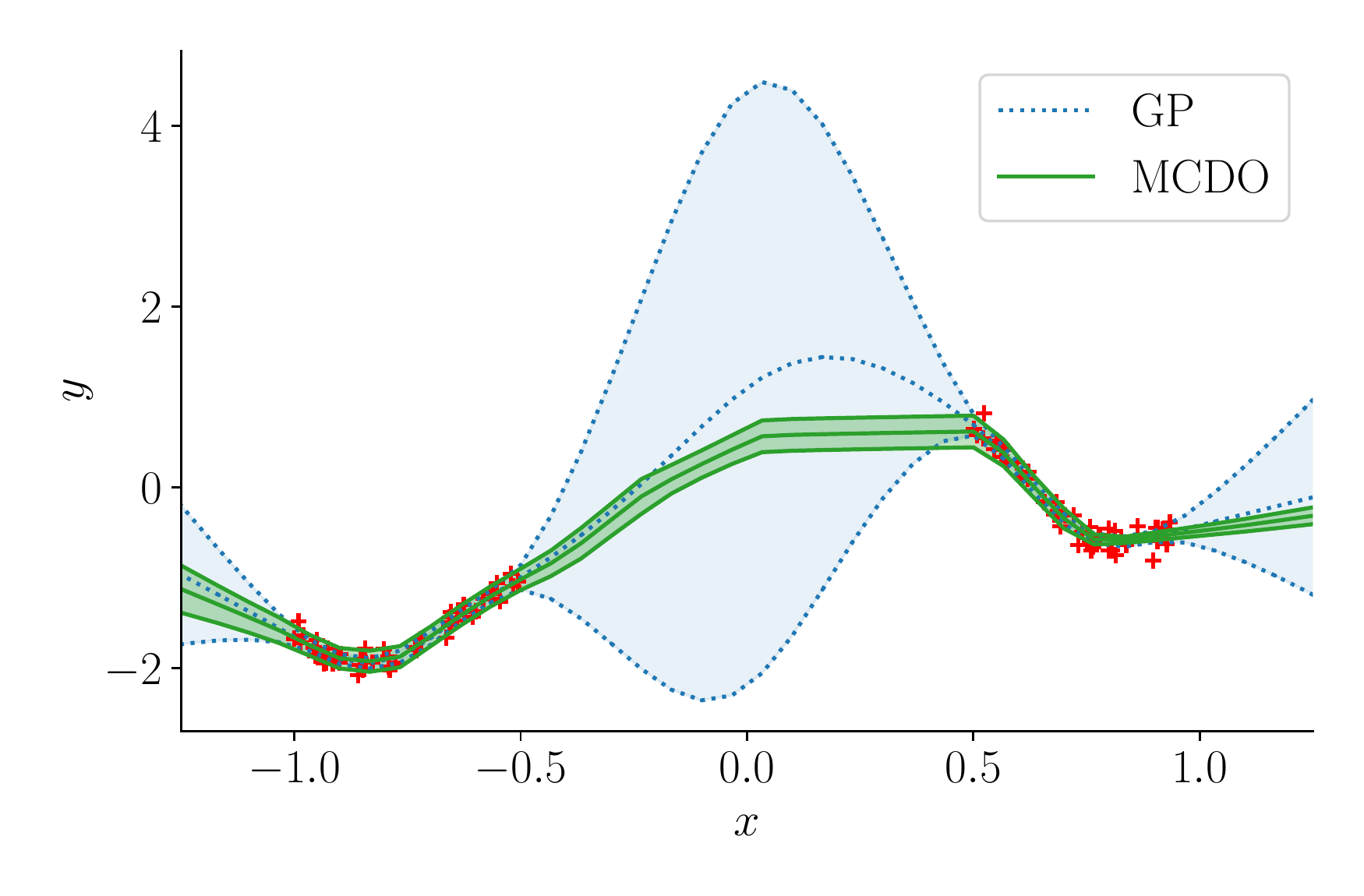}
        \caption{MC Dropout}
    \end{subfigure}
        \caption{Mean and error bars ($\pm$ 2 standard deviations) for the GP and the BNN with each inference scheme, trained on the data shown by the red crosses. The inference algorithms were initialised by first minimising the squared error to the reference GP mean and variance, and then running the respective inference algorithm.}
    \label{fig:loss-initialization}
\end{figure}

\section{Details and Additional Plots for Active Learning}\label{app:active-learning}

\subsection{Experimental Setup}\label{app:details-active-learning}
  We use the same initialisation as in \cref{app:details-effect-of-depth}. As the dataset has low noise, we use a homoskedastic Gaussian noise model with a fixed standard deviation of $0.01$ for all models. We used the ADAM optimiser with learning rate $1\times 10^{-3}$ for 20,000 epochs to optimise both MFVI and MCDO. We perform full batch training. All BNNs are retrained from scratch after the acquisition of each point from the pool set. We used 32 Monte Carlo samples from $q_{\phi}$ to estimate the objective function for both MFVI and MCDO. All networks had 50 neurons in each hidden layer. The prior for all BNNs and the GP was chosen to have $\sigma_w = \sqrt{2}, \sigma_b =1$ (see \cref{app:details-effect-of-depth} for definitions). $\sigma_w=\sqrt{2}$ was chosen so that the prior in function space has a stable variance as depth increases \cite{schoenholz2016deep}. The dropout probability was set at $p=0.05$ for all MCDO networks. The dropout $\ell_2$ regularisation was chosen to match the `KL condition' as stated in \citet[Section 3.2.3]{gal2016uncertainty}. The results are averaged over $20$ random initialisations and selections of the $5$ initial points in the active set. For MFVI and MCDO, the predictive distribution at test time and the predictive variances used for active learning were estimated using 500 samples from the approximate posterior. The parameter initialisations are the same as those in \cref{app:details-effect-of-depth}.
\subsection{Additional Figures}\label{app:figures-active-learning}
\Cref{fig:tsne3} shows the points chosen by deep BNNs. Again the GP chooses points from every cluster, and seems to focus on the `corners' of each cluster. MFVI samples from more clusters than the 1HL case, but still comparatively oversamples clusters further from the origin, and undersamples those near the origin. MCDO has a more spread out choice of points than the 1HL case, but still fails to obtain significantly better RMSE than random.

\Cref{fig:tsne_uncertainty_iter50,fig:tsne_uncertainty_iter0} show the predictive uncertainty of 1HL models at the beginning and end of active learning. All models significantly reduce their uncertainty around clusters that have been heavily sampled, except for MCDO. This causes MCDO to repeatedly sample near locations that have already been labelled, in contrast to the GP. Note also that MFVI is most confident at clusters near the origin that have never been sampled, and least confident at clusters far from the origin that have already been heavily sampled.

\vfill

\pagebreak

\begin{figure}
    \centering
    \begin{subfigure}[b]{0.32\textwidth}
        \includegraphics[width=\textwidth]{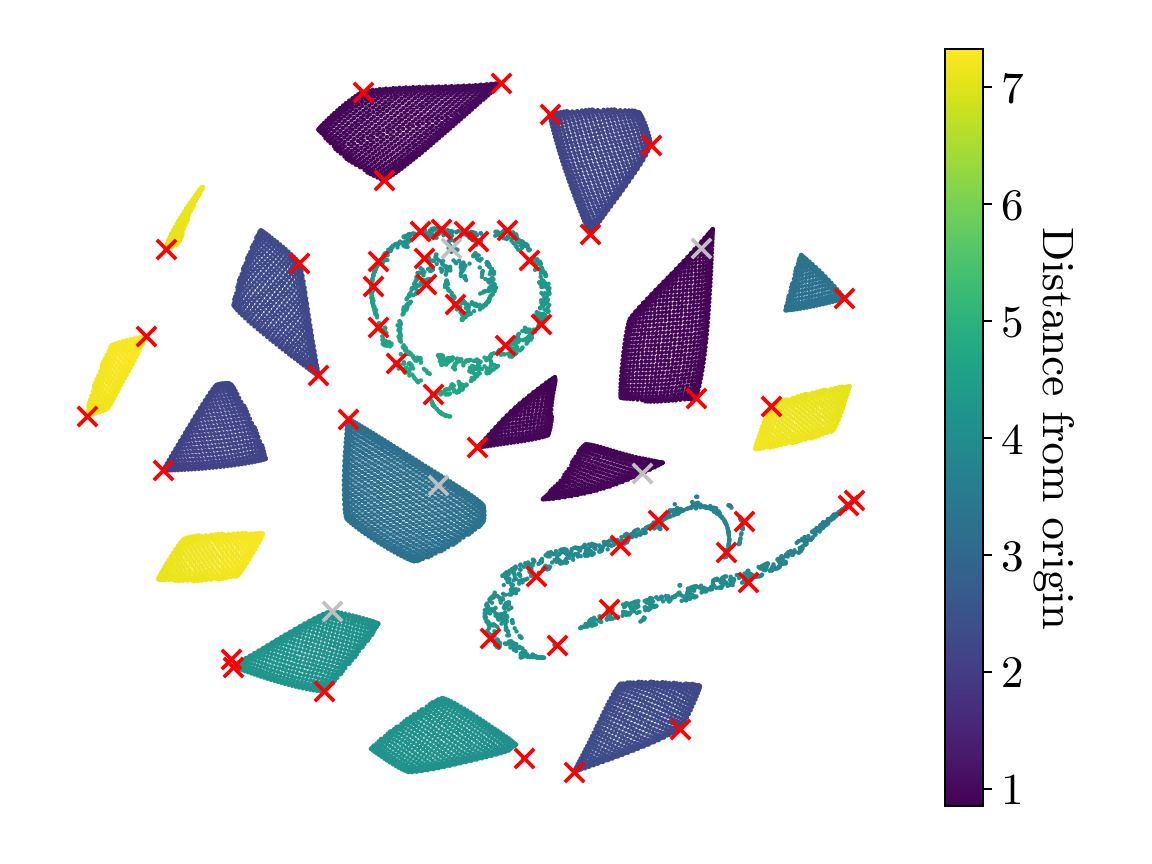}
        \caption{Limiting GP}
    \end{subfigure}
    \begin{subfigure}[b]{0.32\textwidth}
        \includegraphics[width=\textwidth]{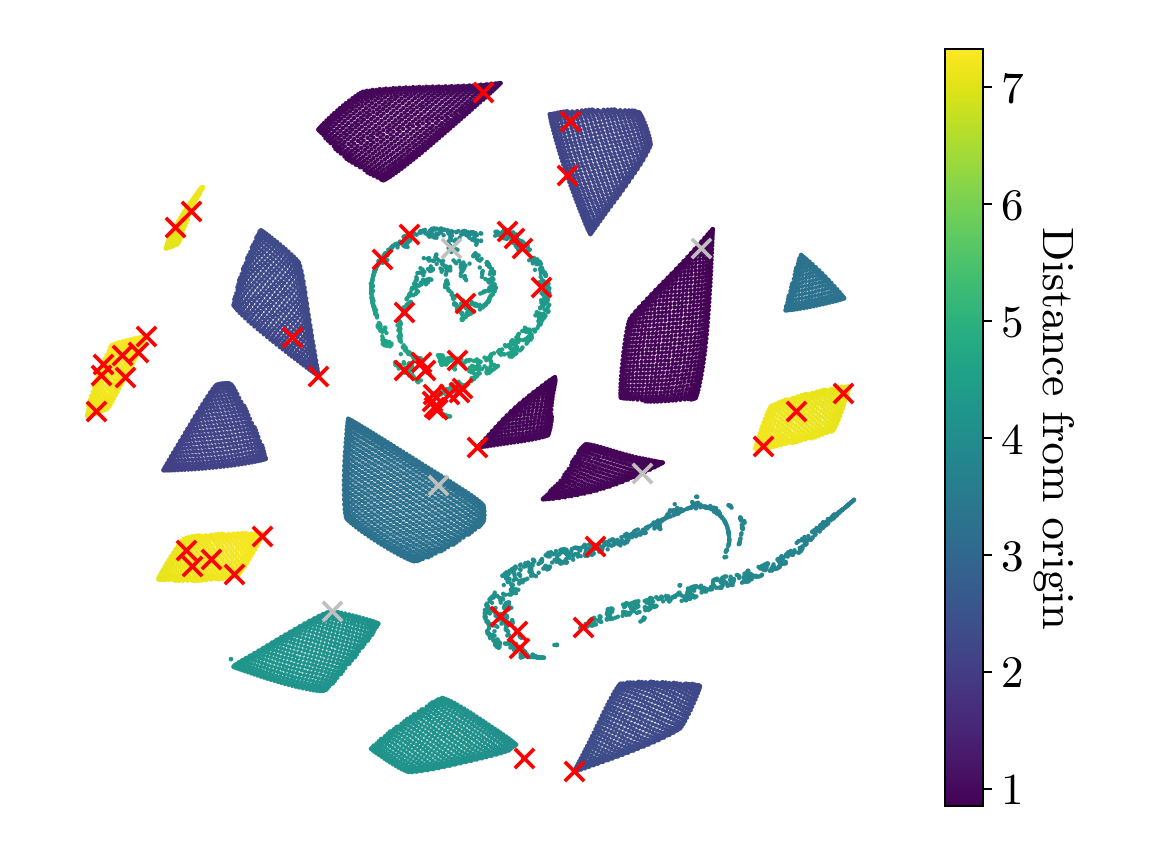}
        \caption{MFVI}
    \end{subfigure}
    \begin{subfigure}[b]{0.32\textwidth}
        \includegraphics[width=\textwidth]{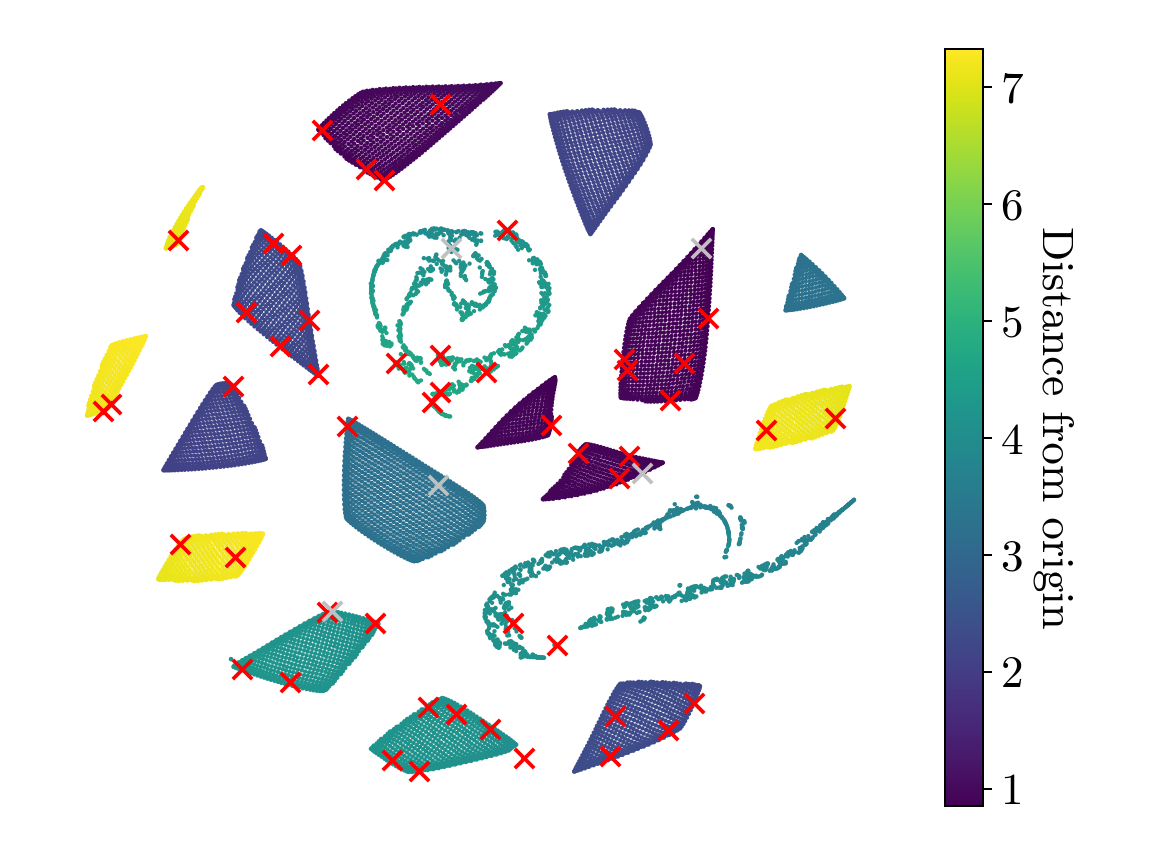}
        \caption{MCDO}
    \end{subfigure}
    \caption{Points chosen during active learning in the 3HL case. Colours denote distance from the origin in 14-dimensional input space. Grey crosses (\textcolor{gray}{\ding{53}}) denote the five points randomly chosen as an initial training set. Red crosses (\textcolor{red}{\ding{53}}) denote the 50 points selected by active learning. Again, the GP samples the corners of each cluster, and MFVI oversamples clusters far from the origin.}
    \label{fig:tsne3}
\end{figure}

\begin{figure}
    \centering
    \begin{subfigure}[b]{0.32\textwidth}
        \includegraphics[width=\textwidth]{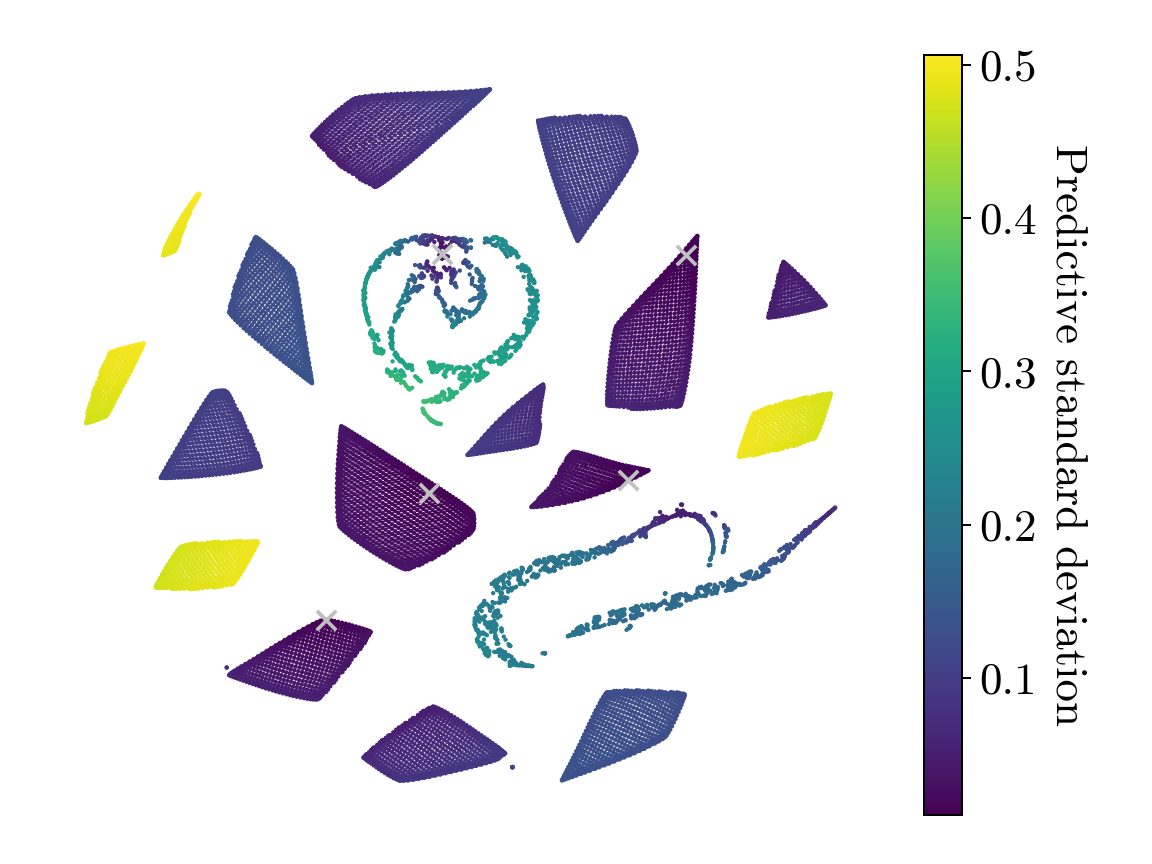}
        \caption{Limiting GP}
    \end{subfigure}
    \begin{subfigure}[b]{0.32\textwidth}
        \includegraphics[width=\textwidth]{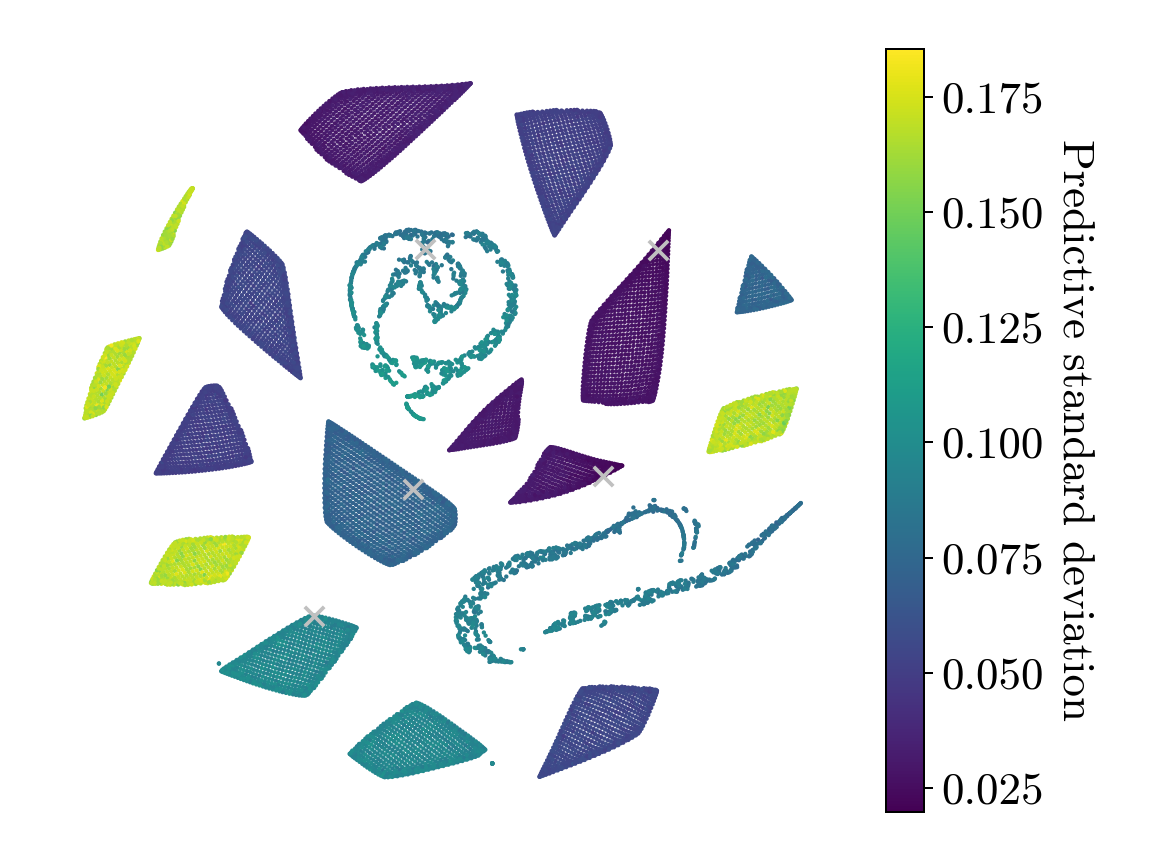}
        \caption{MFVI}
    \end{subfigure}
    \begin{subfigure}[b]{0.32\textwidth}
        \includegraphics[width=\textwidth]{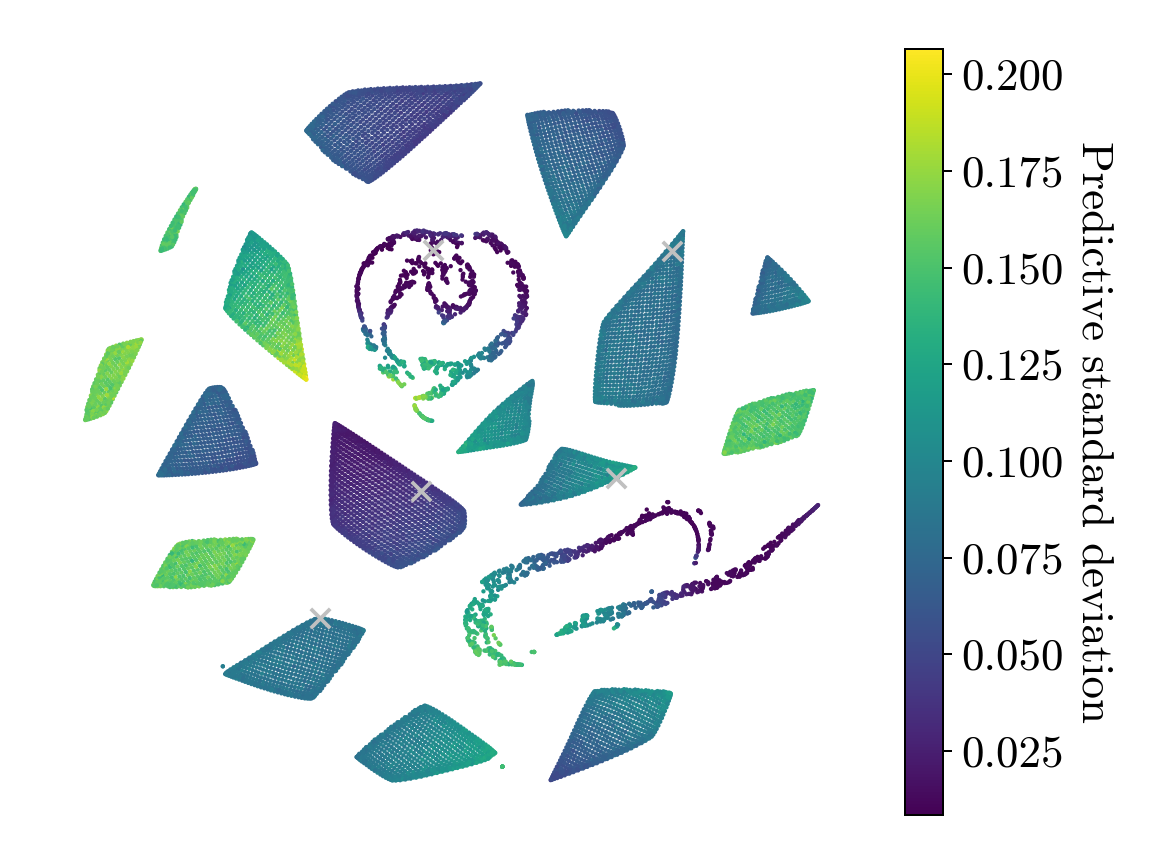}
        \caption{MCDO}
    \end{subfigure}
    \caption{Colours denote \emph{predictive uncertainties} in the 1HL case, at the beginning of active learning. As the noise standard deviation was fixed to $0.01$ for all models, changes in the predictive standard deviation reflect model uncertainty. Grey crosses (\textcolor{gray}{\ding{53}}) denote the five points randomly chosen as an initial training set.}
    \label{fig:tsne_uncertainty_iter0}
\end{figure}

\begin{figure}
    \centering
    \begin{subfigure}[b]{0.32\textwidth}
        \includegraphics[width=\textwidth]{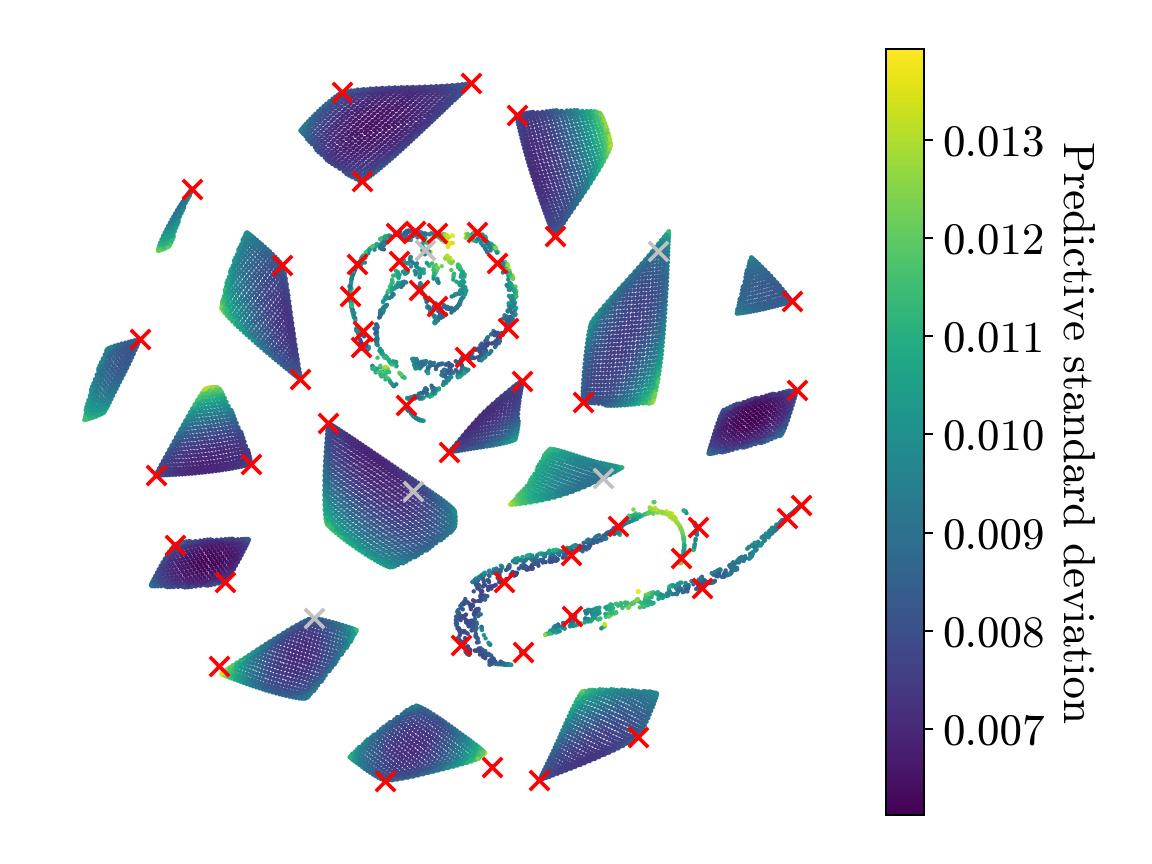}
        \caption{Limiting GP}
    \end{subfigure}
    \begin{subfigure}[b]{0.32\textwidth}
        \includegraphics[width=\textwidth]{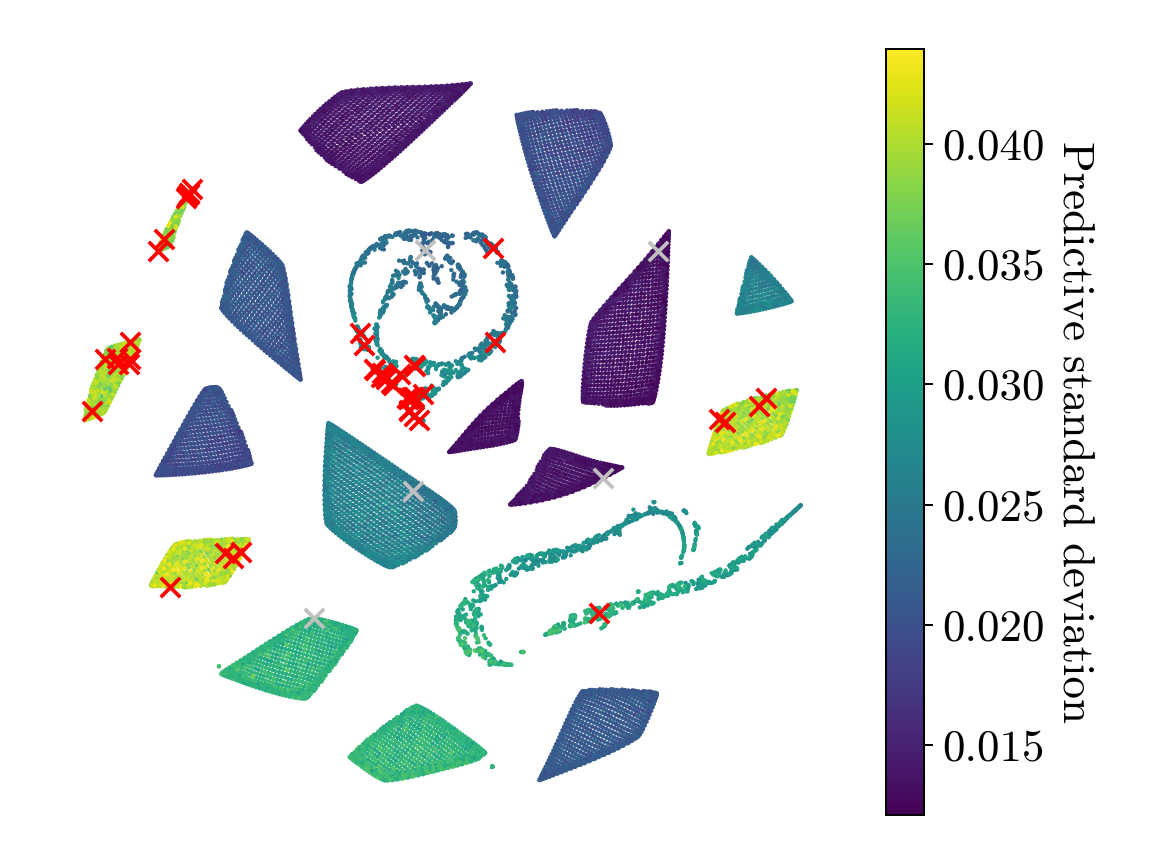}
        \caption{MFVI}
    \end{subfigure}
    \begin{subfigure}[b]{0.32\textwidth}
        \includegraphics[width=\textwidth]{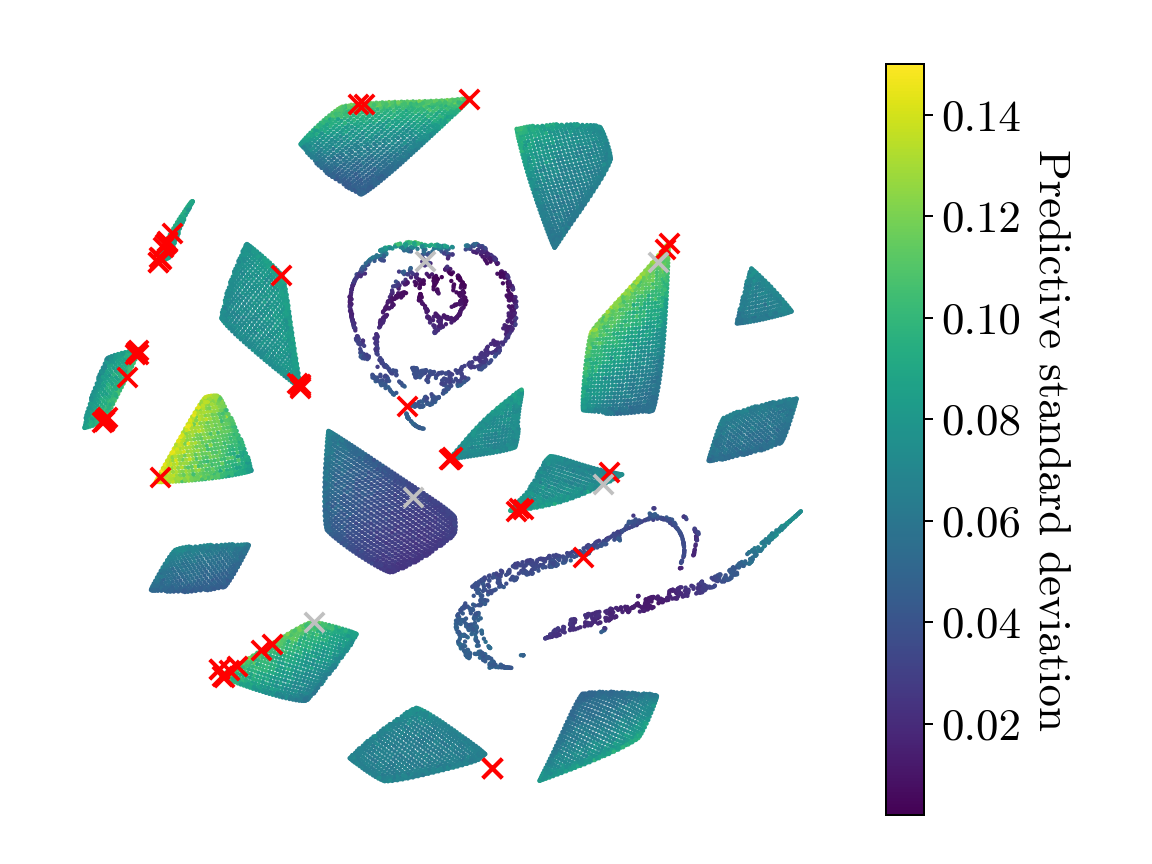}
        \caption{MCDO}
    \end{subfigure}
    \caption{Colours denote \emph{predictive uncertainties} in the 1HL case, after 50 iterations of active learning. As the noise standard deviation was fixed to $0.01$ for all models, changes in the predictive standard deviation reflect model uncertainty. Grey crosses (\textcolor{gray}{\ding{53}}) denote the five points randomly chosen as an initial training set. Red crosses (\textcolor{red}{\ding{53}}) denote the 50 points selected by active learning. Note how, compared to \cref{fig:tsne_uncertainty_iter0}, the GP has reduced its uncertainty near points it has observed, and is most uncertain on corners opposite those points. In contrast, for both MFVI and MCDO, the network is still uncertain around regions it has already collected points from, leading it to oversample those clusters and undersample others.}
    \label{fig:tsne_uncertainty_iter50}
\end{figure}

\end{document}